\documentclass[10pt,twocolumn,letterpaper]{article}

\usepackage{amsthm}

\usepackage[pagenumbers]{cvpr} 

\usepackage{graphicx}
\usepackage{amsmath}
\usepackage{amssymb}
\usepackage{booktabs}
\usepackage[accsupp]{axessibility}  
\usepackage{kmath}

\theoremstyle{plain}
\newtheorem{theorem}{Theorem}[section]

\theoremstyle{definition}

\theoremstyle{remark}

\usepackage[pagebackref,breaklinks,colorlinks]{hyperref}

\usepackage{xcolor}
\definecolor{darkgreen}{RGB}{119,185,0}
\definecolor{DarkCoral}{rgb}{0.8, 0.36, 0.27}
\usepackage{caption}
\usepackage{subcaption}

\definecolor{sb_blue}{rgb}{0.2,0.45,0.63}
\definecolor{sb_red}{rgb}{0.84,0.16,0.16}
\definecolor{sb_violet}{rgb}{0.58,0.4,0.74}
\definecolor{sb_orange}{rgb}{1,0.5,0.05}
\definecolor{sb_green}{rgb}{0.17,0.63,0.17}
\definecolor{sb_brown}{rgb}{0.55,0.34,0.29}

\newcommand{\ood}{OOD\xspace}
\newcommand{\id}{ID\xspace}
\newcommand{\ts}{TS\xspace}

\newcommand{\myparagraph}[1]{\noindent\textbf{#1}.}
\newcommand{\myparintro}[1]{\noindent\textbf{\textit{#1}.}}

\usepackage[capitalize]{cleveref}
\crefname{section}{Sec.}{Secs.}
\Crefname{section}{Section}{Sections}
\Crefname{table}{Table}{Tables}
\crefname{table}{Tab.}{Tabs.}

\def\cvprPaperID{9917} 
\def\confName{CVPR}
\def\confYear{2023}

\title{Reliability in Semantic Segmentation: Are We on the Right Track?} 

\author{
  Pau de Jorge \\
  \normalsize{University of Oxford} \\
  \normalsize{NAVER LABS Europe}\thanks{\url{https://europe.naverlabs.com}}
  \and
  Riccardo Volpi \\
  \normalsize{NAVER LABS Europe}
  \and
  Philip H. S. Torr \\
  \normalsize{University of Oxford}
  \and
  Gr\'egory Rogez \\
  \normalsize{NAVER LABS Europe}
}

\begin{document}

\maketitle


\begin{abstract}
Motivated by the increasing popularity of transformers in computer vision, in recent times
there has been a rapid development of novel architectures. 
While in-domain performance follows a constant, upward trend, properties like robustness or uncertainty estimation are 
less explored---leaving doubts about advances in model \textit{reliability}. Studies along these axes exist, but they are mainly limited to classification models.
In contrast, we 
carry out a study on
semantic segmentation, a relevant task for many real-world applications where model reliability is paramount. We analyze a broad variety of models, spanning from older ResNet-based architectures to novel transformers and assess their reliability based on four metrics: robustness, calibration, misclassification detection and out-of-distribution 
(\ood) detection. 
We find that while recent models are significantly more robust, they are not overall more reliable in terms of uncertainty estimation. We further explore methods that can come to the rescue and show that improving calibration can also help with other uncertainty metrics such as misclassification or \ood detection. This is the first study on modern segmentation models focused on both robustness and uncertainty estimation and 
we hope it will help practitioners and researchers interested in this fundamental vision task\footnote{Code available at \url{https://github.com/naver/relis}}.
\end{abstract}

\section{Introduction}
\label{sec:intro}

Humans tend to overestimate their abilities, a cognitive bias known as Dunning-Kruger
effect~\cite{Kruger1999UnskilledAndUnaware}. Unfortunately, so do deep neural
networks. Despite impressive performance on a wide range of tasks, deep learning models tend to be overconfident---that is, they 
predict with high-confidence
even when they are wrong~\cite{GuoICML17OnCalibrationOfModernNNs}. This effect is even more severe under domain shifts, where models tend to underperform in general~\cite{ovadia2019can, RechtICLR19DuImageNetClassifiersGeneralizeToImageNet, HendrycksICLR19BenchmarkingNNRobustnessCommonCorruptionsPerturbations}.

\begin{figure}
\centering
\includegraphics[width=\linewidth]{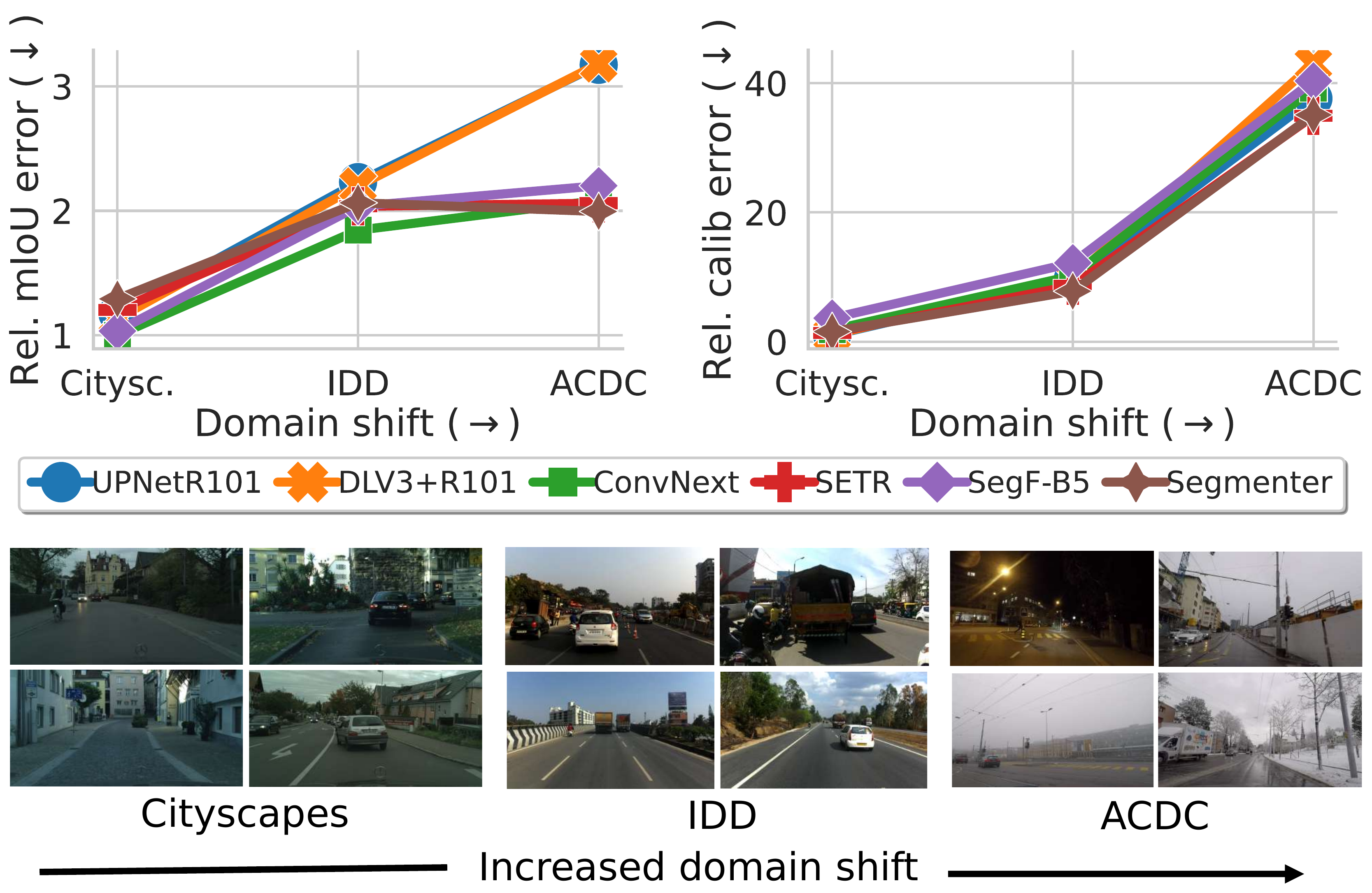}

\caption{\textbf{Top: mIoU and ECE \vs domain shift.} Errors are normalized with respect to the lowest error on the training distribution (Cityscapes). We compare recent segmentation models, both transformer-based (SETR \cite{ZhengCVPR21RethinkingSiSSeq2SeqPerspTransformers}, SegFormer \cite{XieNIPS21SegFormerSemSegmTransformers} and Segmenter \cite{StrudelICCV21SegmenterTransformerForSemSegm}) and convolution-based (ConvNext \cite{LiuCVPR22AConvNet4The2020s}) with ResNet baselines (UPerNet\cite{XiaoECCV18UnifiedPerceptualParsingSceneUnderstanding} and DLV3+\cite{ChenX17RethinkingAtrousConvolutionSemSegm}). All recent models (both transformers and CNNs) are remarkably more robust than ResNet baselines (whose lines in mIoU overlap), however, ECE increases sharply for all methods. \textbf{Bottom:} Sample images for each dataset.
}
\vspace{-5pt}
\label{figure:splash}
\end{figure}

While these vulnerabilities affect deep models in general, they are often studied for classification models and are comparably less explored for semantic segmentation, a fundamental task in computer vision that is key to many
%
critical applications such as autonomous driving and
AI-assisted medical imaging. In those applications, domain shifts are more the rule than the exception (\eg, changes in weather for
a self-driving car or differences across patients for a medical imaging system).
Therefore, brittle performance and overconfidence under domain
shifts are two important and challenging problems 
to address
for a
safe deployment of 
artificial intelligence
systems in the real world.

With that in mind, we argue that a \textit{reliable} model should 
\textit{i)}
be robust to domain shifts and 
\textit{ii)}
provide good uncertainty estimates. The core goal of this study is providing an answer to the following, crucial question:
\textbf{are state-of-the-art semantic segmentation models improving in terms of robustness \textit{and} uncertainty estimation?}

To shed light on this, we evaluate a large body of segmentation models, assessing their in-domain (\id) \vs out-of-domain (\ood) prediction quality (\textbf{robustness}) together with their calibration, misclassification detection and \ood detection (\textbf{uncertainty estimation}).

We argue that a study of this kind is crucial to understand whether research on semantic segmentation is moving in the right direction.
Following the rise of transformer architectures in computer vision~\cite{DosovitskiyICLR21AnImageIsWorth16x16WordsTransformersAtScale, touvron2021training, carion2020end,LiuICCV21SwinTransformerHierarchicalViTShiftedWindows},
several studies have compared recent self-attention and CNN-based \textit{classification} models in terms of robustness \cite{BhojanapalliICCV21UnderstandingRobustnessTransformersImageClass, naseer2021intriguing, bai2021transformers, paul2022vision, mao2022towards, LiuCVPR22AConvNet4The2020s} and predictive uncertainty \cite{minderer2021revisiting, pinto2022impartial}. Yet, when it comes to \textit{semantic segmentation}, prior studies \cite{XieNIPS21SegFormerSemSegmTransformers, zhou2022understanding} only focused on robustness, using synthetic corruptions as domain shifts (\eg, blur, noise) \cite{KamannCVPR20BenchmarkingRobustnessSemSegmModels}. In contrast, we 
consider
natural, realistic domain shifts and study segmentation models both in terms of robustness and uncertainty, leveraging datasets captured in different conditions---see \cref{figure:splash} (bottom). 

Task-specific studies are important, since task-specific architectures and learning algorithms may carry different behaviors and some observations made for classification might not hold true when switching to segmentation.
For instance, contrary to Minderer~\etal \cite{minderer2021revisiting}, we observe that improvements in calibration are far behind those in robustness, see \cref{figure:splash} (top). Furthermore, previous analyses only
consider simple 
calibration approaches~\cite{GuoICML17OnCalibrationOfModernNNs}
while assessing
model reliability;
in contrast,
we make a step forward and
explore content-dependent calibration
strategies~\cite{gong2021confidence, ding2021local}, which show promise to improve reliability out of domain.

Our analysis allows us individuating
in which directions 
we are
improving and in which 
we are lagging behind.
This is the first work to systematically study robustness
and uncertainty under domain shift for a large suite of segmentation models and
we believe it
can
help 
practitioners and researchers 
working on
semantic segmentation. 
We summarize our main observations in the following.

\myparintro{i) Remarkable improvements in robustness, but poor in calibration} Under domain shifts, recent segmentation models perform significantly better (in terms of mIoU)---with larger improvements for stronger shifts. Yet, \ood calibration error increases dramatically for all models.

\noindent
\textit{\textbf{ii) Content-dependent calibration~\cite{ding2021local} can improve \ood calibration}}, especially under strong domain shifts, where models are poorly calibrated.

\myparintro{iii) Misclassification detection shows different 
model ranking
in and out of domain} 
When tested in domain, recent models 
underperform the ResNet baseline. As the domain shift increases, 
recent models take the lead.

\myparintro{iv) \ood detection is inversely correlated 
with
performance} 
Indeed,
a small ResNet-18 backbone performs best.

\myparintro{v) Content-dependent calibration~\cite{ding2021local} can improve \ood detection and misclassification out of domain} We observe a significant increase in misclassification detection under strong domain shifts after improving calibration. We also observe improvements for \ood detection, albeit milder.



\begin{table}[t!]
\begin{center}
{\scriptsize
\setlength{\tabcolsep}{2.5pt}
\begin{tabular}{@{}lccccc@{}}
\toprule
 & 
\begin{tabular}{@{}c@{}}\textbf{Sem.} \\ \textbf{segm.}\end{tabular} &
\begin{tabular}{@{}c@{}}\textbf{Robust} \\ \textbf{performance}\end{tabular} &
\begin{tabular}{@{}c@{}}\textbf{Uncertainty} \\ \textbf{estimation}\end{tabular} &
\begin{tabular}{@{}c@{}}\textbf{Natural} \\ \textbf{shifts}\end{tabular} &
\begin{tabular}{@{}c@{}}\textbf{OOD calib} \\ \textbf{methods}\end{tabular}


\\

\midrule
Kamann~\etal 
\cite{KamannCVPR20BenchmarkingRobustnessSemSegmModels} & \checkmark & \checkmark & & & \\

\midrule
Bhojanapalli~\etal 
\cite{BhojanapalliICCV21UnderstandingRobustnessTransformersImageClass} &  & \checkmark & & \checkmark & \\

\midrule
Xie~\etal 
\cite{XieNIPS21SegFormerSemSegmTransformers} &  \checkmark & \checkmark & &  & \\

\midrule
Naseer~\etal 
\cite{naseer2021intriguing} &  & \checkmark & & & \\

\midrule
Bai~\etal 
\cite{bai2021transformers} &  & \checkmark &  & \checkmark & \\

\midrule
Minderer~\etal 
\cite{minderer2021revisiting} &  & \checkmark & \checkmark & \checkmark & \\

\midrule
Paul and Cheng 
\cite{paul2022vision} &  & \checkmark & & \checkmark & \\

\midrule
Mao~\etal 
\cite{mao2022towards} &  & \checkmark & & \checkmark & \\

\midrule
Liu~\etal 
\cite{LiuCVPR22AConvNet4The2020s} &  & \checkmark & & \checkmark & \\

\midrule
Zhou~\etal 
\cite{zhou2022understanding} &  \checkmark & \checkmark & & & \\

\midrule
Pinto~\etal 
\cite{pinto2022impartial} &  & \checkmark & \checkmark & \checkmark & \\

\midrule
\midrule

\textit{Ours} &  \checkmark & \checkmark & \checkmark & \checkmark  & \checkmark \\

\bottomrule
\end{tabular}
}
\end{center}
\caption{\textbf{Studies of recent architectures}. While several prior works studied robustness and uncertainty of transformer- and CNN- based \textit{classifiers}, studies on \textit{segmentation} limited to robustness. This is the first study assessing robustness \textit{and} uncertainty of modern segmentation models. Moreover, we consider natural domain shifts and are the only analysis to include content-dependent methods~\cite{gong2021confidence, ding2021local} to improve calibration
in \ood settings. 
}
\label{tab:comparison} 
\end{table}

\section{Related work}

\label{sec:related_work}
We study robustness and uncertainty in 
semantic
segmentation. In doing so, we touch several fields, which we cover in the following. We further discuss related studies.


\myparagraph{Segmentation models} 
Modern segmentation pipelines typically consist of encoder-decoder architectures~\cite{CsurkaNow22SemanticImageSegmentation2DecasesOfResearch, BadrinarayananPAMI17SegnetDeepConvEncoderDecoder, NohICCV15LearningDeconvolutionNNSegmentation, RonnebergerMICCAI15UNetSegmentation}. Decoders are usually designed \textit{ad hoc} for segmentation, with DeepLab~\cite{ChenPAMI17DeeplabSemanticImgSegmentationDeepFullyConnectedCRF, ChenX17RethinkingAtrousConvolutionSemSegm, ChenECCV18EncoderDecoderAtrousSeparableConvSemSegm} and UPerNet~\cite{XiaoECCV18UnifiedPerceptualParsingSceneUnderstanding} being two of the most prominent. On the other hand, the evolution of 
encoders
has been closely related to that of classification 
models, with ResNet~\cite{he2016deep} being one of the most popular for years.
The rise of transformers in computer vision~\cite{DosovitskiyICLR21AnImageIsWorth16x16WordsTransformersAtScale} has led to a flurry of works leveraging self-attention for 
segmentation~\cite{ZhengCVPR21RethinkingSiSSeq2SeqPerspTransformers,StrudelICCV21SegmenterTransformerForSemSegm,XieNIPS21SegFormerSemSegmTransformers}.
Novel convolutional architectures inspired by transformers have also risen~\cite{LiuCVPR22AConvNet4The2020s}. We compare several recent segmentation models 
against
ResNet baselines, in terms of both robustness and uncertainty.

\myparagraph{Robustness} The brittleness of neural networks to changes in the input domain is a well-studied problem and many sub-formulations exist~\cite{TaoriNeurIPS20MeasuringRobustnessToNaturalDistributionShifts}. 
Robustness against synthetic shifts takes into account 
samples crafted
by artificially altering images, for example injecting noise or blur (corruption robustness~\cite{HendrycksICLR19BenchmarkingNNRobustnessCommonCorruptionsPerturbations,KamannCVPR20BenchmarkingRobustnessSemSegmModels}), or crafting imperceptible perturbations to induce model failure (adversarial robustness~\cite{GoodfellowICLR15ExplainingHarnessingAdvExamples}).
Robustness against \textit{natural} shifts focuses on changes that may arise naturally, without human intervention~\cite{RechtICLR19DuImageNetClassifiersGeneralizeToImageNet,hendrycks2021natural}.

In this work we are interested in comparing the robustness of different off-the-shelf segmentation models under \textit{natural} domain shifts, since these are particularly relevant in real-world applications.
In particular, we focus on semantic segmentation of urban scenes, hence, 
we evaluate models on samples from unseen geographical locations~\cite{VarmaWACV19IDDDatasetExploringADUnconstrainedEnvironments} and weather conditions~\cite{SakaridisICCV21ACDCAdverseConditionsDatasetSIS}.
Segmentation robustness against natural shifts has been studied before~\cite{YueICCV19DomainRandomizationPyramidConsistencySimulationToReal,VolpiCVPR22OnRoadOnlineAdaptSiS}, yet not in tandem with uncertainty and within a large-scale study taking 
into
consideration several recent models.\looseness=-1


\myparagraph{Uncertainty}
Guo \etal \cite{GuoICML17OnCalibrationOfModernNNs} have shown that deep 
models 
are
overconfident. They have proposed a simple, yet effective solution known as temperature scaling (\ts) where the output logits are divided by a temperature parameter before 
the softmax layer. Other calibration methods have been proposed (\eg, \cite{naeini2015obtaining, kull2019beyond, gupta2020calibration}), but \ts is still
very
popular due to its simplicity and the fact that it does not alter predictions. 

Calibration with \ts is effective
in \id settings; yet, Ovadia \etal \cite{ovadia2019can} have shown that model calibration degrades significantly out of domain. Some methods have been proposed that address this problem~\cite{pampari2020unsupervised, park2020calibrated, wang2020transferable}, by assuming access to unlabeled \ood images beforehand.
On the other hand, Gong~\etal~\cite{gong2021confidence} have proposed methods that improve \ood calibration without any data from the target domain. 
They propose to cluster the calibration set in different ``domains'' and find a different temperature value for each. At test time, images are calibrated using the temperature from the closest cluster. 
%
\textit{Ad hoc} for semantic segmentation, Ding~\etal~\cite{ding2021local} propose a content-dependent calibration strategy that learns a small calibration network to predict a temperature for each pixel in an image.

In our study we test \id and \ood performance of several calibration methods, focusing on techniques that do not require access to \ood samples~\cite{GuoICML17OnCalibrationOfModernNNs,gong2021confidence,ding2021local}---as generally robustness is evaluated on unseen domains~\cite{HendrycksICLR19BenchmarkingNNRobustnessCommonCorruptionsPerturbations,KamannCVPR20BenchmarkingRobustnessSemSegmModels,TaoriNeurIPS20MeasuringRobustnessToNaturalDistributionShifts,pinto2022impartial}


\myparagraph{Previous analyses}
In~\cref{tab:comparison} we 
compare
related studies 
on different aspects of reliability. Several works have suggested that transformer-based \textit{classifiers} are more robust than CNNs \cite{BhojanapalliICCV21UnderstandingRobustnessTransformersImageClass, naseer2021intriguing, bai2021transformers, mao2022towards, paul2022vision}. Yet, the recent ConvNeXt \cite{LiuCVPR22AConvNet4The2020s} has challenged this result and later work have suggested that further investigation is needed~\cite{pinto2022impartial}. 
Minderer \etal \cite{minderer2021revisiting} have compared calibration of several classifiers, concluding that convolution-free models are more robust \textit{and} better calibrated. In contrast, Pinto \etal \cite{pinto2022impartial} have compared recent transformers and CNNs, arguing there is ``no clear winner''. 
Some works have compared the robustness of transformers and CNNs 
for
segmentation~\cite{XieNIPS21SegFormerSemSegmTransformers, zhou2022understanding}---but only against \textit{synthetic} domain shifts.
We broadly study robustness \textit{and} uncertainty in segmentation under \textit{natural} domain shifts. Similarly to \cite{pinto2022impartial}, we do not observe a single model family which is better calibrated in all scenarios. In contrast with \cite{minderer2021revisiting} though, we observe that robustness and calibration \textit{do not} go hand in hand. 
This shows that not all trends observed in classification transfer to segmentation, confirming the importance of task-specific studies like ours.

\section{Experimental settings and preliminaries}
\label{sec:experimental_settings}

\subsection{Datasets} 

As discussed in Section~\ref{sec:related_work}, we use different datasets for semantic segmentation of urban scenes to model natural domain shifts---inspired by prior art~\cite{YueICCV19DomainRandomizationPyramidConsistencySimulationToReal,ChoiCVPR21RobustNetImprovingDGUrbanSiSInstanceSelectiveWhitening,VolpiCVPR22OnRoadOnlineAdaptSiS}. 


\noindent
\textbf{Cityscapes (CS)~\cite{CordtsCVPR16CityscapesDataset}} contains images taken across 50 European cities at day time with overall good weather. Training and validation sets use sequences from disjoint sets of cities. Following this protocol, we further split validation cities into a calibration and a test set. 
Since CS is a mainstream benchmark in semantic segmentation, we use it as our training set (\id) to leverage available trained 
weights.

\noindent
\textbf{IDD~\cite{VarmaWACV19IDDDatasetExploringADUnconstrainedEnvironments}} 
was captured in the cities of Hyderabad, Bangalore and their outskirts. Given the different geographical location, it poses a clear domain shift for CS models.
%
%

\noindent
\textbf{ACDC~\cite{SakaridisICCV21ACDCAdverseConditionsDatasetSIS}} 
contains images captured in adverse conditions 
(Fog, Rain, Snow and Night),
which translate into strong domain shifts. 
%
Similarly to previous work~\cite{pinto2022impartial,zhou2022understanding, XieNIPS21SegFormerSemSegmTransformers}, we focus on \textit{covariate} shifts, \ie, changes in the input distribution---keeping the 
label set
fixed. In practice, for \ood settings (IDD and ACDC) we consider the 19 classes from CS, ignoring the others. 
The only exception is one experiment in~\cref{ap:local_ood}, where we consider label shifts.


\subsection{Architectures} 
\label{sec:arch}

We implement our models with MMsegmentation~\cite{mmseg2020}. Following prior work~\cite{pinto2022impartial, minderer2021revisiting}, we use the original training recipes for each model to compare them at their best.
For completeness, we also explore the effects of pre-training dataset and number of training iterations in \cref{ap:training_ablations}.

\myparagraph{\textcolor{sb_red}{SETR} \cite{ZhengCVPR21RethinkingSiSSeq2SeqPerspTransformers}} The first convolution-free segmentation model. It uses a ViT backbone \cite{DosovitskiyICLR21AnImageIsWorth16x16WordsTransformersAtScale} and 
different
decoders (SETR-Naive, SETR-MLA and SETR-PUP).
We use the ViT-Large backbone and analyze all three decoders.\looseness=-1

\myparagraph{\textcolor{sb_brown}{Segmenter} \cite{StrudelICCV21SegmenterTransformerForSemSegm}} Similarly to SETR, it also uses a ViT backbone; yet, unlike the simpler SETR decoders, 
it carries a transformer-based one.
We also test ViT-Large.

\myparagraph{\textcolor{sb_violet}{SegFormer} \cite{XieNIPS21SegFormerSemSegmTransformers}} This 
model incorporates
an original self-attention mechanism and several architectural changes to be more efficient. 
We evaluate all models from this family (B0--B5), gradually increasing the number of parameters.

\myparagraph{\textcolor{sb_green}{ConvNeXt} \cite{LiuCVPR22AConvNet4The2020s}} 
Convolutional model with
changes inspired by transformers.
We use ConvNeXt-Large, comparable in size to ViT-Large, 
with an UPerNet decoder~\cite{XiaoECCV18UnifiedPerceptualParsingSceneUnderstanding}.

\myparagraph{ResNet-based~\cite{he2019bag}}
We use ResNet-V1c model, the default ResNet in MMsegmentation library~\cite{mmseg2020}. Compared to vanilla ResNet\cite{he2016deep} it uses a stem with three 3x3 convs (instead of a 7x7 conv). We use ResNet-18/50/101 models and two popular decoders: \textbf{\textcolor{sb_orange}{DLV3+}}~\cite{ChenX17RethinkingAtrousConvolutionSemSegm} and \textbf{\textcolor{sb_blue}{UPerNet}}~\cite{XiaoECCV18UnifiedPerceptualParsingSceneUnderstanding}

Additionally,
in \cref{ap:mask2former}
we explore \textbf{Mask2Former}, an architecture for \textit{universal image segmentation} \cite{ChengCVPR22MaskedAttentionMaskTransformer4UniversalSiS} that does not follow the conventional \textit{logits + softmax} paradigm.

\subsection{Reliability metrics}\label{sec:reliability_metrics}

\noindent
We evaluate model reliability on four aspects: robustness, calibration, misclassification detection and \ood detection.

\myparagraph{Robustness}
We measure robustness by evaluating the standardized mean
Intersection-over-Union (mIoU) performance in \ood settings, \ie, on ACDC and IDD. We also provide \id performance,
evaluating models on CS.

\myparagraph{Calibration} A model is said to be calibrated when the predictive probabilities (\ie, the logits after a softmax) correspond to the true probabilities. For instance, if we group all samples where the predicted probability is 
90\%,
we would expect that 
90\% 
of those 
predictions
are correct. The most common calibration metric is the \textit{Expected Calibration Error (ECE)} \cite{naeini2015obtaining}, which looks at the expected difference between the predicted and actual probabilities. To estimate the ECE we quantize the predicted probabilities and compare the accuracy with the mean probability in each bin. Since the binning strategy can affect the results, we 
test
ECE with equally spaced bins \cite{naeini2015obtaining}, equally populated bins \cite{nixon2019measuring, nguyen2015posterior} and the Kolmogorov-Smirnov test \cite{gupta2020calibration}, which gets rid of the binning strategy altogether. We report results with standard ECE, but find that all three aforementioned metrics 
yield similar conclusions
(see~\cref{sec:ablation_calibration_metrics}).

Given that in segmentation we have per-pixel predictions, the number of calibration samples explodes (a single CS image contains 2048 $\times$ 1024 $\simeq $ 2M pixels). We ablate the different calibration metrics as we sub-sample the number of pixels per image and observe that 20k pixels per image in enough to estimate the ECE (see~\cref{sec:ablation_number_pixels_calibration}).

\myparagraph{Misclassification detection} 
A desiderata for a reliable model is to assign a larger confidence
to correct outputs than incorrect ones\footnote{We use max softmax as confidence in the paper; in \cref{sec:ablation_prob_vs_entropy} we present similar results with negative entropy.}. 
In the ideal case, if we sorted all predictions from least to most confident, we would have all the incorrect predictions first and correct ones later. \textit{Misclassification detection} measures how far away are we from such an ideal case. This can be measured with \textit{Rejection-Accuracy curves} \cite{fumera2002support, hendrycks2021natural}: we reject samples with low confidence and compute the accuracy \vs amount of rejected samples. However, these are biased in favor of better-performing models,
since
the base accuracy 
is higher in the first place.
To avoid that, we follow Malinin~\etal~\cite{malinin2019ensemble} 
and normalize
the area under the curve by that of an oracle and subtracts a baseline score with randomly sorted samples. The resulting metric, known as the Prediction Rejection Ratio (PRR), will be positive if the confidence is better than the random baseline and will have a maximum score of 100\% when the model matches the oracle.\looseness=-1

\myparagraph{Out-of-domain detection} Another important aspect in reliability is that models are aware of their ``domain of expertise'' (\ie, their training domain). When a sample differs significantly from the training 
samples,
we would expect the model to be more uncertain of its prediction. Similarly to misclassification detection, in \ood detection we try to separate \id from \ood images based on the network confidence. This has 
broad applicability,
\eg,
generating alerts for samples too far from the training domain
or, connecting to active learning, gathering them for further review and annotation. As in the rest of our work, we consider a whole image to be out of domain if it presents a significant 
domain
shift; that is, we consider CS in domain and IDD/ACDC out of domain. Since we define \id and \ood samples at the image level, we consider the average confidence of all pixels in a given image.
We use the Area Under the Receiver Operating Characteristic curve (AUROC) \cite{murphy2012machine}, which goes from 0 to 1 (1 being the best score). 
Additionally, in \cref{ap:local_ood} we consider \ood detection at a region level, considering classes in the IDD dataset not present in CS, \ie, \textit{rickshaw}, \textit{billboard}, \textit{guard rail}, \textit{tunnel} and \textit{bridge}.

\section{Are modern segmentors more \textit{reliable}?}

In the following, we present the main findings of our study on the reliability of semantic segmentation models. 

\subsection{Robustness}

\begin{figure}
\centering
\includegraphics[width=\linewidth]{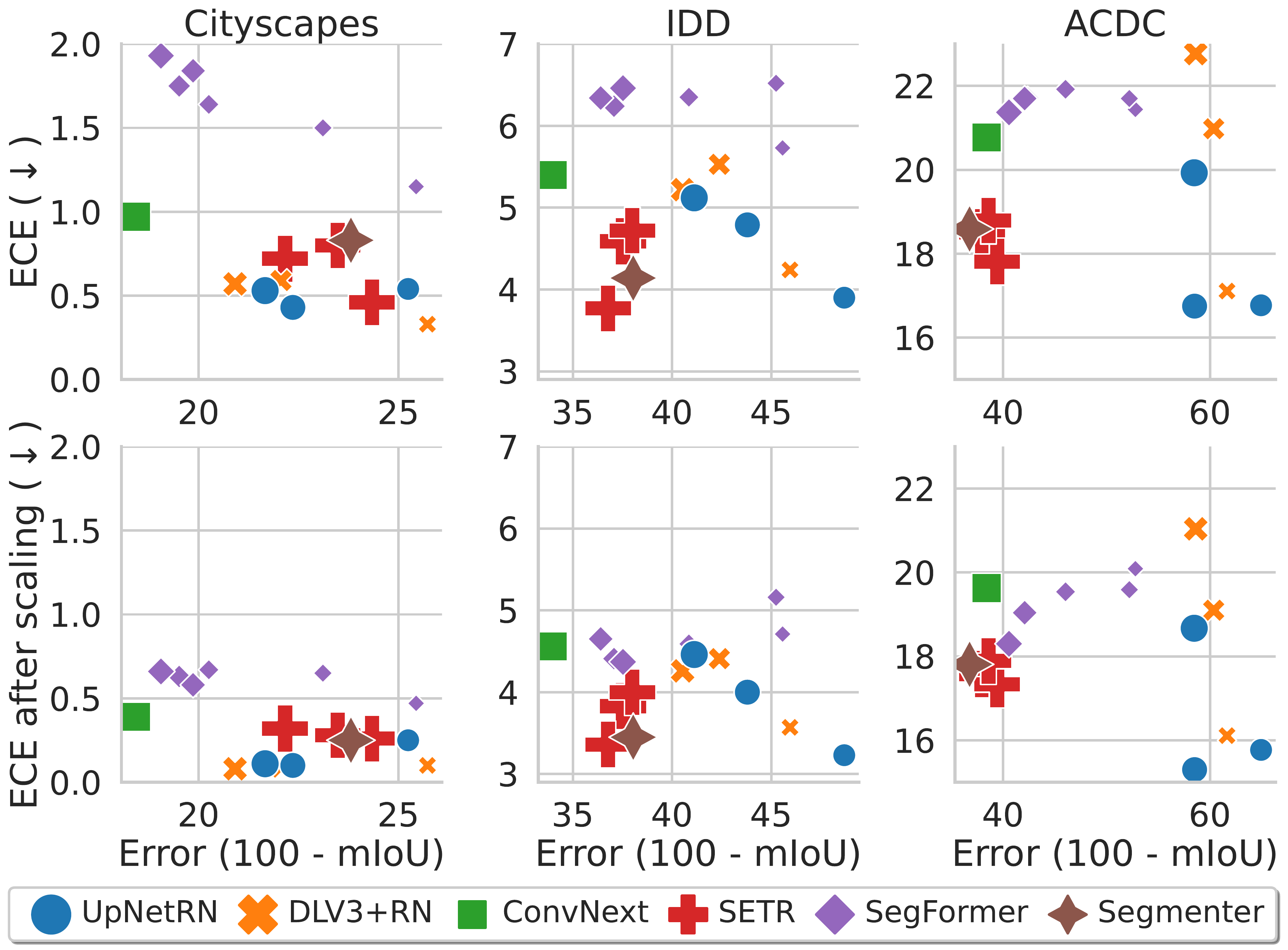}
\caption{
\textbf{Expected mIoU error $(\downarrow)$ \vs calibration error $(\downarrow)$}
before and after \ts for different model families (top and bottom, respectively). All models trained and calibrated on CS. Markersize proportional to number of parameters. Notice how \ts yields marginal \textit{relative} gains in \ood settings (IDD and ACDC).
}
\label{figure:miou_vs_ece}
\end{figure}

Since the domain shifts we consider
are natural and not synthetically induced, there is no straightforward way to evaluate 
their strength. To this end, we establish an 
ordering for the severity of the shift
based on performance degradation of the ResNet baselines (DLV3+R101 and UPNetR101), which results in
CS $<$ IDD $<$ ACDC.
This aligns with a
qualitative evaluation 
of 
the different datasets
(see \cref{figure:splash}, bottom, for a few samples). 
In \cref{figure:splash} (top left) we present the mIoU error (\ie, $100 - \textrm{mIoU}$) for several models evaluated on the three datasets. 
To highlight the loss in performance as we increase the shift,
we normalize 
all
errors 
\wrt
the best performing CS model (ConvNeXt).
The trend is clear: the larger the domain shift, the larger the improvement brought by more recent segmentation models.

We expand on this
in \cref{figure:miou_vs_ece} (top), where we plot the mIoU error (on the x-axis) \vs the calibration error for all 
the models belonging to the different families (Sec.~\ref{sec:arch}).
The size of the markers is proportional to the number of parameters. Similarly to \cref{figure:splash}, we observe that the gap in mIoU between ResNet baselines and recent models grows larger as we increase the domain shift (\cf marker position on the x-axis).
Interestingly,
two of the models that perform best under ACDC's 
strong
shift (SETR and Segmenter) are not significantly better than the ResNet baseline on the training domain (CS). This
indicates
that, in our setting, 
only assessing \id performance can hide the real value of 
newly crafted models and, hence, it is important to evaluate architectures
out of domain in order to fully grasp their potential.
%

While there is no 
single model family that performs significantly better in all datasets,
we can reach the clear conclusion
that \textit{all recent models are significantly more robust than well established baselines under natural shifts}.

\subsection{Calibration error}

\myparagraph{Off-the-shelf calibration} In \cref{figure:splash} (top-right) we present the ECE for different models as we increase the domain shift. Similarly to the mIoU error, ECE is normalized by the 
best CS model (in terms of calibration).
Interestingly, 
despite the remarkable improvements in terms of robustness, recent models are not significantly better calibrated.
%
When moving from CS to ACDC, the \textit{relative} mIoU error increases by a factor $\sim$2$\times$ for recent models \vs $\sim$3$\times$ for ResNet baselines; yet, in terms of \textit{relative} calibration error, all models increase by a $\sim$40$\times$ factor. This clearly highlights the need for further advances in model calibration.

In \cref{figure:miou_vs_ece} (top) we show the ECE \vs mIoU error for all models and datasets. 
When it comes to 
calibration \vs robustness
trade-off, there is no 
clear winner among the model families we consider. 
Moreover, 
we do not observe a
clear trend between mIoU and ECE in any 
domain.

\myparagraph{Calibration with \ts} In \cref{figure:miou_vs_ece} (bottom) we present the same results after applying \ts~\cite{GuoICML17OnCalibrationOfModernNNs}, 
tuned on CS.
Comparing top and bottom,
we observe an overall improvement in calibration for all networks. In particular, SegFormer models (\textcolor{sb_violet}{$\Diamondblack$}), which had the largest ECE on CS and IDD, seem to benefit the most from \ts. Nevertheless, even after this improvment, \ood calibration error (IDD, ACDC) remains significantly larger than the \id one (CS) for all models.
Regarding ECE \vs mIoU out of domain, a mild trend emerges after \ts : 
For transformer models, better-calibrated models are also the most robust (before \ts, SegFormer (\textcolor{sb_violet}{$\Diamondblack$}) did not follow this trend). On the other hand, for ResNet baselines (\textcolor{sb_orange}{$\times$},\textcolor{sb_blue}{$\medbullet$}) the better-calibrated models are generally the most brittle.
%
%
Regarding ConvNeXt (\textcolor{sb_green}{$\blacksquare$}), although it is one of the 
best
performing models, it seems to be worse calibrated than other models with similar mIoU. Overall, \textit{recent models are not significantly better calibrated than ResNet baselines neither before nor after \ts}. 

\subsection{Can we improve out-of-domain calibration?}
\label{sec:improving_calibration}
Calibration error on samples from the training domain in not alarming,
especially after TS. Nonetheless, the sharp increase in ECE out of domain is concerning for many applications, especially since recent segmentation models do not show a clear improvement in this direction. This renders methods that seek to improve calibration out of domain all the more important, but yet this is a rather underexplored research area. 
As discussed in \cref{sec:related_work}, 
to the best of our knowledge, only Gong~\etal~\cite{gong2021confidence} tackle \ood calibration \textit{without} additional information about the test domain. They suggest clustering the calibration set into multiple ``domains'' based on the image features extracted by the network. A different temperature per cluster is then selected and test-time predictions are scaled according to the cluster assigned to the images. This \textit{adaptive} \ts method was originally devised for classification, but we extend it to the segmentation task by scaling all the logits of a given image with the same temperature. Regarding the number of clusters, we find 16 to be a reasonable number (see \cref{sec:ablation_num_clusters} for this 
analysis).


\begin{figure}[t]
\centering
\includegraphics[width=\linewidth]{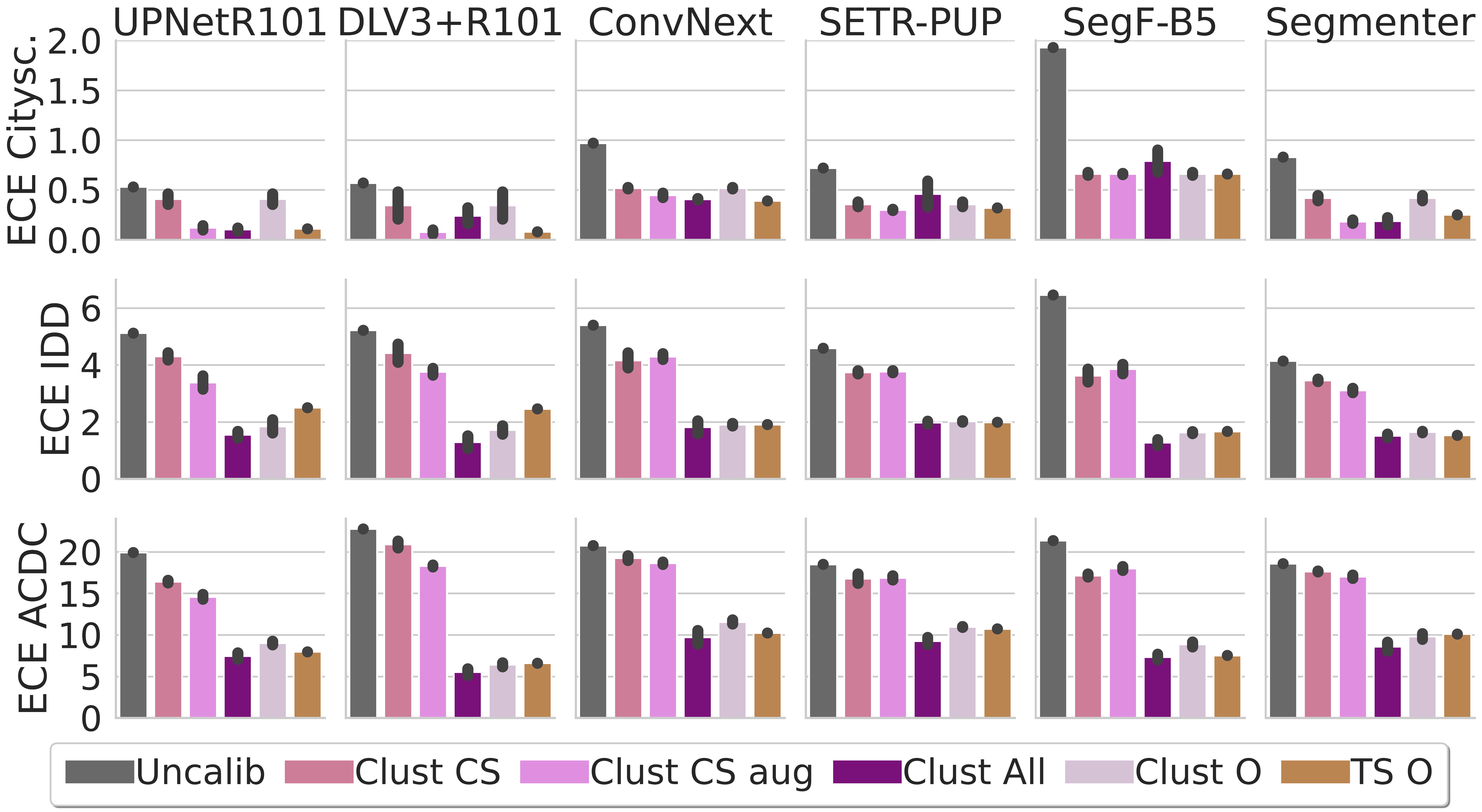}
\caption{\textbf{ECE $(\downarrow)$ after clustering \ts} for a selection of models (best mIoU on CS per family). If the calibration samples are representative of the test domains \textit{Clust All}
can indeed improve ECE, however, without access to \ood samples (\textit{Clust CS} and \textit{Clust CS aug.}) benefits of clustering are limited. Oracle baselines (\textit{O}) always use calibration images from the test domain.
}
\label{figure:cluster_ablation}
\end{figure}

\myparagraph{Clustering on different calibration sets} \label{sec:clustering}
In \cref{figure:cluster_ablation} we present the results of calibrating with clusters computed on different calibration sets. Since our training dataset is CS, we assume that only CS images are available for calibration. As a naive alternative to obtain more diverse clusters, we introduce a \textit{CS aug.} dataset where some of the CS calibration images are randomly augmented with different transformations (\eg, color scaling, changes in brightness, contrast, etc). 
The rationale is that more diverse clusters might generalize better to new domains.

To assess how beneficial \ood samples can be during calibration, we also add another calibration set which contains images 
from
all the datasets (CS, IDD and ACDC) mixed together, we will refer to it as \textit{All}.
This serves as a sort of upper bound, since our main goal is still assessing robustness in unseen domains.
Furthermore, we introduce two more oracle baselines (\textit{O}), which use calibration images from the test domain. For instance, when evaluating on IDD, the oracle calibration set will consist of \textit{only} IDD while \textit{All} will contain images of IDD, ACDC and CS mixed. One oracle baseline uses clustering, while the other uses vanilla \ts (\textit{Clust O} and \textit{TS O}, respectively).

As expected, 
calibrating on all datasets (\textit{Clust All}) significantly improves ECE, with comparable performance to oracles in most settings. In contrast, without access to \ood samples (\textit{Clust CS}), calibrating with the method by Gong~\etal yields rather limited improvements. Moreover, increasing cluster diversity via data augmentation (\textit{Clust CS aug.}) is not always beneficial.

\begin{figure}[t]
\centering
\includegraphics[width=\linewidth]{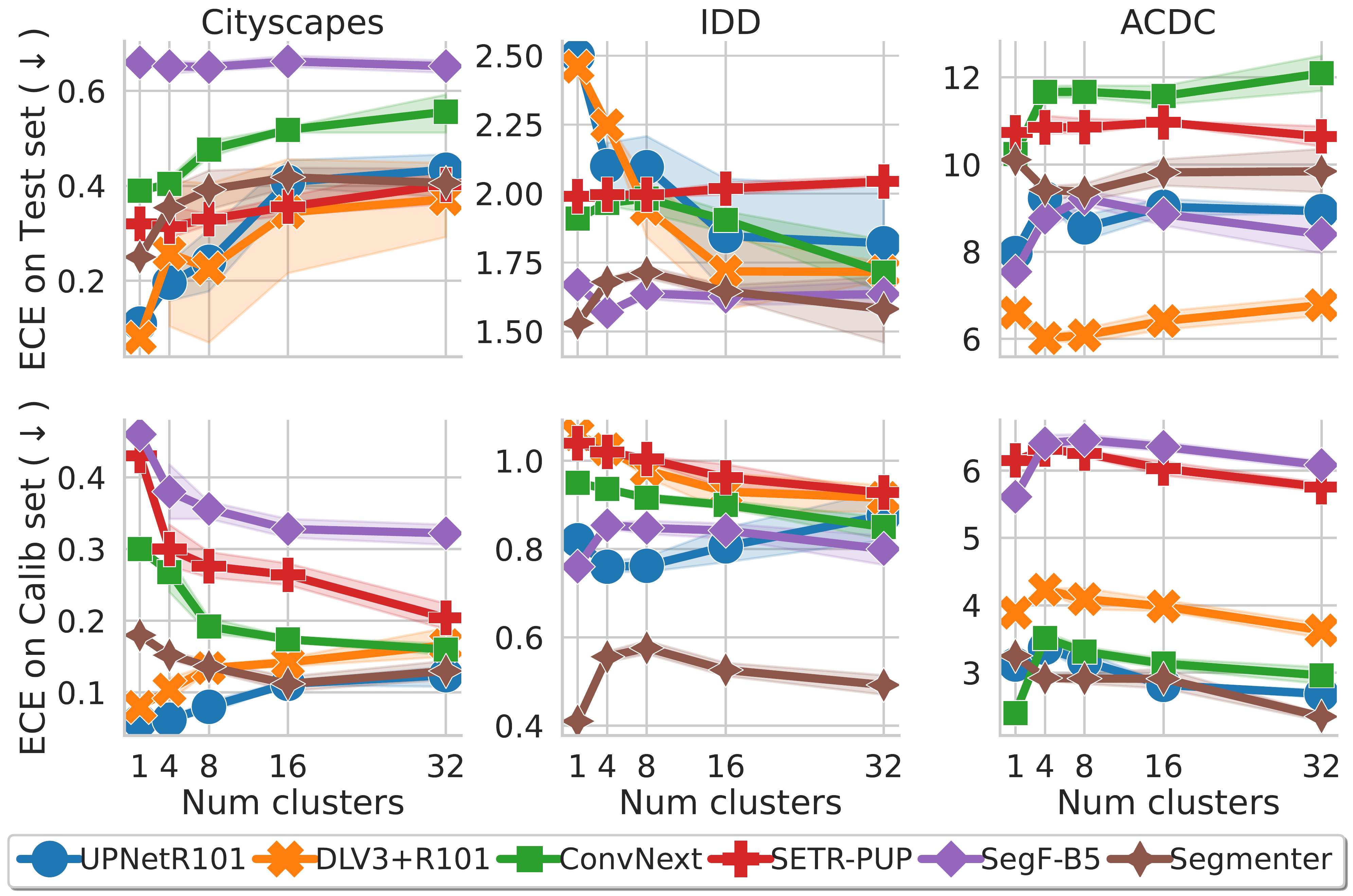}
\caption{\textbf{ECE $(\downarrow)$ \vs number of clusters for oracles} 
(calibration images 
extracted 
from the test domain). 
Even when evaluating on the calibration set, where there can be no overfitting of the temperatures,
ECE does not decrease monotonically 
as we increase
the number of clusters (one cluster is equivalent to vanilla TS).
}
\label{figure:val_set_clustering}
\end{figure}

To gain more intuition, we visualize the cluster assignments (see \cref{sec:visualization_cluster_samples}). When using all datasets for calibration, test-time images are qualitatively close to the assigned clusters; yet, with \textit{Clust CS} or \textit{Clust CS aug.}, \ood images do not blend well with the calibration images of their corresponding clusters. We argue that one implicit assumption for clustering to work well is that test-time images are close to one of the clusters (domains) in the calibration set; therefore, under strong domain shifts, it is unlikely to bring much improvement.\looseness=-1

On the bright side, if representative images from the deployment domains are available, clustering could be applied to allow a single model to be calibrated on multiple domains. Of course, with \ood annotated samples, additional fine-tuning or adaptation techniques could be applied, but this is out of the scope of this study, since our focus is \textit{off-the-shelf} model robustness, without adaptation.\looseness=-1

\myparagraph{Clustering does not improve ECE in domain} Comparing oracles in~\cref{figure:cluster_ablation}, we can observe that clustering does not improve significanly over \ts. In some settings, it is even worse. One possible explanation would be that this is due to an overfitting of the temperature parameters to the particular calibration clusters. However, in \cref{figure:val_set_clustering} we observe that even when evaluating the ECE in the calibration set, the error does not monotonically decrease with the number of clusters. Although somewhat surprising, this is in fact possible since decreasing the ECE for several disjoint subsets of images (clusters) independently does not guarantee that the ECE on the union set will decrease. We provide a formal theorem in \cref{sec:proof_subset_ECE}, in support of this claim. Note that we are not asserting that clustering will not improve ECE \textit{in general} (we empirically observed it can, if provided with a representative calibration set) but rather that it is not guaranteed to do so.

To sum up, \textit{clustering the calibration set does not bring significant improvements unless representative images of the test domain are present in the calibration set}. Moreover, \textit{it is not better than \ts for \id calibration.}\looseness=-1

\begin{figure}
\centering
\includegraphics[width=\linewidth]{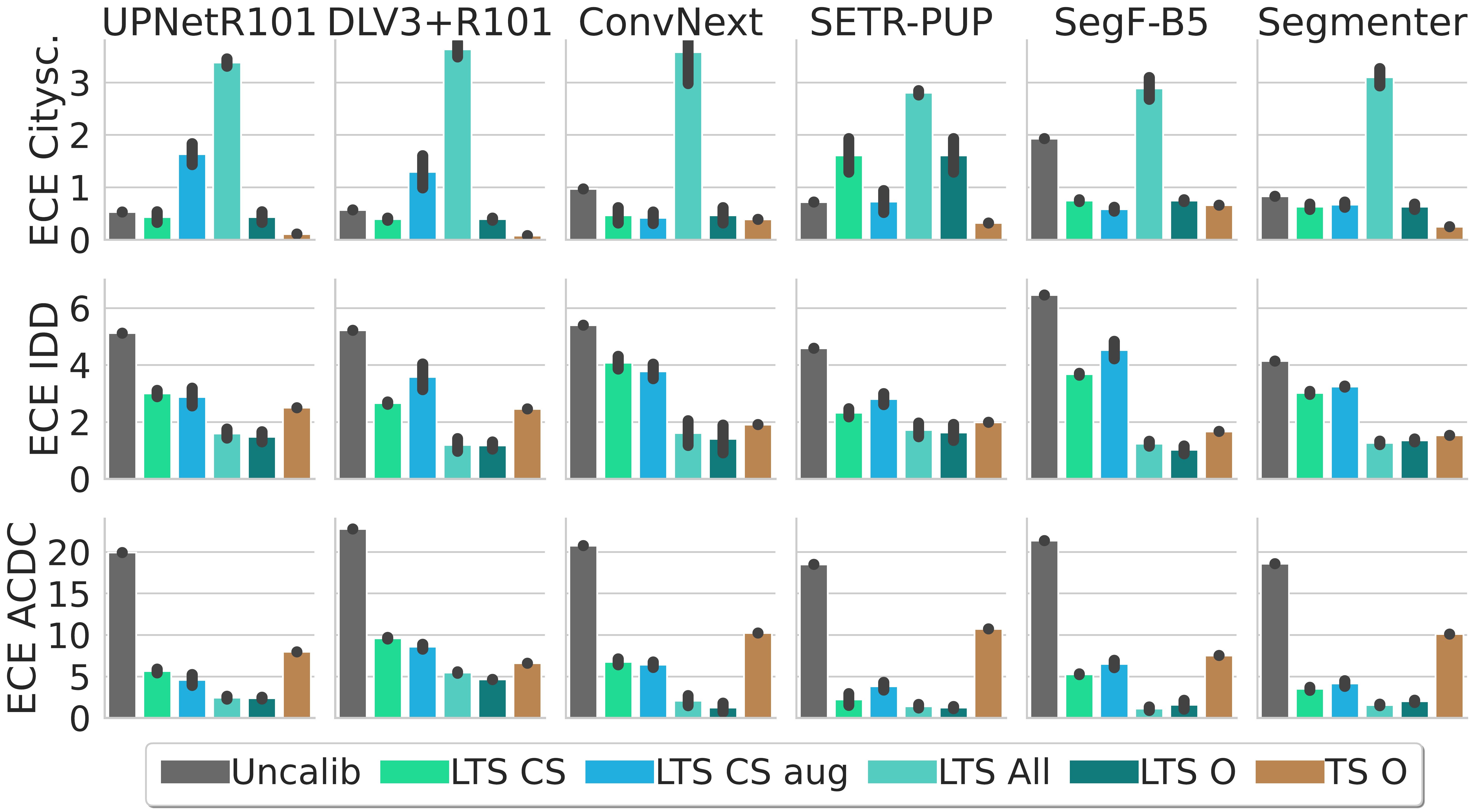}
\caption{\textbf{ECE $(\downarrow)$ after local \ts (LTS)} for a selection of models (best mIoU on CS per family). Even without access to \ood samples (\textit{LTS CS}), LTS calibration improves ECE out of domain, especially under strong domain shifts (ACDC). With access to \ood samples (\textit{LTS All}), ECE out of domains improves further, albeit it degrades in domain.
}
\label{figure:LTS_ablation}
\end{figure}

\subsubsection{Adaptive temperature via calibration network} 
The partial failure of the clustering approach motivates us to investigate other methods that adjust the temperature adaptively \wrt the input, since we 
can expect this to help improving
\ood calibration. In 
this
regard, Ding~\etal~\cite{ding2021local} suggest 
training
a small
calibration network that predicts the 
temperature values
as a function of both the
input image and the segmentation model logits. This Local Temperature Scaling
(LTS) method is specific to segmentation so the output is not a single
temperature per image but a ``temperature map'' with 
pixel-level temperature values.
Despite not being designed
for \ood conditions,
our intuition is that a network providing sample-dependent temperatures can be 
beneficial under domain shifts.

\begin{figure}[t]
\centering
\includegraphics[width=\linewidth]{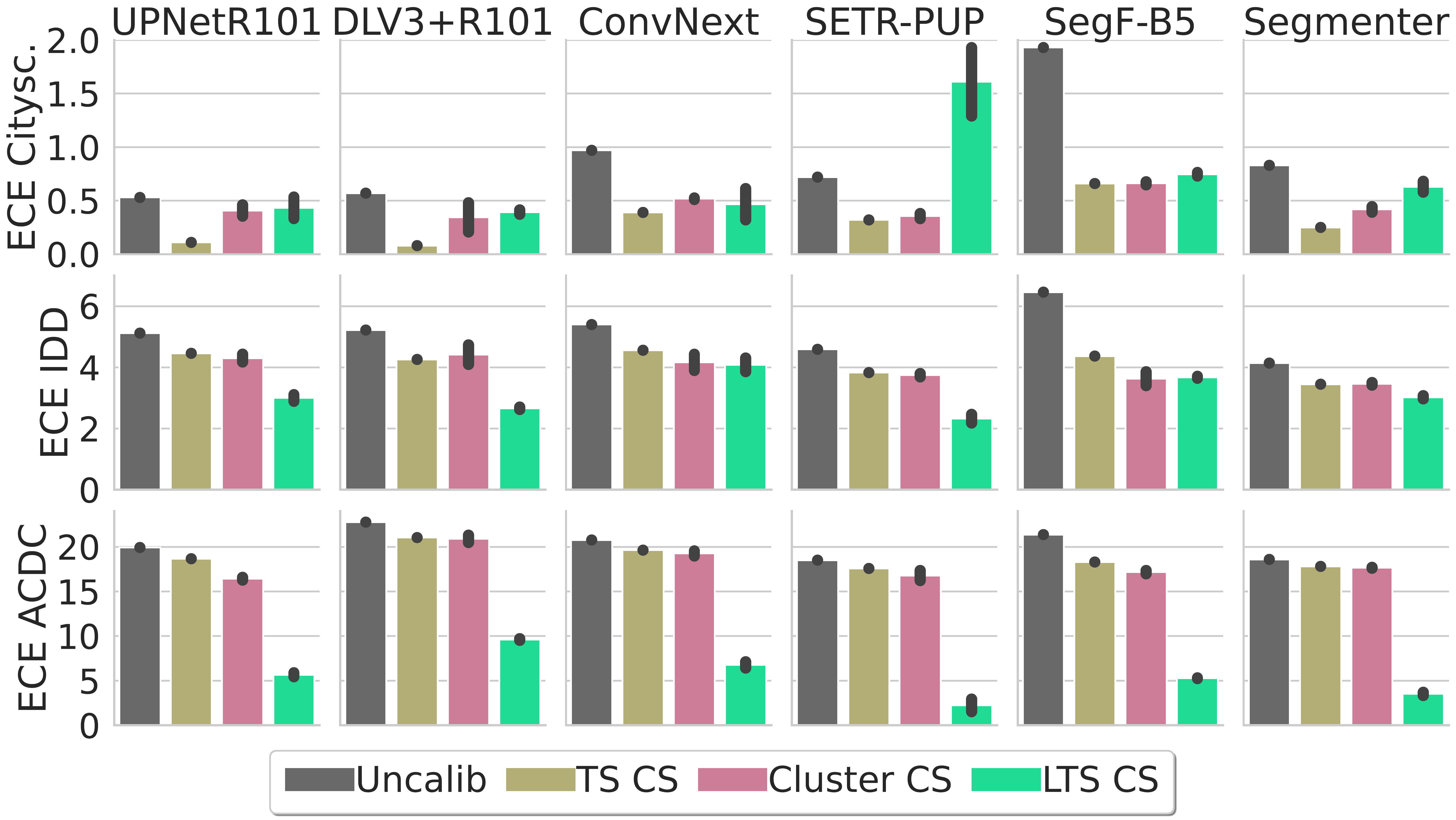}
\caption{\textbf{Comparison of calibration methods.} ECE $(\downarrow)$ after calibration for a selection of models (best mIoU on CS per family). All models are calibrated on CS calibration set. LTS is markedly the best calibration method out of domain with remarkable improvements under strong domain shifts.
}
\label{figure:calib_methods_comparison}
\end{figure}

\myparagraph{LTS using different calibration sets} \label{sec:LTS} 
As in \cref{sec:clustering}, we test LTS on multiple calibration sets (\textit{CS}, \textit{CS aug.}, \textit{All}), which in this case are used to learn the calibration network. 
Also here, we compare against oracles, 
for which the calibration network is learned using images from the test domains.
Results are shown in \cref{figure:LTS_ablation}. Interestingly, we find that LTS using only CS images for calibration (\textit{LTS CS}) leads to 
a noticeable improvement in \ood ECE.
In particular, when testing on ACDC---where the 
domain shift is stronger---\textit{LTS CS} outperforms even \ts with access to images on the test domain (\textit{TS O}) for some models. Also 
in these experiments,
introducing naive data augmentations on CS (\textit{LTS CS aug.}) does not yield substantial improvements.

When using all the datasets for calibration (\textit{LTS All}), \ood results improve even further; yet, there is a noticeable increase in calibration error on CS. Unlike clustering, where the temperature was optimized independently for each cluster, LTS trains the calibration network using all the samples at once and samples with large calibration error (like ACDC or IDD) 
may
dominate the loss.
We hypothesize that further improvements in the architecture and training schedule of the calibration network can lead to even better performance and are promising directions 
for
future work. 

Focusing on
the oracle baselines, LTS outperforms \ts on IDD and
ACDC, but \ts outperforms LTS on CS (\cf \textit{LTS O} and \textit{TS O}). This may be due to the fact that CS is a very 
homogeneous
dataset if compared to the other two, hence, it 
can be reasonable 
that a simpler method may perform best.

\myparagraph{Comparing all calibration methods} \label{sec:comparison_calib_methods}
In \cref{figure:calib_methods_comparison} we compare 
\textit{TS CS}, \textit{Clust CS} and \textit{LTS CS} (all calibrated on CS).
LTS is markedly the best calibration method out of domain, especially under stronger domain shifts. In domain, \ts works best, but it does not bring significant improvements out of domain. Since LTS predicts the temperature parameter at the pixel level, this motivates an ablation of clustering where we predict different temperatures per image; yet this does not improve results (see details in \cref{sec:per_class_clustering}). Additionally, we perform an ablation of LTS using only the image or the logits for calibration. Image information seems to be more important for \ood calibration, while the logits are more important in the \id setup (see \cref{sec:LTS_image_logits_ablation}).

\begin{figure}[t]
\centering
\begin{subfigure}[t]{\linewidth}
\includegraphics[width=\linewidth]{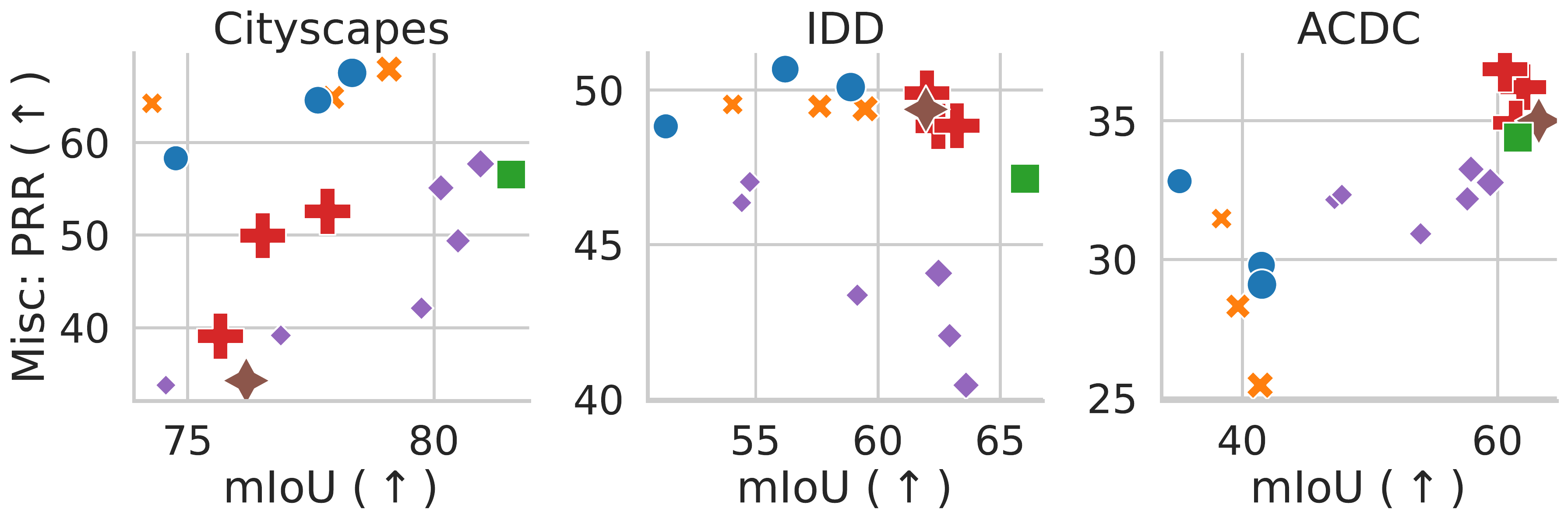}
\end{subfigure}

\vspace{0.3cm}

\begin{subfigure}[b]{\linewidth}
\includegraphics[width=\linewidth]{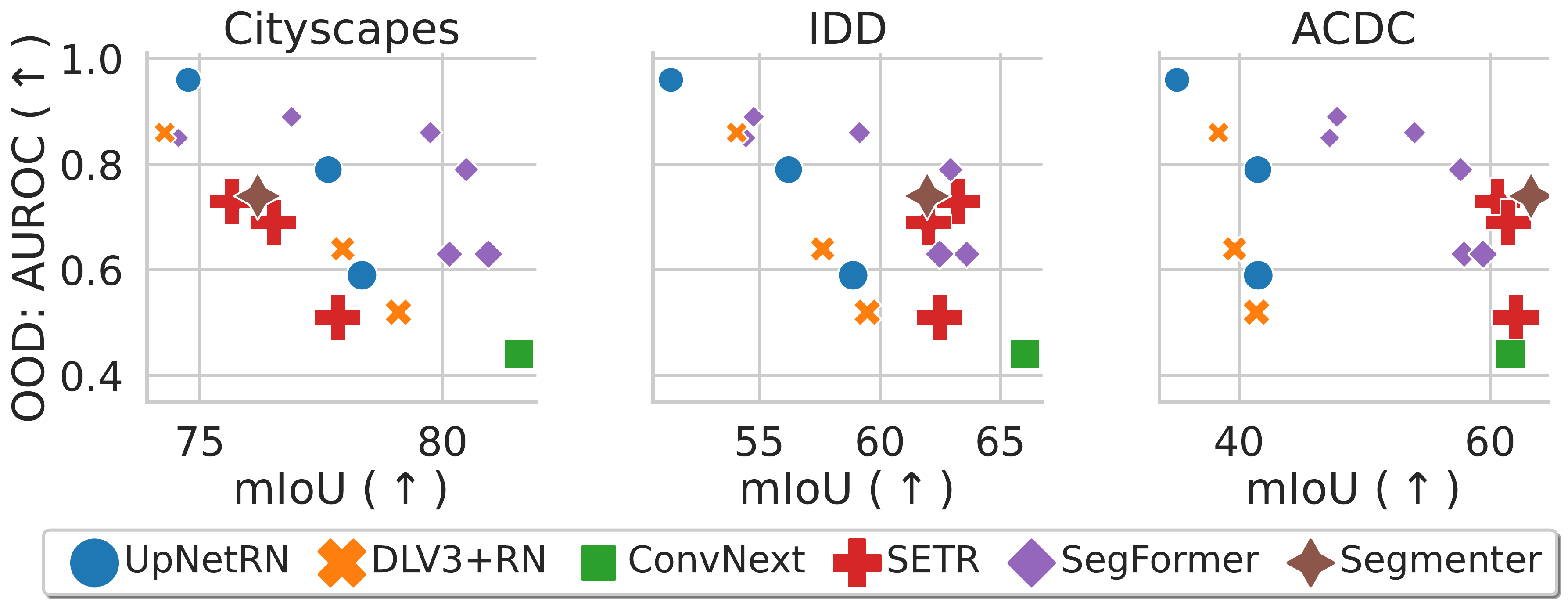}
\end{subfigure}

\caption{\textbf{Misclassification and \ood detection \vs} 
\textbf{Robustness} for different segmentation models and datasets.
%
\textbf{(Top) PRR $(\uparrow)$ \vs mIoU $(\uparrow)$:} 
ResNet-based models (\textcolor{sb_blue}{$\medbullet$},\textcolor{sb_orange}{$\times$}) outperform more recent models (other markers) in \id misclassification detection (CS, left), but the trend is opposite under strong domain shifts (ACDC, right). 
\textbf{(Bottom) AUROC $(\uparrow)$ \vs mIoU $(\uparrow)$:} There is no free-lunch between robustness and \ood detection in any considered domain.
}

\label{figure:misc_ood_all_models}
\end{figure}

\subsection{Misclassification detection} \label{sec:misclassification}

In \cref{figure:misc_ood_all_models} (top) 
we compare all models in terms of \textit{misclassification detection} \vs \textit{robustness}
---PRR score $(\uparrow)$ \vs mIoU $(\uparrow)$.
In domain, we observe a clear trend:
within the same model family, better performing models tend to also show better PRR. However, 
when considering all models,
higher mIoU does not generally imply higher PRR and ResNet-based backbones perform significantly better
than more recent architectures.
As we increase the domain shift, the trend changes: for ACDC, recent models perform best both in terms of mIoU and PRR. Moreover, out of domain, ResNet families show a negative correlation where better mIoU leads to worse PRR. 
Overall, \textit{recent 
models seem to improve misclassification detection under strong domain shifts, but underperform baselines in domain.}

\subsection{Out-of-domain detection} \label{sec:OOD}

In \cref{figure:misc_ood_all_models} (bottom) we compare 
\textit{\ood detection} \vs \textit{robustness}
(AUROC score \vs mIoU) for all models. To measure \ood detection we separate the \id images (CS) from the \ood ones (IDD and ACDC). Therefore, the y-axis is the same in all three plots and only the mIoU changes. For \ood detection, there is a marked negative trend between CS mIoU and AUROC. When looking at IDD and ACDC, the negative trend continues but there seems to be a distinction between ResNet baselines and other recent models: 
the latter
perform better in terms of mIoU, but at the same time show
a drop in AUROC.
In short, \textit{there is no free-lunch between robustness and \ood detection. In terms of \ood detection, a small ResNet-18
(\textcolor{sb_blue}{$\medbullet$})
performs best.
}

\begin{figure}
\centering
\begin{subfigure}[t]{\linewidth}
\includegraphics[width=\linewidth]{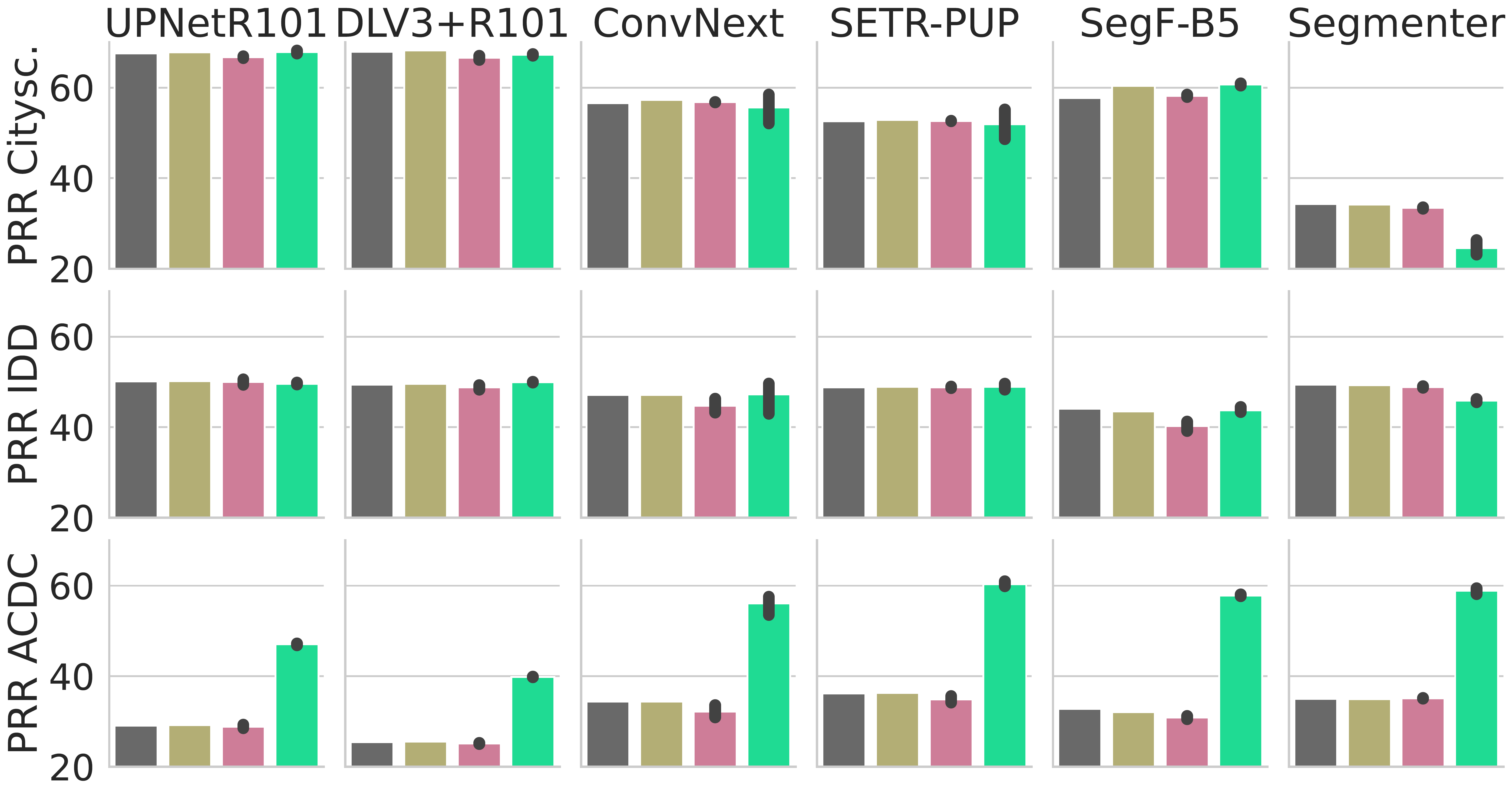}
\end{subfigure}

\vspace{0.3cm}

\begin{subfigure}[b]{\linewidth}
\includegraphics[width=\linewidth]{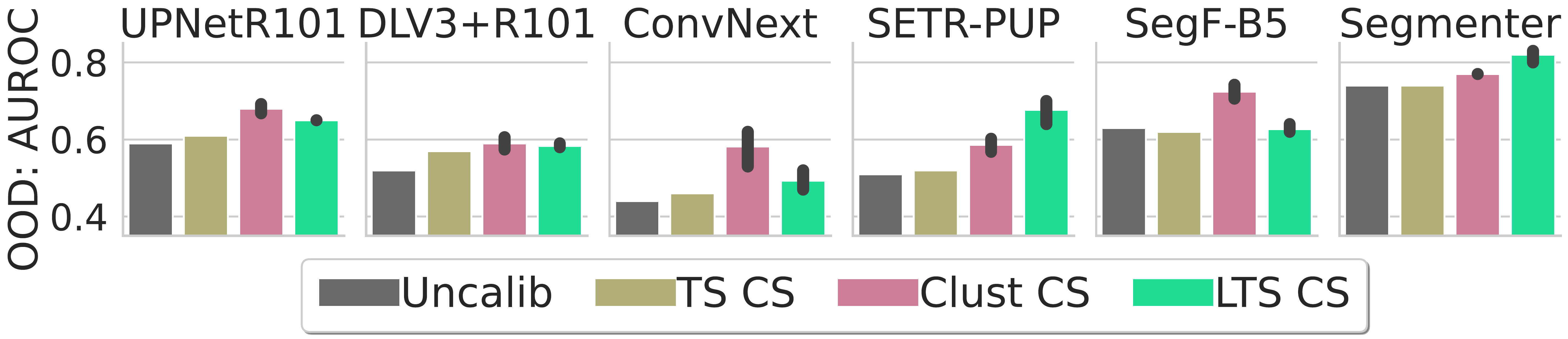}
\end{subfigure}

\caption{\textbf{Misclassification and \ood detection after calibration} for several models after applying different calibration techniques using CS samples.
\textbf{Misclassification -- PRR~$(\uparrow)$ (first 3 rows)}: Under strong domain shifts (ACDC), LTS calibration significantly improves PRR. 
\textbf{\ood detection -- AUROC~$(\uparrow)$ (last row)}: Both clustering and LTS yield improvements in \ood detection.
}
\label{figure:misc_ood_after_calibration}
\end{figure}

\subsection{Can calibration improve misclassification and out-of-domain detection?} \label{sec:improve_misc_ood}
In \cref{figure:misc_ood_after_calibration} we show misclassification (top) and \ood (bottom) detection metrics for different models after calibration using only CS samples. For misclassification, we observe a sharp PRR improvement on ACDC after we calibrate models with LTS. This is 
encouraging
as it indicates that the calibration network learned in LTS can help discern correct from incorrect predictions given its output temperature. We do not observe significant improvements in other datasets or with other methods. This is reasonable, since the largest calibration gain was 
observed with
LTS on ACDC.
(see \cref{figure:calib_methods_comparison}).

Regarding \ood detection (\cref{figure:misc_ood_after_calibration} bottom) we observe that both clustering and LTS calibration can improve \ood detection. 
We find this interesting,
since clustering on CS did not improve \ood calibration significantly. Although the clusters using only CS images are not representative enough to produce adequate temperatures for IDD or ACDC, \ood samples are assigned to clusters which have larger temperatures. This is enough to decrease the confidence for \ood samples compared to \id and leads to better \ood detection. Similarly for LTS, the calibration network assigns larger temperatures to \ood images. 

In conclusion, \textit{adaptive \ts techniques are a promising avenue to improve \ood detection and misclassification detection under strong domain shifts.}

\section{Conclusion}
\label{sec:conclusion}

We have studied the reliability of recent segmentation models---in terms of \textbf{robustness} and \textbf{uncertainty estimation} under natural domain shifts. Overall, while no single model family is better in all scenarios, recent models are remarkably more robust to domain shifts than ResNet baselines. Yet, this does not translate into better \textit{calibration}---severely degraded out of domain. Thus, it is crucial to find methods to improve model calibration in \ood settings. 
To this end, we have explored state-of-the-art methods and found that Local Temperature Scaling \cite{ding2021local}, although originally devised for \id settings, is a promising 
technique.

Furthermore, we have explored \textit{misclassification} and \textit{\ood detection}---two other important tasks regarding uncertainty estimation. We have shown that recent and more robust models tend to perform better at misclassification under strong domain shifts, 
but yet
they underperform ResNet baselines \id. On the other hand, \ood detection under domain shifts is negatively correlated with the mIoU, which translates into a trade-off between robustness and uncertainty. 
Finally, we find that adaptive temperature scaling techniques can help beyond calibration and improve \ood detection and misclassification in some settings. 

All in all,
although we appear to be \textit{on the right track} for what concerns robustness, our findings motivate the need to improve reliability of segmentation models in other dimensions, where results are not equally postive. In that regard, we identify several promising directions which we hope may encourage future research on this important topic.

\paragraph{Acknowledgements}
\noindent We would like to thank Francesco Pinto and Gabriela Csurka for helpful discussions. Prof. Philip Torr is supported by the UKRI grant: Turing AI Fellowship EP/W002981/1 and EPSRC/MURI grant: EP/N019474/1. We would also like to thank the Royal Academy of Engineering and FiveAI.

{\small
\bibliographystyle{ieee_fullname}
\bibliography{egbib}
}

\clearpage
\appendix
\onecolumn
\section{Ablation of calibration metrics}
\label{sec:ablation_calibration_metrics}
In the main paper we present calibration results computing the Expected Calibration Error with equally spaced bins, however, alternative calibration metrics have been suggested. In \cref{figure:ablation_calibration_metrics} we compare the results obtained with: ECE with equally spaced bins (ECE), ECE with equally populated bins (Ada ECE) and the Kolmogorov-Smirnov Error (KS Error). For further details see \cref{sec:reliability_metrics}. We observe that the three different metrics yield almost identical results.

\begin{figure}[ht]
\centering
\includegraphics[width=0.9\linewidth]{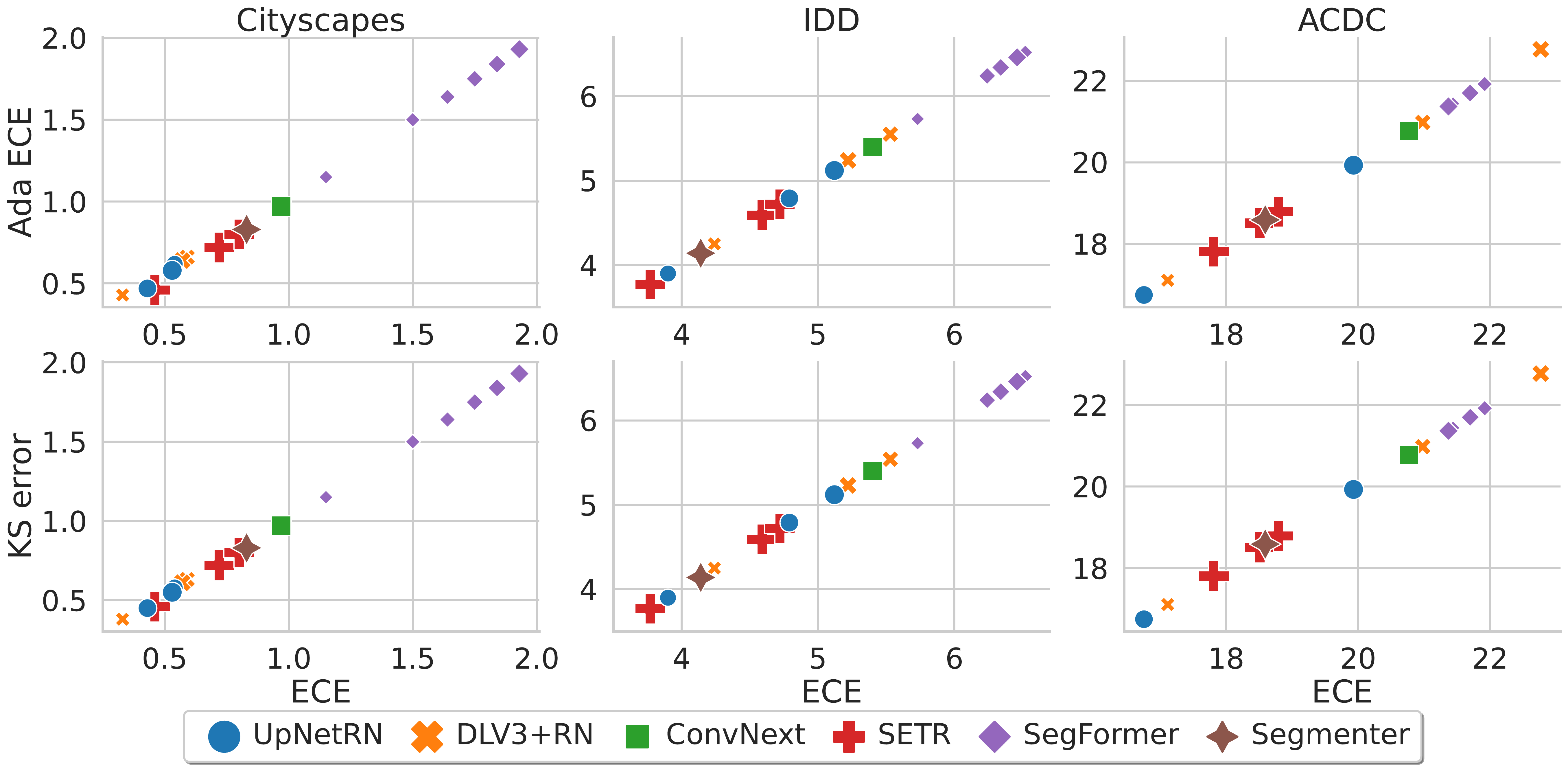}
\caption{\textbf{Comparison of calibration error metrics $(\downarrow)$} Calibration error for different datasets and networks computed with different metrics. All different metrics yield very similar result. 
}
\label{figure:ablation_calibration_metrics}
\end{figure}

\section{Ablation of number of pixels for calibration}
\label{sec:ablation_number_pixels_calibration}
As discussed in \cref{sec:reliability_metrics}, in segmentation, the number of 
samples to be taken into account for calibration scales with the number of pixels in an image. In order to be more cost-effective when testing different calibration metrics and strategies, we use a random subset of pixels within each image rather than the full image. In \cref{figure:ablation_num_samples_calibration}, we ablate the evolution of the different calibration metrics as we vary the number of sampled pixels. We can see that from 10k datapoints on, the metrics stabilizes; therefore, we chose to use 20k randomly sampled pixels per image for our experiments.

\begin{figure}[ht]
\centering
\includegraphics[width=0.6\linewidth]{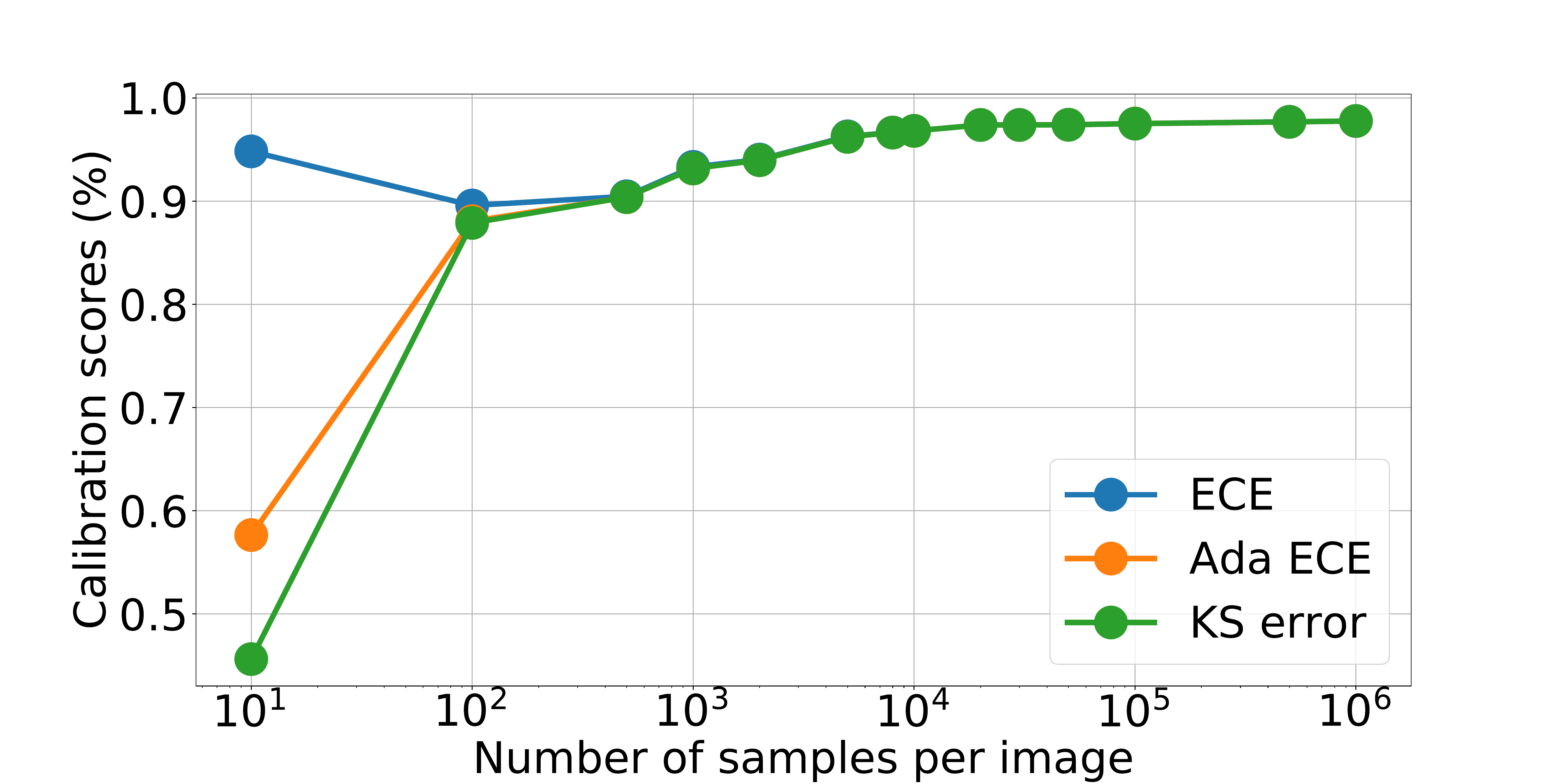}
\caption{\textbf{Pixels-per-image ablation.} Evolution of calibration metrics as we vary the number of pixels sampled at random from each image (as opposed to the full image). We observe that when sampling more than 10k pixels all calibration metrics are very similar and the calibration error remains stable. We use 20k random samples in our experiments.
}
\label{figure:ablation_num_samples_calibration}
\end{figure}

\section{Ablation of confidence score: max probability \vs entropy}
Misclassification detection and \ood detection both rely on a metric to evaluate how confident a model is on its predictions. The most straightforward metric would be the pseudo-probability of the predicted class (i.e. the max probability). If the probability is high it is reasonable to assume that the network is confident (this is precisely what we want to impose in the calibration task). Other metrics which involve all the logits have been suggested, negative entropy being the most popular. In \cref{figure:ablation_prob_vs_entropy} we compare the results obtained with probabiliy and entropy as confidence metrics and observe that there is not a significant difference between the two. Therefore, the simpler confidence metric based on the predicted class probability is used for other experiments by default.
\label{sec:ablation_prob_vs_entropy}
\begin{figure}[ht]
\centering
\begin{subfigure}{\linewidth}
\centering
\includegraphics[width=0.9\linewidth]{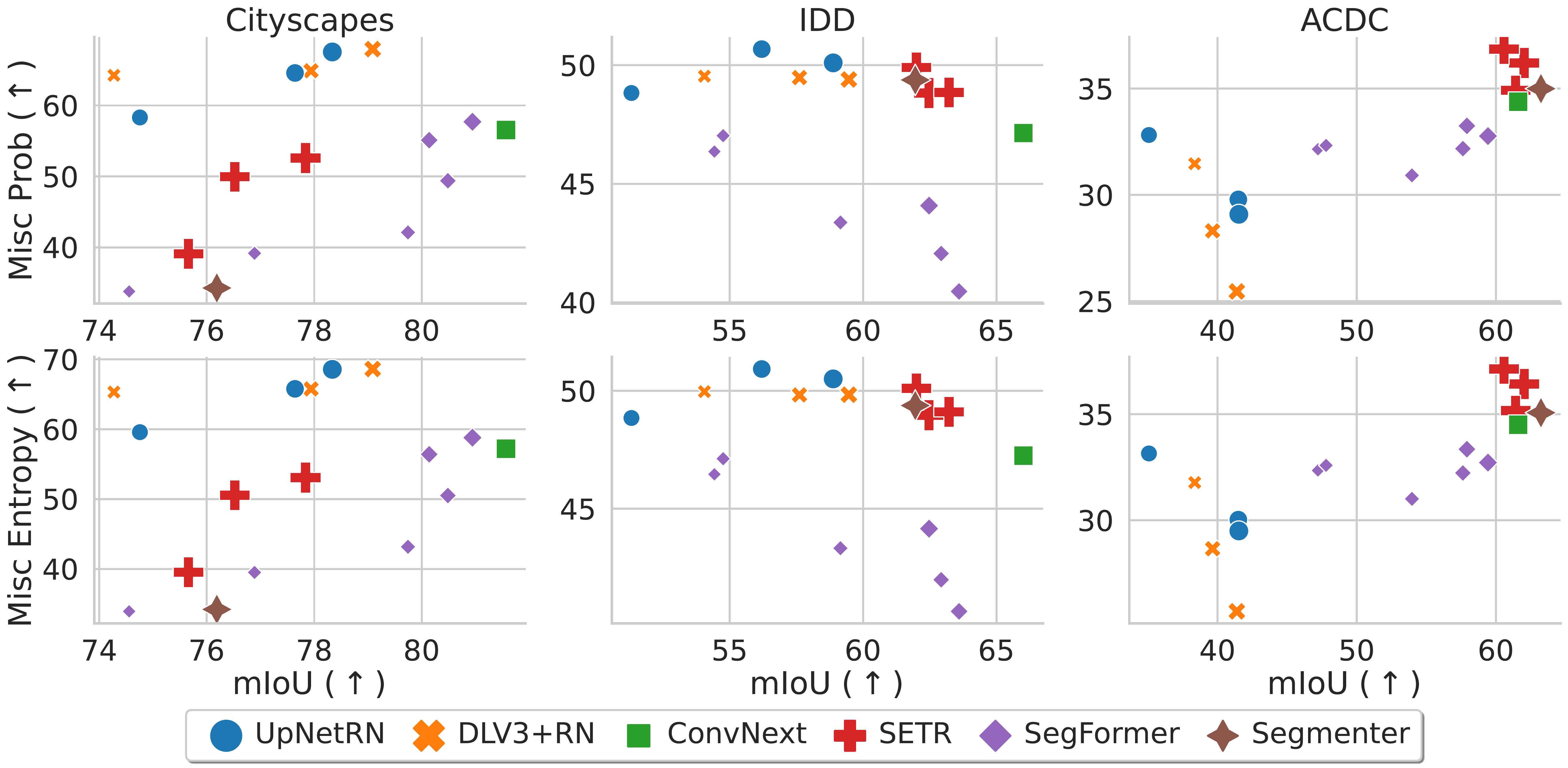}
\end{subfigure}

\vspace{0.3cm}

\begin{subfigure}{\linewidth}
\centering
\includegraphics[width=0.9\linewidth]{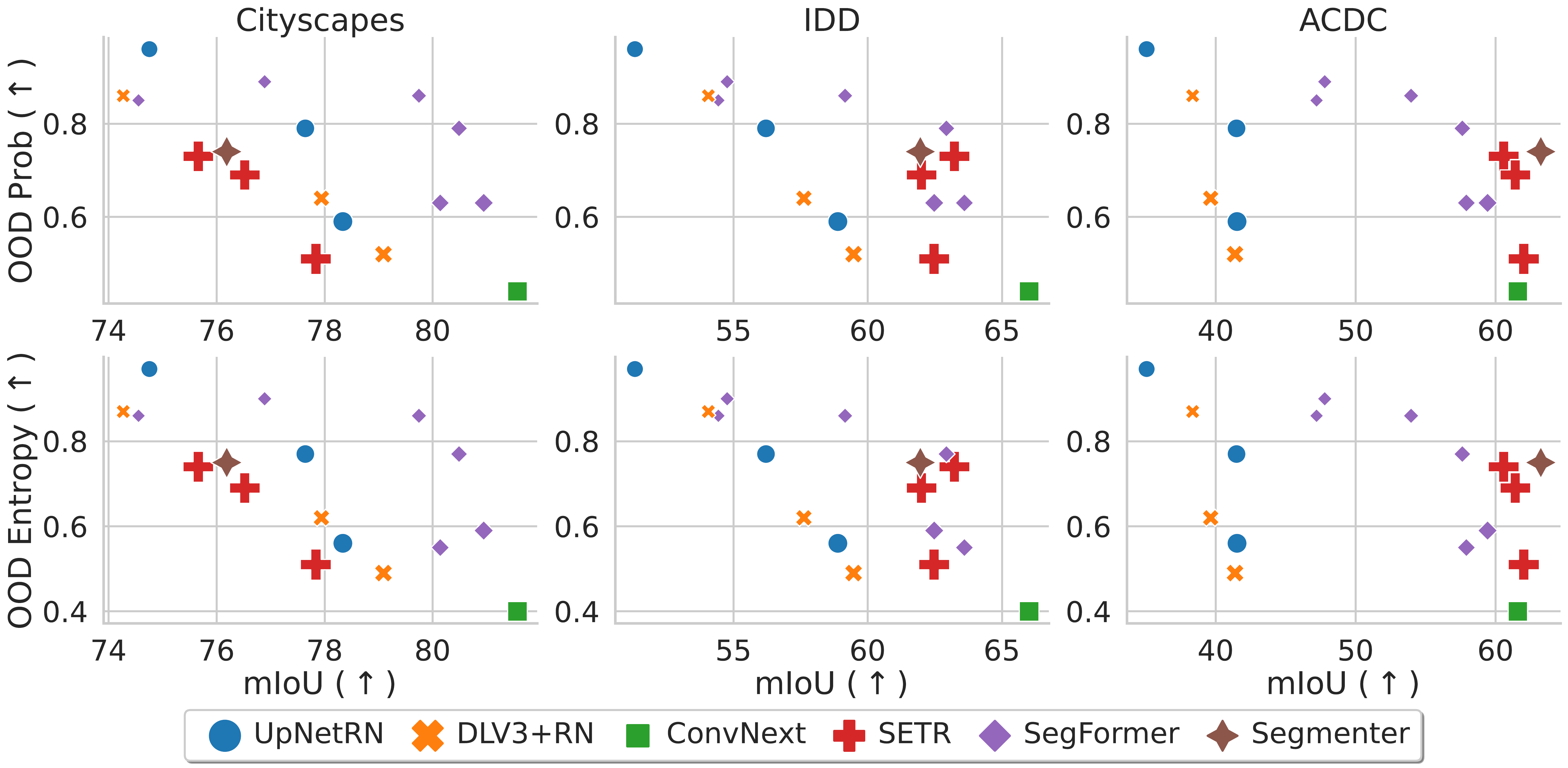}
\end{subfigure}
\caption{\textbf{Ablation of confidence score: max probability \vs entropy} Comparison of misclassification (top) and \ood (bottom) detection when using probability or negative entropy as confidence metrics. We observe that there is no significant difference between the two metrics, therefore we use the simpler probability as the default.
}
\label{figure:ablation_prob_vs_entropy}
\end{figure}

\section{Ablation number of clusters}
One of the main hyperparameters in Gong~\etal~\cite{gong2021confidence} is the number of clusters. In \cref{figure:ablation_num_clusters}, we ablate the number of clusters for different test datasets (columns) and calibration datasets (rows). Although not all networks evolve in the same way, we observe that after 16 clusters, performance is more or less stable.

\label{sec:ablation_num_clusters}
\begin{figure}[ht]
\centering
\includegraphics[width=0.9\linewidth]{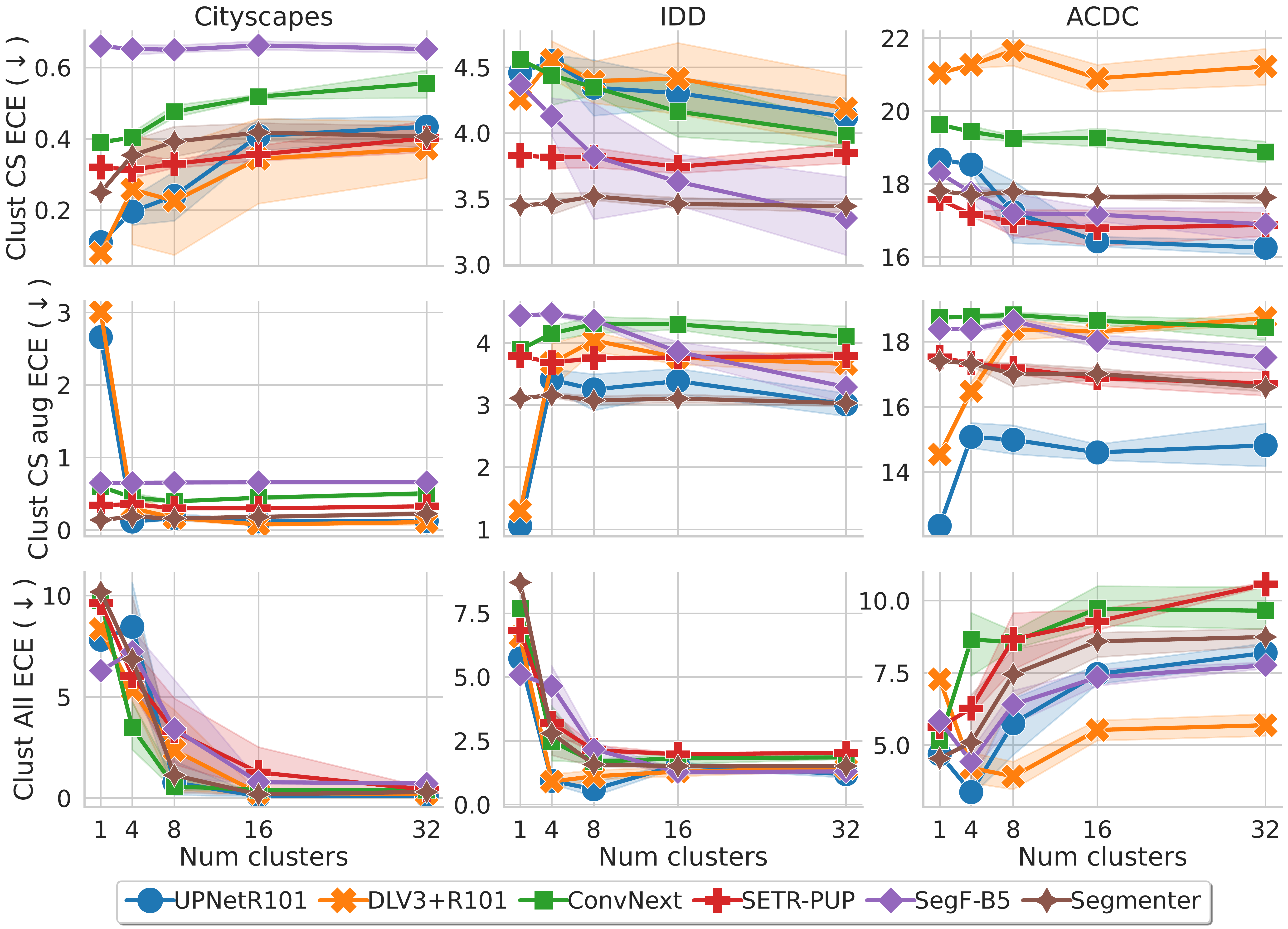}
\caption{\textbf{Ablation number of clusters.} ECE $(\downarrow)$ for different models and datasets as we vary the number of clusters computed for calibration. We find 16 clusters to be relatively stable.
}
\label{figure:ablation_num_clusters}
\end{figure}

\section{Visualization of cluster samples}
\label{sec:visualization_cluster_samples}
In \cref{sec:improving_calibration} we observe that adaptive temperature scaling via clustering does not significantly improve calibration under distribution shift -- especially when the shift is strong. 
The method by Gong~\etal~\cite{gong2021confidence} makes the implicit assumption that the different domains captured in the clusters during calibration will be representative of the domains encountered at test time. In order to have a better intuition, we visualize a few samples randomly picked from each cluster. We show images from both the calibration set (used to compute the clusters and calibrate the models) and the test set (used to evaluate the calibration error). In our visualizations, the test set comprises images of the three datasets (CS, IDD and ACDC), while the calibration set changes for each Figure. In~\cref{figure:cluster_viz_All_dsets,figure:cluster_viz_Cityscapes,figure:cluster_viz_CS_aug}, respectively, we show representatives from clusters in \textit{Clust All} (all datasets used during calibration), \textit{Clust CS} (CS images used) and \textit{Clust CS aug} (augmented CS images used). Qualitatively, when all datasets are used for calibration, the cluster assignments appear quite reasonable (\eg night ACDC images are assigned to night images from calibration). However, when calibrating on CS and CS augmented, we observe that the calibration clusters are not diverse enough for the test images and the cluster assignments do not appear so intuitive.

\clearpage 

\begin{figure}[ht]
\centering
\includegraphics[width=0.8\linewidth]{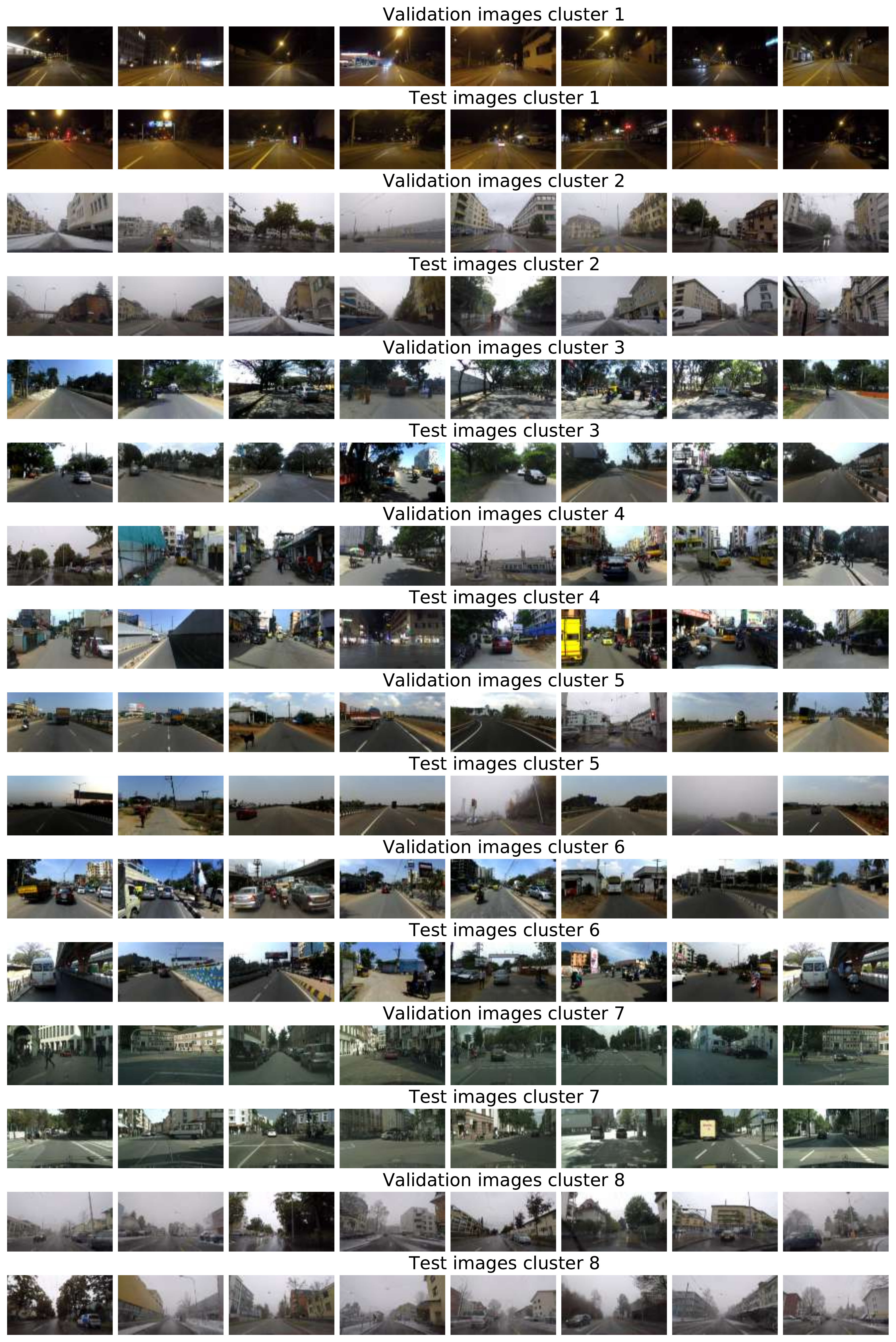}
\caption{\textbf{Visualization of clusters (Clust All).} Sample images from clusters computed for \cite{gong2021confidence}. In this case, the calibration set (where clusters are computed) contains imaged from all datasets and we qualitatively observe the cluster assignments to align with human intuition.
}
\label{figure:cluster_viz_All_dsets}
\end{figure}

\clearpage 

\begin{figure}[ht]
\centering
\includegraphics[width=0.8\linewidth]{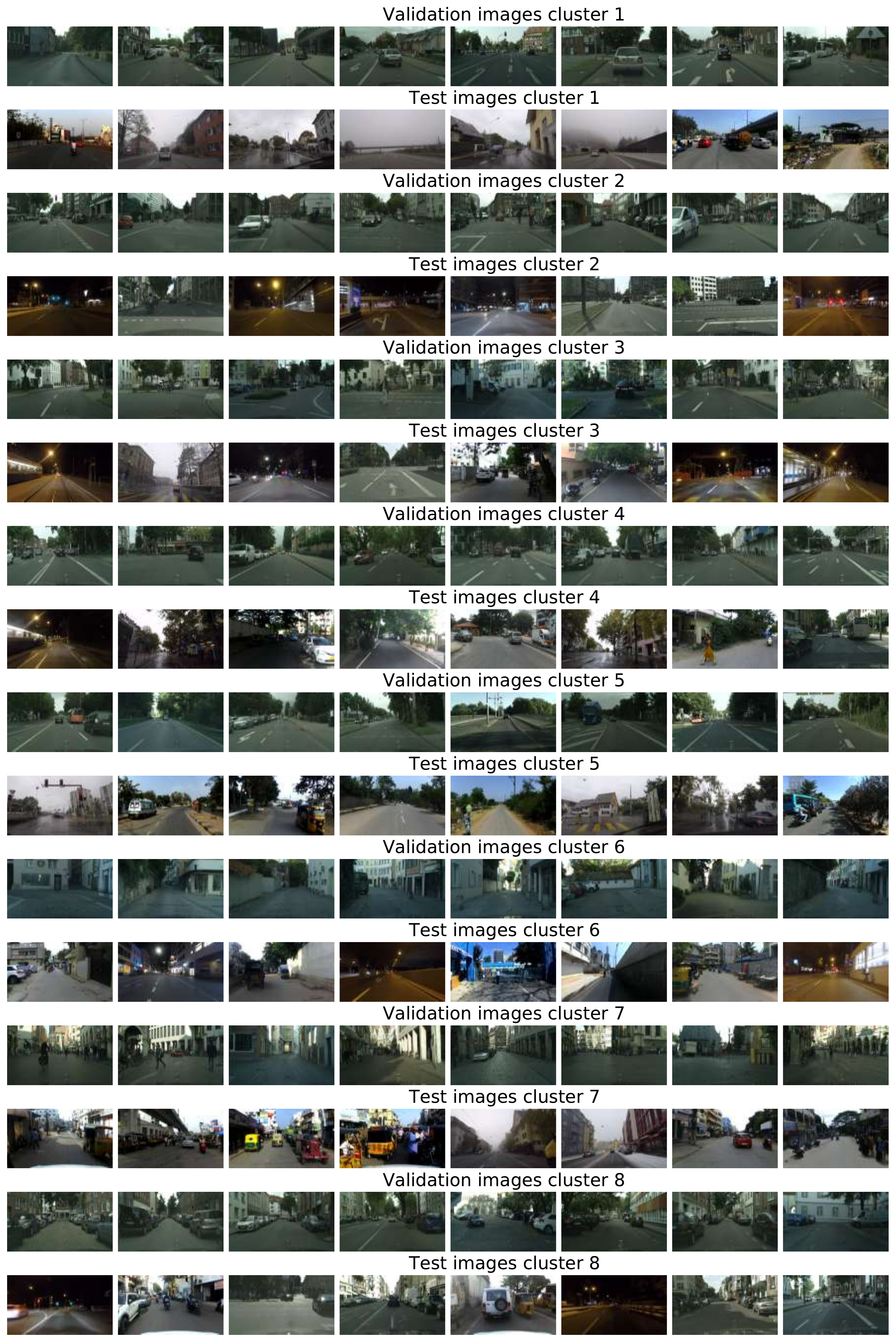}
\caption{\textbf{Visualization of clusters (Clust CS).} Sample images from clusters computed for \cite{gong2021confidence}. In this case, the calibration set (where clusters are computed) contains imaged from CS only. We qualitatively observe that the clusters are not representative of the test distribution.
}
\label{figure:cluster_viz_Cityscapes}
\end{figure}

\clearpage 

\begin{figure}[ht]
\centering
\includegraphics[width=0.8\linewidth]{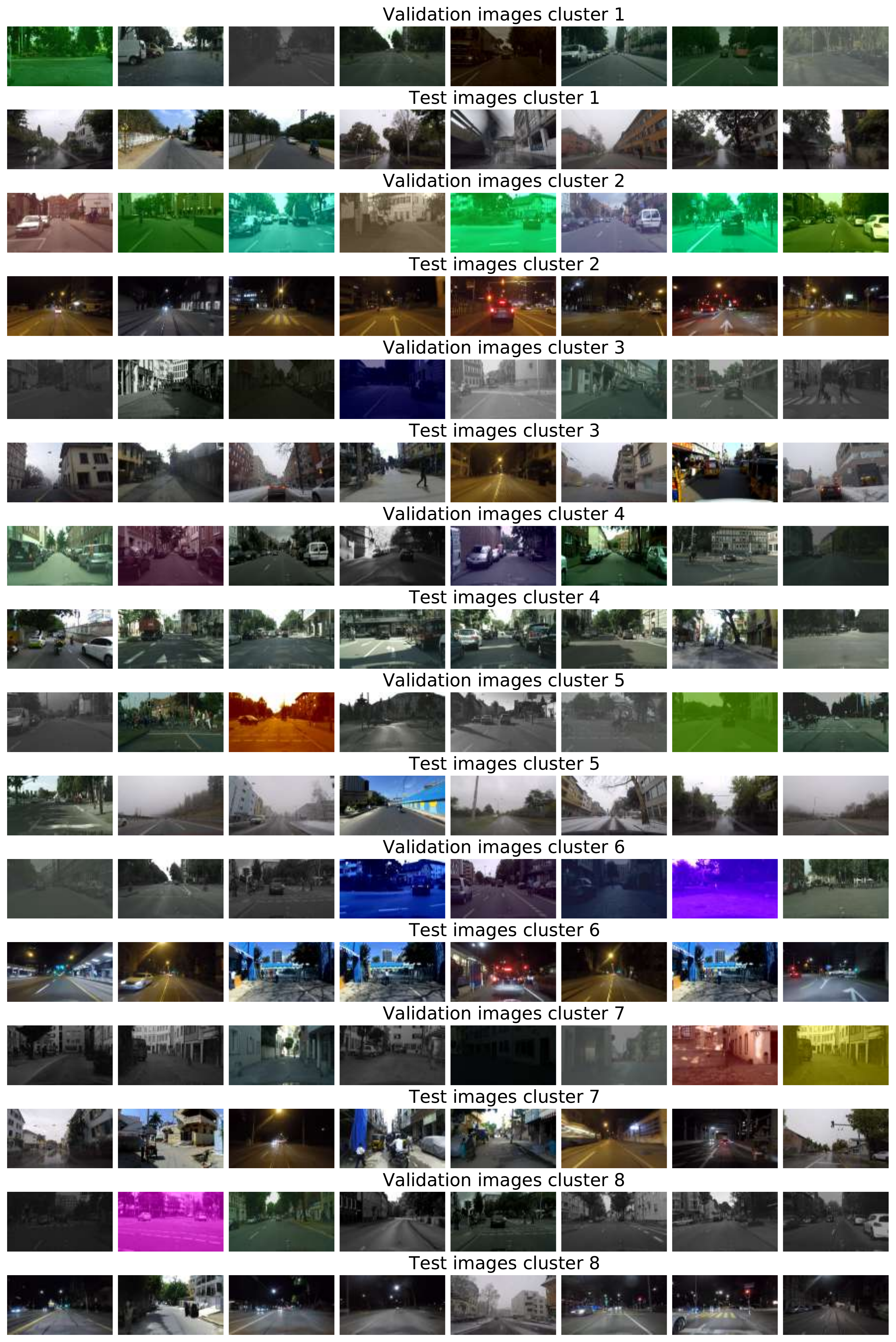}
\caption{\textbf{Visualization of clusters (Clust CS aug).} Sample images from clusters computed for \cite{gong2021confidence}. In this case, the calibration set (where clusters are computed) contains imaged from CS augmented only. Even if data augmentations introduce variability to the dataset, it is still not representative of the test distribution.
}
\label{figure:cluster_viz_CS_aug}
\end{figure}

\clearpage 

\section{Theoretical insights on adaptive temperature via clustering}
\label{sec:proof_subset_ECE}

In \cref{figure:val_set_clustering} we observed that even when we evaluate the ECE on the calibration set, the calibration error does not monotonically decrease as we increase the number of clusters. This is somewhat counterintuitive as one would think that, with more clusters, the temperatures can be more fine-grained and evaluating on the calibration set there are no issues with overfitting. However, it is indeed possible since the temperatures of each cluster are optimized independently. In the following we present a theorem and proof to show that decreasing the ECE for several disjoint subsets of images (clusters) independently does not guarantee that the ECE on the union set will decrease.

First, we introduce some preliminaries and notation. Consider a classifier $f : \mathcal{X} \rightarrow \mathcal{Y}$ with $\mathcal{Y} = \{1, 2, 3, \dots, k\}$. We model our classifier as $f = \textrm{argmax}\ \hat{p}(\mathbf{y} | \mathbf{x})$ where $\hat{p}(\mathbf{y} | \mathbf{x})$ are the pseudo-probabilities estimated by the model for each class given the input. We say that a model is calibrated when
\begin{equation}
P \left( \mathbf{y}\ |\ \hat{p}(\mathbf{y} | \mathbf{x}) = p \right) = p, \quad \forall \mathbf{x} \sim \mathcal{D}
\end{equation}
where $P$ is the true probability of the classes and $\mathcal{D}$ the data distribution. However, most works focus on a simplification of this problem where only the probability of the predicted class is taken into account, that is:
\begin{equation}
P \left( y = \textrm{argmax}\ \hat{p}(\mathbf{y} | \mathbf{x})\ |\ \textrm{max}\ \hat{p}(\mathbf{y} | \mathbf{x}) = p^* \right) = p^* \quad \forall \mathbf{x}, y \sim \mathcal{D}.
\end{equation}

The most common metric to measure calibration is the \textit{Expected Calibration Error (ECE)}. Which looks at the expected difference between the predicted and actual probabilities:

\begin{equation}
\mathbb{E} \left [\ \left|\ p^*  - \mathbb{E} \left [\ P(\textrm{argmax}\ \hat{p}(\textbf{y} | \mathbf{x}) = y)\ |\  \textrm{max}\ \hat{p}(\textbf{y} | \mathbf{x}) = p^*\ \right ]\ \right| \ \right ].
\end{equation}

In order to empirically estimate the ECE, it is standard practice to quantize the output probabilities given by the model and compute the mean probability (confidence) and accuracy in each bin.  That is,  

\begin{equation}
\widehat{\textrm{ECE}}_{f} = \sum_{i=1}^m \frac{\#B_i}{n} \ |\textrm{accuracy}(B_i) - \textrm{confidence}(B_i)|
\end{equation}

where $\#B_i$ denotes the number of elements in the $i^{th}$ bin,  $m$ denotes the number of bins and $n = \sum_{i=1}^m \#B_i$ the total number of elements used to estimate the ECE.  We also use $f$ to indicate the dependency of the ECE on the classifier.

Consider now,  that we split the data into two different sets and we quantize it in bins $\mathcal{B} = \{B_i\}$ and $ \mathcal{B'} = \{ B'_i   \}$ with the same boundaries so that for each pair $B_i$ and $B'_i$  the range of confidence values are the same.  Moreover,  consider now the respective ECE computed for each subset of data independently --- denoted $\widehat{\textrm{ECE}}_{f}(\mathcal{B})$ and $\widehat{\textrm{ECE}}_{f}(\mathcal{B'})$ --- and on the full set of points $\widehat{\textrm{ECE}}_{f}(\mathcal{B}+\mathcal{B'})$ where $\mathcal{B}+\mathcal{B'}$ is an abuse of notation to indicate the union of elements in the bins for each index $i$. 

\begin{theorem}
With the notation described above, consider a model $f_{\textrm{oracle}}$ such that an ``oracle" splits the input according to whether it belongs to $\mathcal{B}$ or $\mathcal{B'}$. Moreover, $f_{\textrm{oracle}}$ uses two calibration strategies (one for $\mathcal{B}$ and one for $\mathcal{B'}$) in a way that it improves it's ECE on each subset $\mathcal{B}$, $\mathcal{B'}$ individually compared to some baseline model $f$ (e.g. by means of temperature scaling with a different temperature for each subset).  This does not necessarily imply that the oracle model ($f_{\textrm{oracle}}$) will be better calibrated on the full set of points $\mathcal{B}+\mathcal{B'}$ than the baseline model $f$. That is, given that:
$$
(a) \quad \widehat{\textrm{ECE}}_{f_{\textrm{oracle}}}(\mathcal{B}) \leq \widehat{\textrm{ECE}}_{f}(\mathcal{B})\ \ \  \textrm{and} \ \ \ \widehat{\textrm{ECE}}_{f_{\textrm{oracle}}}(\mathcal{B'}) \leq \widehat{\textrm{ECE}}_{f}(\mathcal{B'})
$$

Then,  condition $(a)$ is not sufficient to claim:
$$
 (b) \quad \widehat{\textrm{ECE}}_{f_{\textrm{oracle}}}(\mathcal{B + B'}) \leq \widehat{\textrm{ECE}}_{f}(\mathcal{B + B'})
$$
\end{theorem}

\begin{proof}
In order to proof the theorem we will construct a counter-example where condition $(a)$ is satisfied but condition $(b)$ is not. Consider the accuracy and confidence for the full set of points in a given bin:

\begin{align*}
    \textrm{acc}(B_i + B'_i) &= \frac{\#B_i\ \textrm{acc}(B_i) + \#B'_i\ \textrm{acc}(B'_i) }{\#B_i + \#B'_i} \\
    \textrm{conf}(B_i + B'_i) &= \frac{\#B_i\ \textrm{conf}(B_i) + \#B'_i\ \textrm{conf}(B'_i) }{\#B_i + \#B'_i}
\end{align*}

Then the ECE of the full set of points will be:

\begin{align*}
\widehat{\textrm{ECE}}_{f}(\mathcal{B + B'})  &= \sum_{i=1}^m \frac{\#B_i + \#B'_i}{n + n'} \left|\textrm{acc}_{f}(B_i + B'_i) - \textrm{conf}_{f}(B_i + B'_i)\ \right| \\
&= \sum_{i=1}^m \frac{1}{n + n'}\ \left|\ \#B_i (\textrm{acc}_{f}(B_i) - \textrm{conf}_{f}(B_i)) + \#B'_i(\textrm{acc}_{f}(B'_i) - \textrm{conf}_{f}(B'_i)) \ \right|.
\end{align*}

To simplify the notation,  let us define $r_i(f) = \textrm{acc}_{f}(B_i) - \textrm{conf}_{f}(B_i)$ and similarly for $r'_i(f)$.  Then,  we can write:

\begin{align*}
\widehat{\textrm{ECE}}_{f}(\mathcal{B + B'}) =&  \sum_{i=1}^m \frac{ 1}{n + n'}\  |\#B_i\ r_i(f)  + \#B'_i\ r'_i(f) |. \\
\widehat{\textrm{ECE}}_{f}(\mathcal{B})  =&  \sum_{i=1}^m \frac{\#B_i}{n} \ |r_i(f) |. \\
\widehat{\textrm{ECE}}_{f}(\mathcal{B'}) =&  \sum_{i=1}^m \frac{\#B'_i }{n'} \ |r'_i(f) |. \\
\end{align*}

Now let us consider a setting where $r_i(f) = r $ and $r_i(f_{\textrm{oracle}}) = -0.5 r $ for some $r \neq 0$ while $r'_i(f) = -r$ and $r'_i(f_{\textrm{oracle}}) = -0.5r $. Moreover, consider $\#B_i = \#B'_i$ which implies $n = n'$, then this setting would satisfy condition $(a)$ since 

\begin{align*}
&\widehat{\textrm{ECE}}_{f}(\mathcal{B})  =  \sum_{i=1}^m \frac{\#B_i }{n} \ |r |  \geq \sum_{i=1}^m \frac{\#B_i}{n} \ |-0.5r | = \widehat{\textrm{ECE}}_{f_{\textrm{oracle}}}(\mathcal{B})  \quad \textrm{and} \\
&\widehat{\textrm{ECE}}_{f}(\mathcal{B'}) =  \sum_{i=1}^m \frac{\#B'_i }{n'} \ |-r | \geq   \sum_{i=1}^m \frac{\#B'_i}{n'} \ |-0.5r |  =  \widehat{\textrm{ECE}}_{f_{\textrm{oracle}}}(\mathcal{B'}).
\end{align*}

However,  this same setting would not satisfy condition $(b)$ since 

\begin{align*}
\widehat{\textrm{ECE}}_{f}(\mathcal{B + B'}) &=  \sum_{i=1}^m \frac{\#B_i}{2n}\  |r  - r | = 0 \quad \textrm{and} \\
\widehat{\textrm{ECE}}_{f_{\textrm{oracle}}}(\mathcal{B + B'}) &= \sum_{i=1}^m \frac{\#B_i}{2n}\  |-0.5r  - 0.5r | > 0.
\end{align*}

Thus, we have showed that condition $(a)$ does not imply $(b)$.
\end{proof}
This result implies that minimizing the ECE for different subsets of the data independently (e.g.  each cluster of images) does not necessarily lead to an overall improvement of the ECE. Moreover,  we have assumed only two sets of samples without loss of generalization since if $(a)$ implied $(b)$ for an arbitrary number of data splits it would in particular imply it for two. Finally, note that our result is valid for either image classifiers or segmentors. In the first case we would each prediction would be the class of a whole image while in the second case the each pixel in an image would have a different prediction.

\clearpage 

\section{Per-class clustering}
\label{sec:per_class_clustering}
In \cref{figure:calib_methods_comparison}, we have observed that adaptive temperature via clustering \cite{gong2021confidence} does not significantly help improving out-of-domain calibration compared to local temperature scaling (LTS) \cite{ding2021local}. One important difference between the methods is that LTS computes a temperature for each pixel in the image while clustering is performed at the image level -- using a single temperature per image. This motivates us to perform an ablation where, on top of the image level clustering, pixels in a given image are grouped according to their predicted class. Intuitively, we are looking for a temperature for regions in the image that look alike to the network (since they are assigned to the same class). In \cref{figure:per_class_clustering} we compare standard per-image clustering (top) with the aforementioned per-class clustering (bottom). Note that per-class clustering always groups pixels according to the predicted class, therefore if $k=1$ then there are 19 clusters (corresponding to the CS classes). Calibration images are from the CS dataset.

Similarly to per-image clustering, increasing the number of clusters does not seem to always help when using per-class clustering. Moreover, we do not find that per-class clustering significantly improves calibration except for SegFormer architecture. 
We are not stating here that finer-grained clustering may not yield further improvements (and reach similar performance to LTS). However, given that improving ECE in different subdomains independently is not guaranteed to improve overall calibration (see \cref{sec:proof_subset_ECE}), perhaps a different approach to finding the temperatures and clusters taking into account both local and global calibration error would be needed.

\begin{figure}[ht]
\centering
\includegraphics[width=0.9\linewidth]{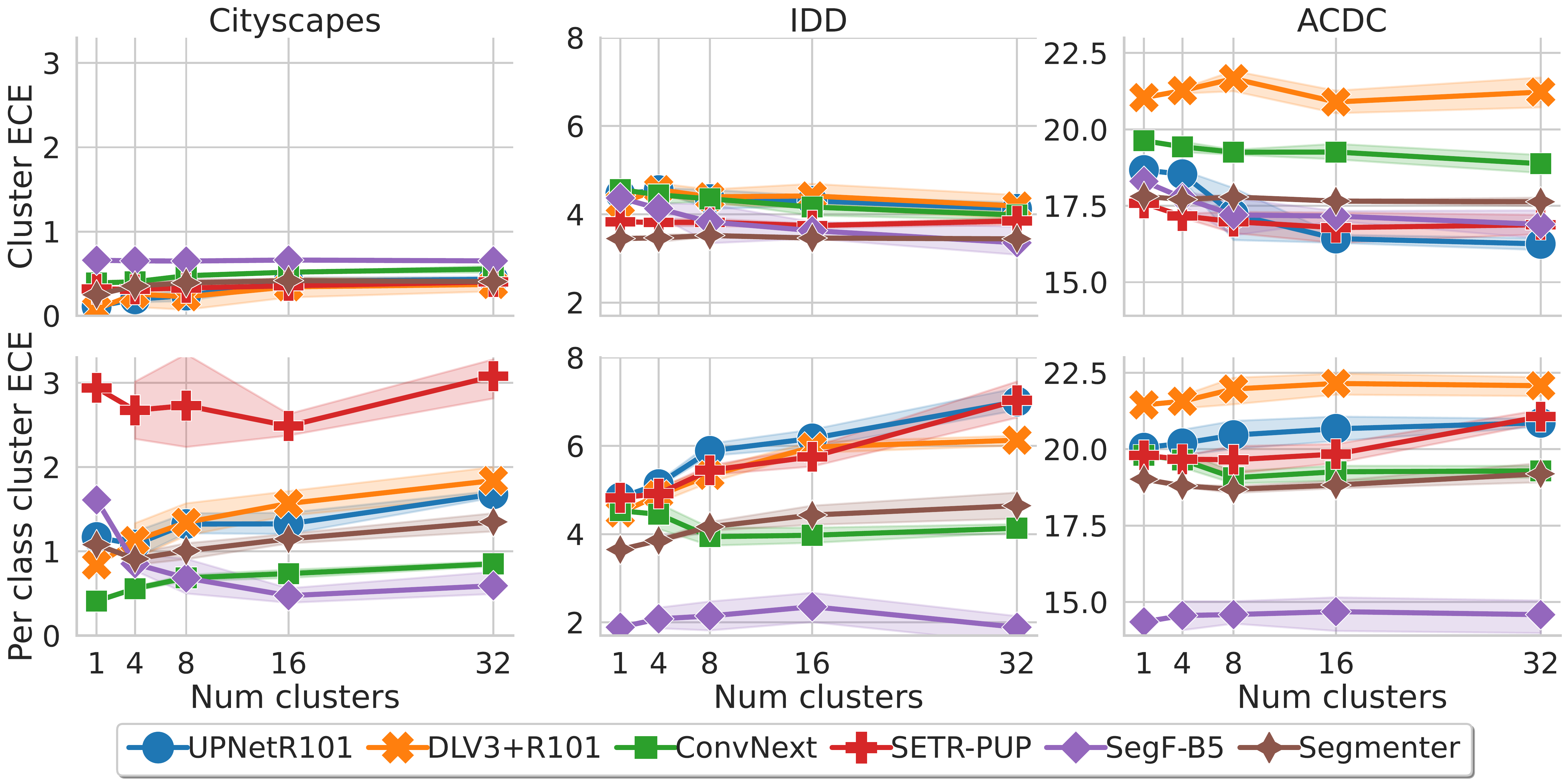}
\caption{\textbf{Per-class clustering ablation.} ECE $(\downarrow)$ for different models and datasets as we vary the number of clusters computed for calibration. We compare per-image clustering (where all pixels in an image are given the same temperature) \vs per-class clustering (where different pixels in an image are given a temperature depending on their predicted class). Overall, we do not observe consistent improvements when using per-class clustering except in SegFormer architecture.
}
\label{figure:per_class_clustering}
\end{figure}

\clearpage

\section{Ablation LTS: image \vs logits}
\label{sec:LTS_image_logits_ablation}
LTS \cite{ding2021local} employs a small-weight calibration network which receives both the image and predicted logits as input and it returns a temperature map to scale the logits (with a different temperature for each pixel in the image). Given its remarkable performance (see~\cref{figure:LTS_ablation,figure:calib_methods_comparison}) and, to get further insights into this method, in \cref{figure:LTS_image_logits_ablation} we perform an ablation where the calibration network only receives the image or the logits as input. To carry out this experiment, we modify the network in \cite{ding2021local} so that both input branches (logits and image) receive the same input, either both logits or both image. All calibration networks have been trained in CS images only. 

Interestingly, we observe that, in distribution (ECE CS), the better performing method for most networks is the LTS variant that uses the logits only. On the other hand, for \ood calibration, the better performing variant in most cases is the one that relies on both logit and image information. Moreover, under strong domain shifts (ECE ACDC), LTS yields better results by using only the image information than by using only the logits. However, this is subject to variability as results vary across different architectures. Further investigations on 
how logit and image signals are combined may constitute a promising direction to further improve calibration results. 

\begin{figure}[ht]
\centering
\begin{subfigure}[t]{\linewidth}
\centering
\includegraphics[width=0.75\linewidth]{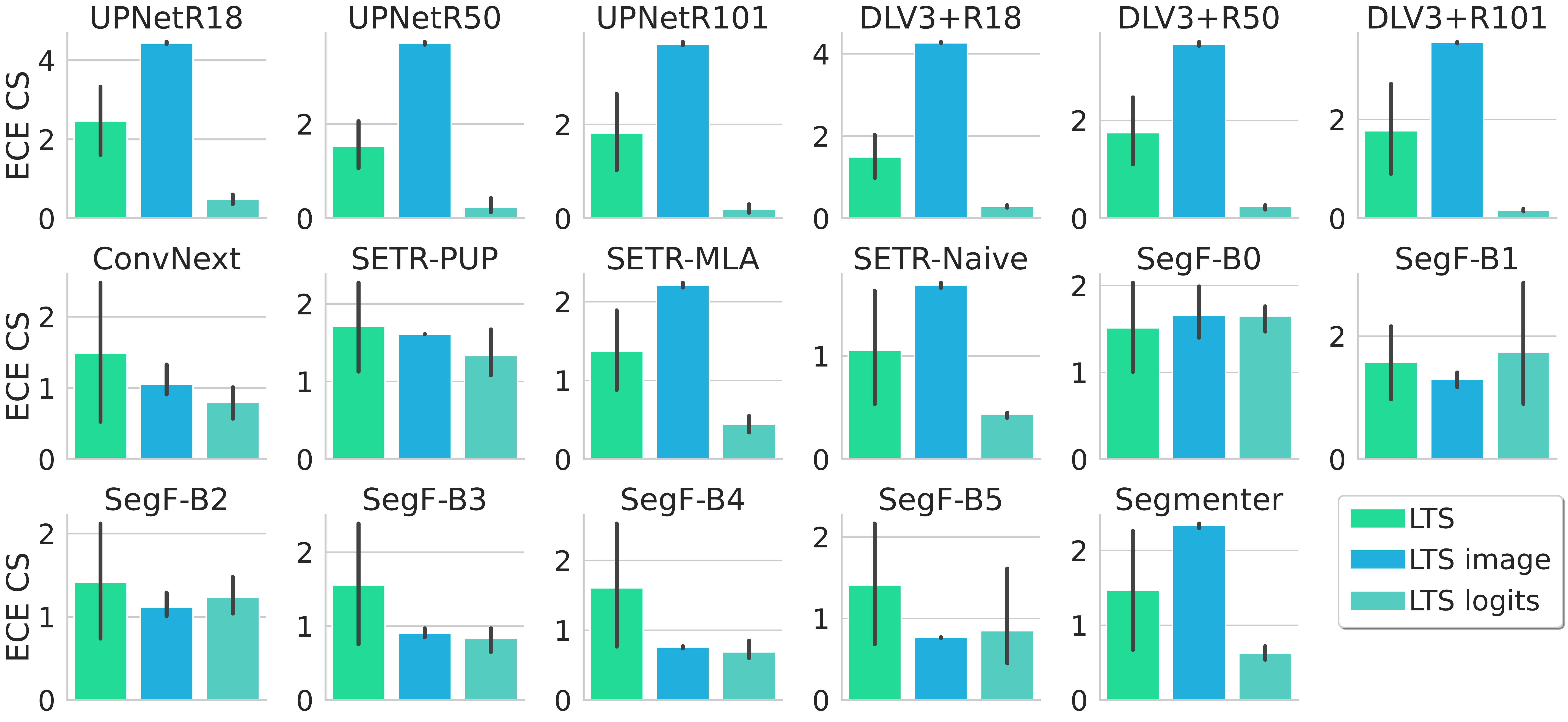}
\end{subfigure}

\vspace{0.3cm}

\begin{subfigure}[b]{\linewidth}
\centering
\includegraphics[width=0.75\linewidth]{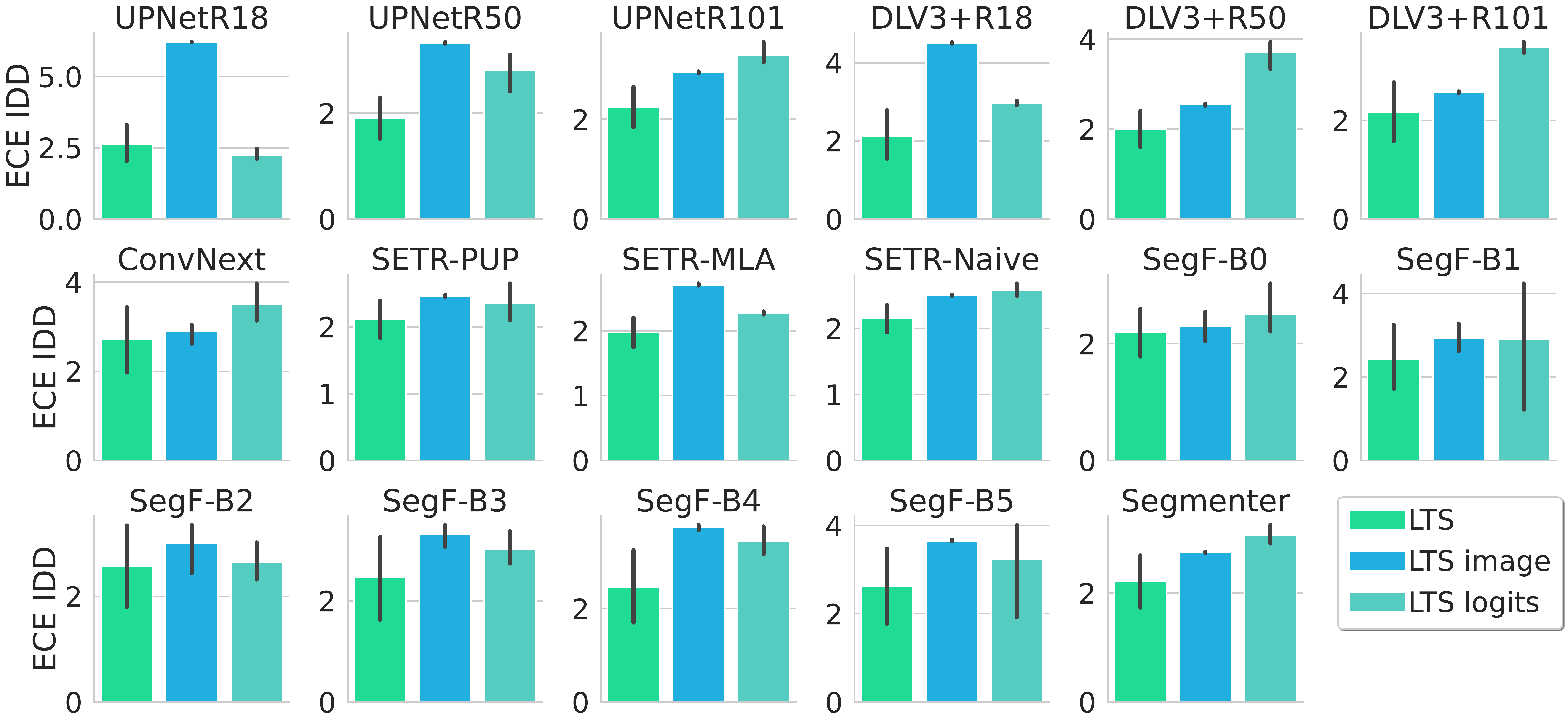}
\end{subfigure}

\vspace{0.3cm}

\begin{subfigure}[b]{\linewidth}
\centering
\includegraphics[width=0.75\linewidth]{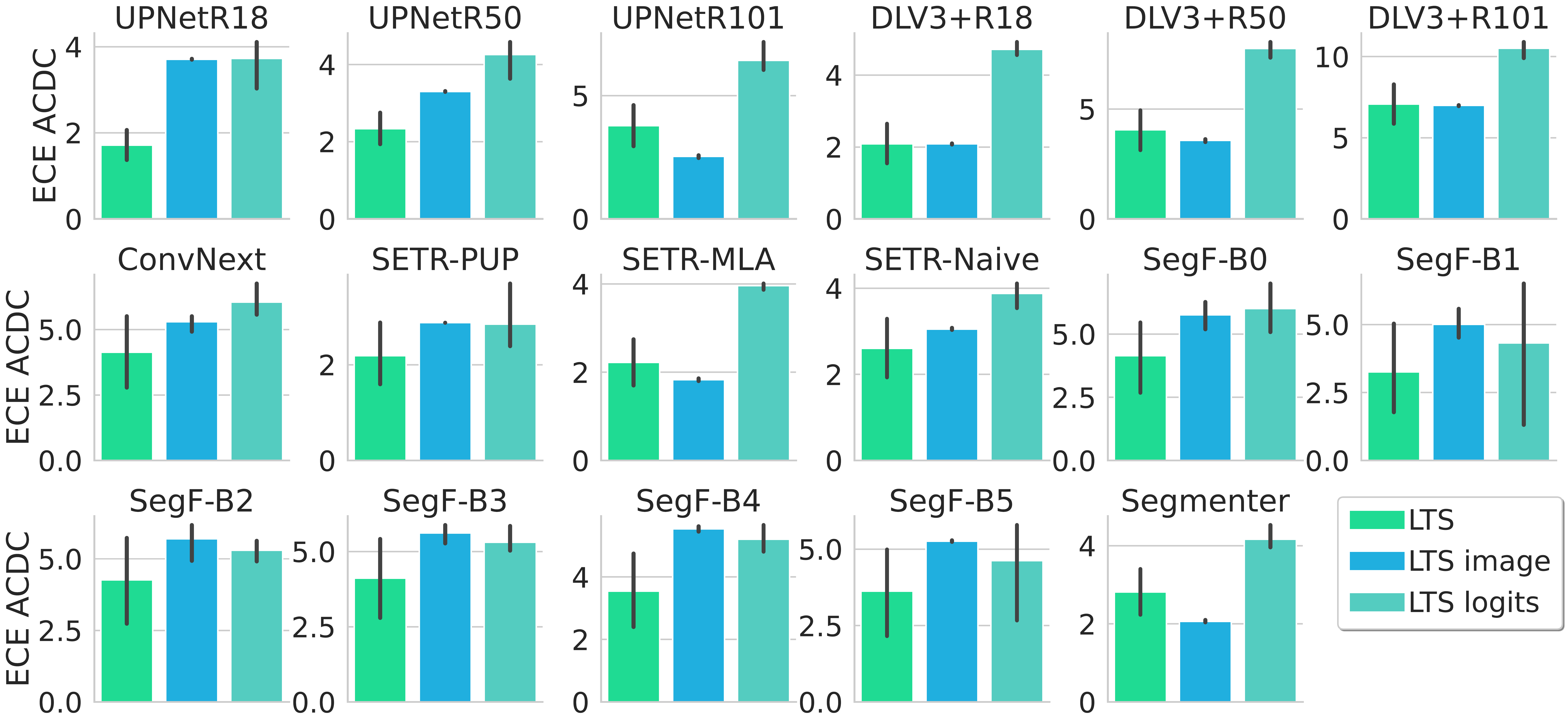}
\end{subfigure}
\caption{\textbf{LTS with image \vs logits information} ECE $(\downarrow)$ after calibration with three LTS variants: the original method combining image and logits (LTS), an LTS version only using logit information (LTS logits) and the complementary version only using image information (LTS image). We observe that in distribution, using only the logits seems to perform better while out of distribution using information from both image and logits works best.
}
\label{figure:LTS_image_logits_ablation}
\end{figure}

\clearpage

\section{Ablation of the effect of training settings}
\label{ap:training_ablations}

Our goal is to assess the reliability of the best available models for segmentation — the ones used by practitioners. Therefore, to test a model at its best, we need to use its most favorable setup (original pre-training, optimizer, etc.), as done in previous work \cite{pinto2022impartial, minderer2021revisiting}. Nevertheless, in this section we ablate the effects of two different training settings: the dataset used to pre-train the backbone and the number of training iterations. Most recent models were pre-trained on ImageNet 21k, however, ResNet backbones are usually pre-trained on ImageNet 1k. In the top rows of \cref{tab:training_ablations} we compare the performance of a BiT-ResNet50 \cite{kolesnikov2020big} and a ConvNext-B \cite{LiuCVPR22AConvNet4The2020s} backbones pre-trained with either ImageNet 1k or 21k. We observe that the benefits of a larger pre-training dataset are small in comparison to the gap between architectures.

On the other hand, most models were trained using 80k iterations (as is quite standard for CS dataset, but Segmenter and SegFormer models' original training schedule uses 160k iterations. In the bottom rows of \cref{tab:training_ablations} we compare Segmenter and SegFormer-B5 networks trained with either 160k or 80k iterations, again the differences are minor compared to the gap between architectures.

\begin{table}[ht]
\centering
\begin{tabular}{@{}lcccccc@{}}
\toprule
                 & \multicolumn{3}{c}{\textbf{mIoU} ($\uparrow$)} & \multicolumn{3}{c}{\textbf{ECE} ($\uparrow$)} \\ 
\midrule
\textbf{Architecture}        & \textbf{CS}    & \textbf{IDD}   & \textbf{ACDC}  & \textbf{CS}   & \textbf{IDD}  & \textbf{ACDC}  \\ 
\midrule
ConvNext-B (IN-1k)  & 80.39 & 64.11 & 58.77 & 0.77 & 5.53 & 18.95 \\
ConvNext-B (IN-21k) & 81.56 & 65.70 & 60.59 & 0.81 & 5.27 & 20.07 \\
BiT-RN50 (IN-1k)    & 76.36 & 57.46 & 47.47 & 0.67 & 5.18 & 19.47 \\
BiT-RN50 (IN-21k)   & 76.49 & 57.23 & 46.70 & 0.65 & 5.20 & 19.63 \\
\midrule
\midrule
Segmenter (160k)    & 76.19 & 61.96 & 63.24 & 0.83 & 4.14 & 18.59 \\
Segmenter (80k)     & 76.22 & 60.74 & 62.96 & 0.71 & 4.28 & 18.36 \\
SegF-B5 (160k)      & 80.94 & 62.47 & 59.43 & 1.93 & 6.46 & 21.37 \\
SegF-B5 (80k)       & 80.15 & 62.95 & 57.00 & 1.07 & 5.32 & 21.27 \\ \bottomrule
\end{tabular}
\caption{Ablations of different training settings. On the top we compare ConvNext and BiT-RN50 architectures with the backbones pre-trained on either ImageNet 1k or 21k datasets. We observe that the benefits of a larger pre-training dataset are small in comparison to the gap between architectures. On the bottom we compare the amount of training iterations, again, we observe that although longer training schedules do improve the performance, changes are also small in comparison to architecture differences.}
\label{tab:training_ablations}
\end{table}

\section{Architectures for Universal Image Segmentation}
\label{ap:mask2former}
With the appearance of transformers, and in particular motivated by DETR \cite{carion2020end}, some architectures have been proposed with the objective to solve the three main Image Segmentation tasks, that is: Semantic Segmentation, Instance Segmentation and Panoptic Segmentation \cite{cheng2021perpixel, zhang2021knet, ChengCVPR22MaskedAttentionMaskTransformer4UniversalSiS}. Here we will focus on Mask2Former \cite{ChengCVPR22MaskedAttentionMaskTransformer4UniversalSiS} which is based on MaskFormer \cite{cheng2021perpixel} and is the best performing universal architecture to the best of our knowledge. Interestingly, to be able to solve the different segmentation tasks jointly, these architectures do not output the standard per-pixel logits when it comes to semantic segmentation. Instead, they predict a set of $N$ object masks (where $N$ is fixed) and the class probabilities for each object. Then, to obtain the per-pixel class probabilities they marginalize over all the possible objects a pixel could belong to, we refer the reader to \cite{cheng2021perpixel, ChengCVPR22MaskedAttentionMaskTransformer4UniversalSiS} for further details. 

Although this final output can be regarded as per-pixel class probabilities, they way it is obtained differ from the standard \textit{logits + softmax} setting that all calibration methods rely on, therefore it is not straightforward to compare these universal architectures with other models with temperature scaling. Nevertheless, given the good performance and wider applicability of these models, we include them in our study comparing only off-the-shelf performance in terms of mIoU, calibration, OOD detection and misclassification. The best performing model from \cite{ChengCVPR22MaskedAttentionMaskTransformer4UniversalSiS} is based on a Swin transformer (Swin Large) \cite{LiuICCV21SwinTransformerHierarchicalViTShiftedWindows}, therefore, we also include a Swin transformer model with UpperNet to the comparison for completeness.

\subsection{Calibration with Mask2Former}
In \ref{figure:calib_mask2former} we present the ECE vs segmentation error (100 - mIoU) for the different models. Interestingly, we observe that Mask2Former seems to be the best-performing model in terms of mIoU in all datasets. In terms of calibration error Mask2Former is poorly calibrated in distribution (CS) but has a milder increase in ECE as the distribution shift becomes stronger. 

\begin{figure}[ht]
\centering
\includegraphics[width=0.7\linewidth]{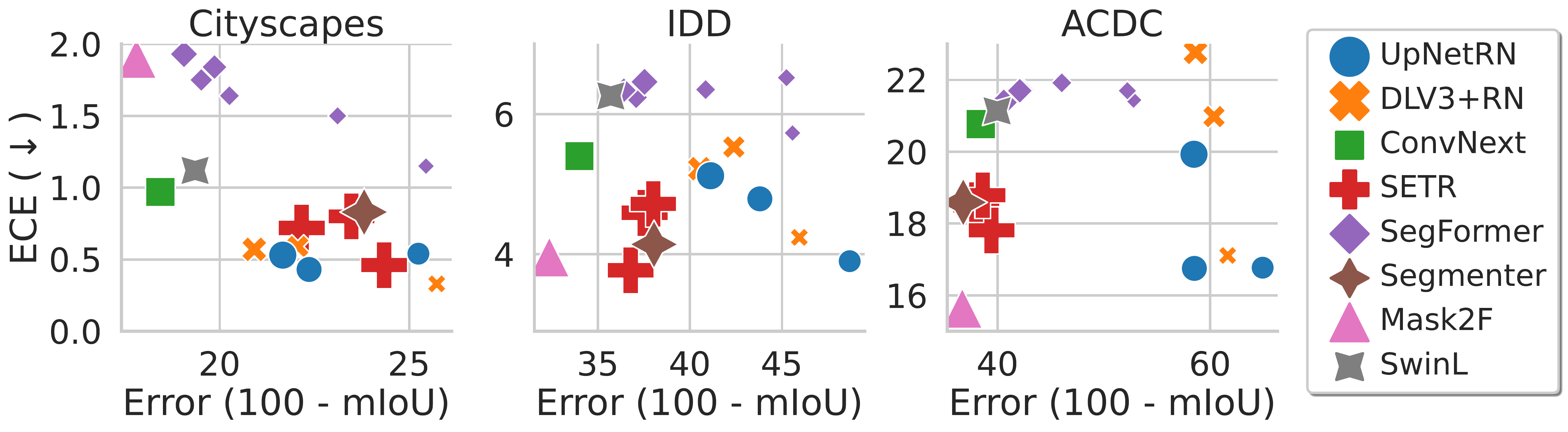}
\caption{
\textbf{mIoU error $(\downarrow)$ \vs Expected calibration error $(\downarrow)$}. All models were trained on CS. Markersize proportional to number of parameters. Interestingly, we observe that Mask2Former seems to be the best-performing model in terms of mIoU in all datasets. In terms of calibration error Mask2Former is poorly calibrated in distribution (CS) but has a milder increase in ECE as the distribution shift becomes stronger. 
}
\label{figure:calib_mask2former}
\end{figure}

\subsection{OOD detection with Mask2Former}
In \ref{figure:ood_mask2former} we present the OOD vs mIoU for the different models. Here we observe that Mask2Former and Swin transformer align with the negative trend observed in previous models where models with better mIoU tend to perform worse at OOD detection. 

\begin{figure}[ht]
\centering
\includegraphics[width=0.7\linewidth]{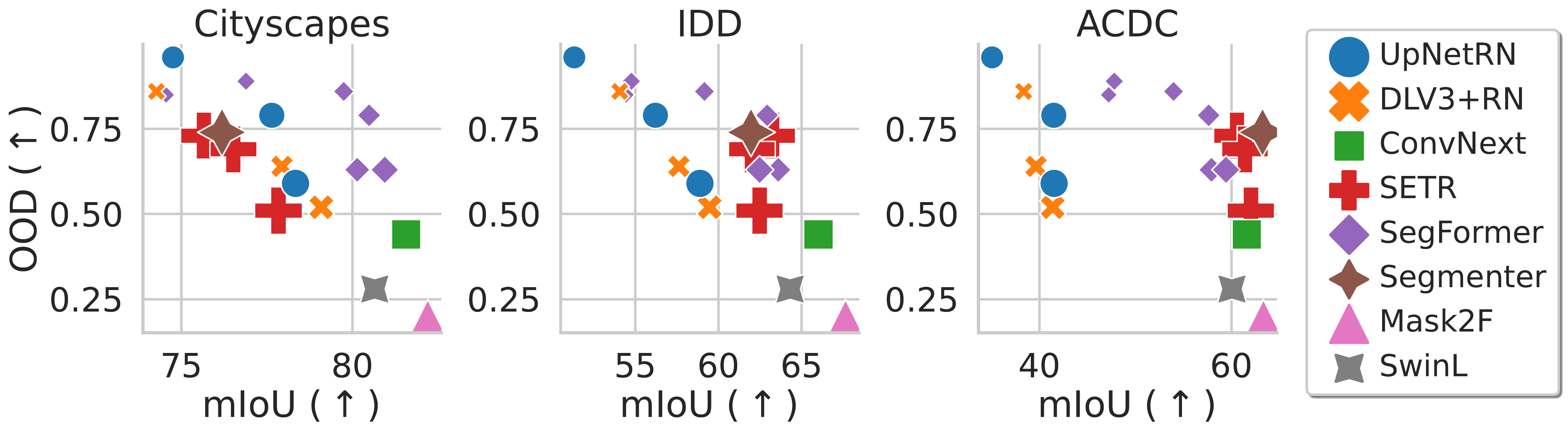}
\caption{
\textbf{mIoU $(\uparrow)$ \vs OOD - AUROC $(\uparrow)$} for different model families. All models trained on CS. Markersize proportional to number of parameters. We observe that Mask2Former and Swin transformer align with the negative trend observed in previous models where models with better mIoU tend to perform worse at OOD detection. 
}
\label{figure:ood_mask2former}
\end{figure}

\subsection{PRR with Mask2Former}

In \ref{figure:prr_mask2former} we present the OOD vs mIoU for the different models. Although Mask2Former is significantly better than other models in terms of misclassification detection in distribution, it seems to perform significantly worse under strong domain shifts (ACDC). This seems to be contrary to the trend followed by other models where robustness seems to be correlated with misclassification under domain shift.

\begin{figure}[ht]
\centering
\includegraphics[width=0.7\linewidth]{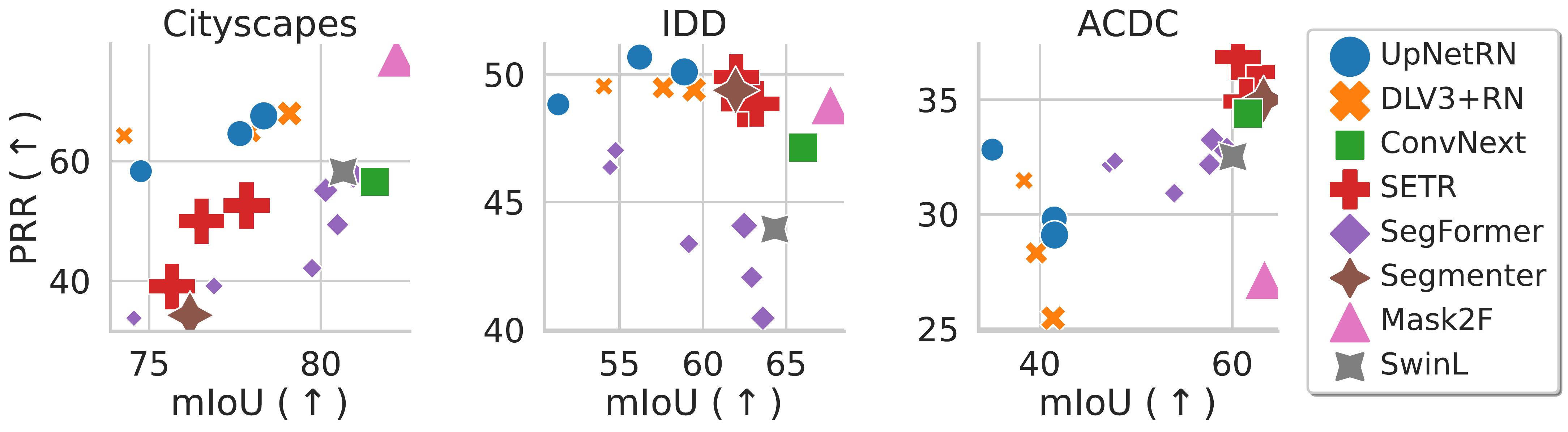}
\caption{
\textbf{mIoU $(\uparrow)$ \vs prediction rejection ratio $(\uparrow)$} for different model families. All models trained on CS. Markersize proportional to number of parameters. Mask2Former is significantly better than other models in terms of misclassification detection in distribution, it seems to perform significantly worse under strong domain shifts (ACDC).
}
\label{figure:prr_mask2former}
\end{figure}

\section{Visualization of uncertainty and temperature maps}
As opposed to classification, semantic segmentation models can assign different confidence to regions of the image. That can allow, among other applications, to detect regions with low confidence that may correspond to novel classes or weird instances of a known class. ACDC shares the same classes as Cityscapes, however, IDD includes some novel classes which are not included in the Cityscapes dataset. In \ref{figure:viz_t_scaling} we illustrate some examples of IDD images with confidence maps before and after LTS, together with the temperature scaling maps predicted by the calibration network. We observe how different \ood classes (\eg autorickshaw, bridge, billboard) or weird instances are highlighted in the calibration maps. In \cref{ap:local_ood} we quantify how useful are the calibration maps in order to perform local \ood detection. 

\begin{figure}[h!]
\centering
\includegraphics[width=0.65\linewidth]{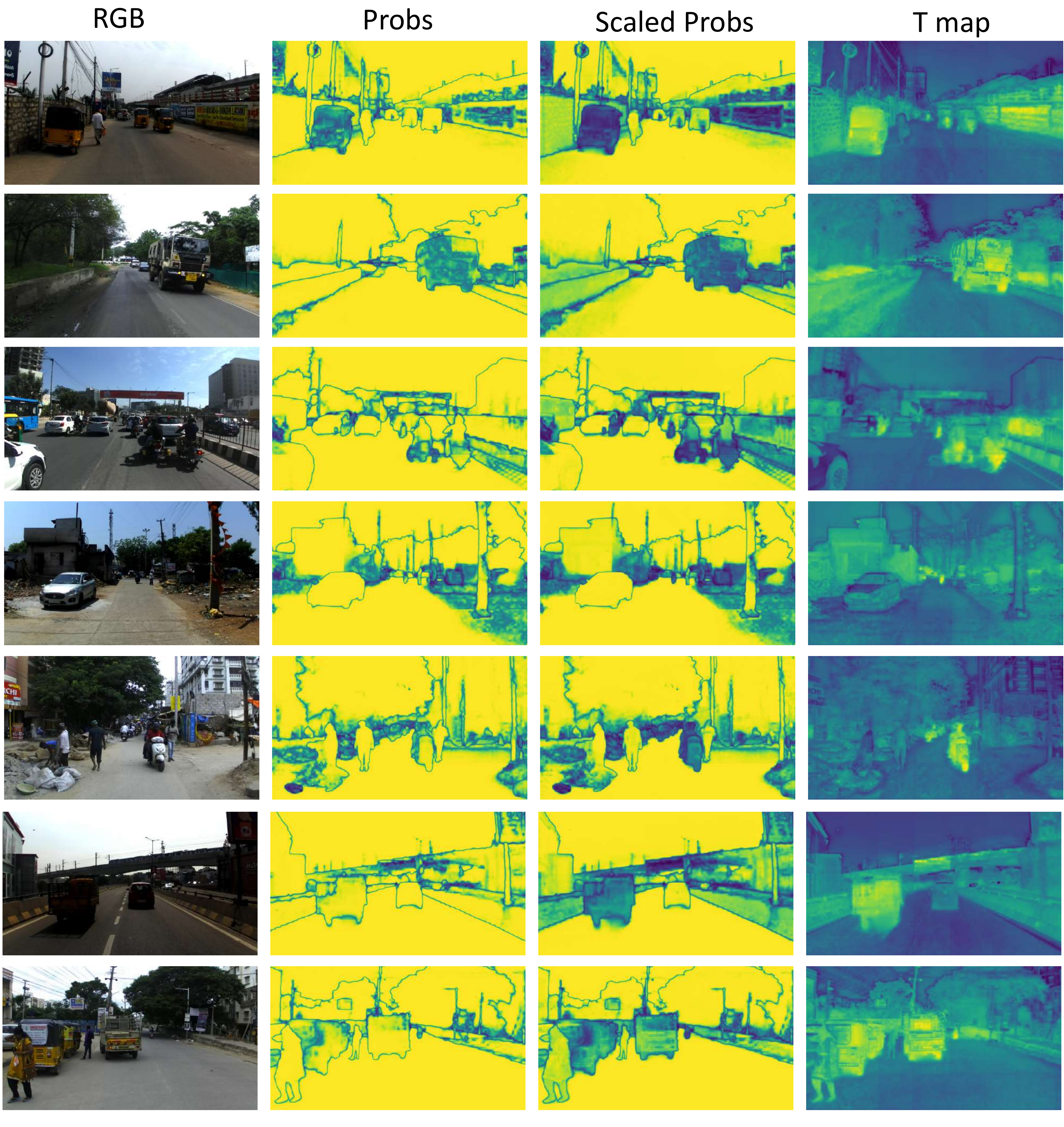}
\caption{
\textbf{Calibration and temperature maps} for different IDD images. We observe how different \ood classes (\eg autorickshaw, bridge, billboard) or weird instances are highlighted in the calibration maps.
}
\label{figure:viz_t_scaling}
\end{figure}

\clearpage

\section{Local OOD detection}
\label{ap:local_ood}
In this section we perform \ood detection at the pixel level to find regions of the image that belong to unknown classes (autorickshaw, guardrail, billboard, bridge). We define pixels of unknown classes as \ood while those corresponding to CS classes are in distribution. In \cref{figure:local_ood} we present the results of local \ood detection vs mIoU. We make two observations: i) Mask2Former has the best local \ood detection; ii) Differences between models are smaller for local \ood: numbers are roughly in the $0.7-0.8$ range \vs $ 0.4-1.0$ range for image-based \ood.

\begin{figure}[ht]
\centering
\includegraphics[width=0.5\linewidth]{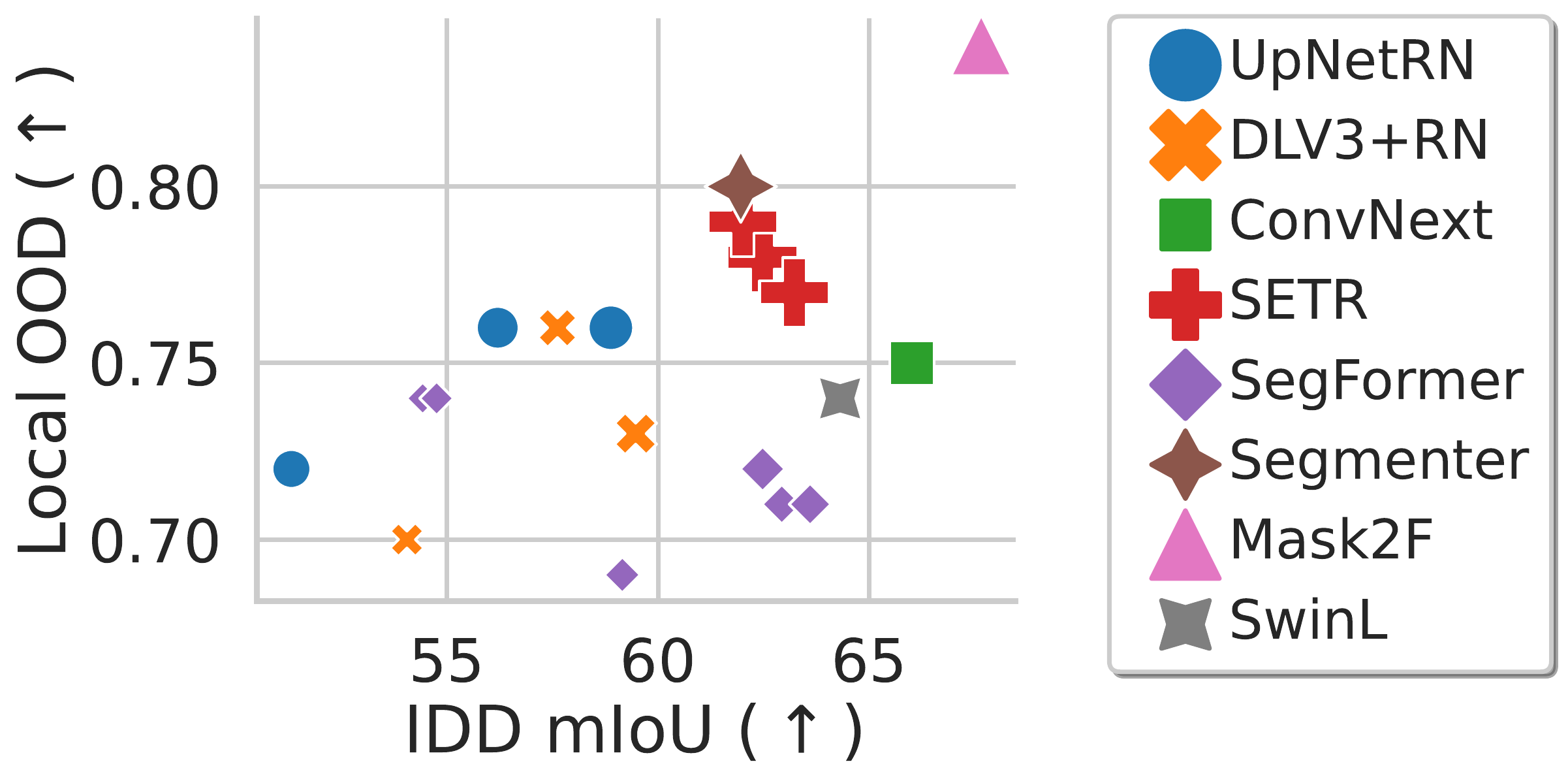}
\caption{
\textbf{mIoU $(\uparrow)$ \vs local ood $(\uparrow)$} for different model families. All models trained on CS. Markersize proportional to number of parameters. Mask2Former achieves the best local \ood detection performance.
}
\label{figure:local_ood}
\end{figure}


\clearpage

\section{Additional plots: adaptive temperature via clustering}
\label{sec:additional_plots_clustering}
In \cref{figure:cluster_ablation} we analyzed the calibration error after applying the method by Gong~\etal~\cite{gong2021confidence}, which clusters the images in the calibration set and computes a different temperature per cluster. Due to space constraints we only showed the best performing model for each family, in \cref{figure:clustering_methods_all_models} we show the results for all models.

\begin{figure}[ht]
\begin{subfigure}[t]{\linewidth}
\centering
\includegraphics[width=0.75\linewidth]{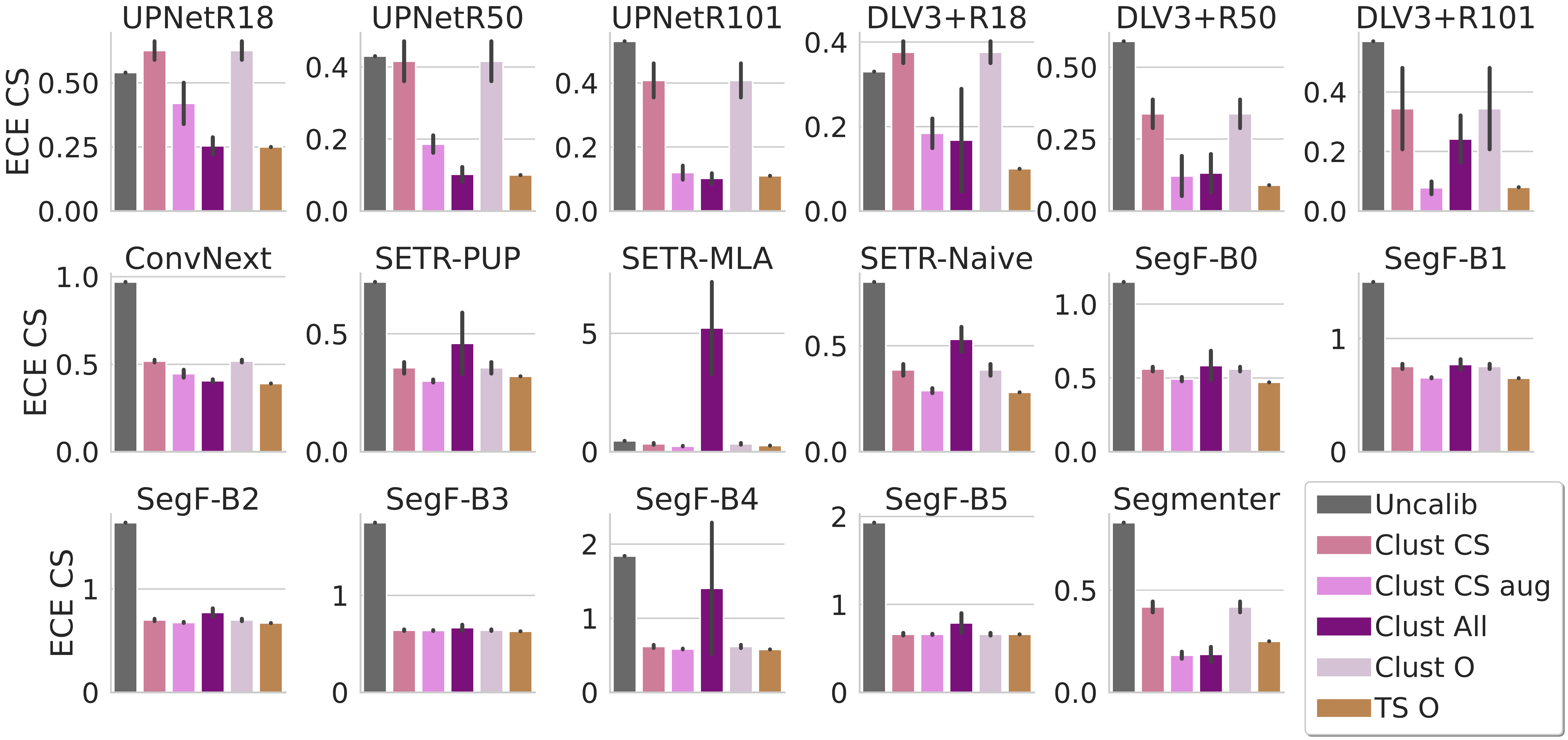}
\end{subfigure}

\vspace{0.3cm}

\begin{subfigure}[b]{\linewidth}
\centering
\includegraphics[width=0.75\linewidth]{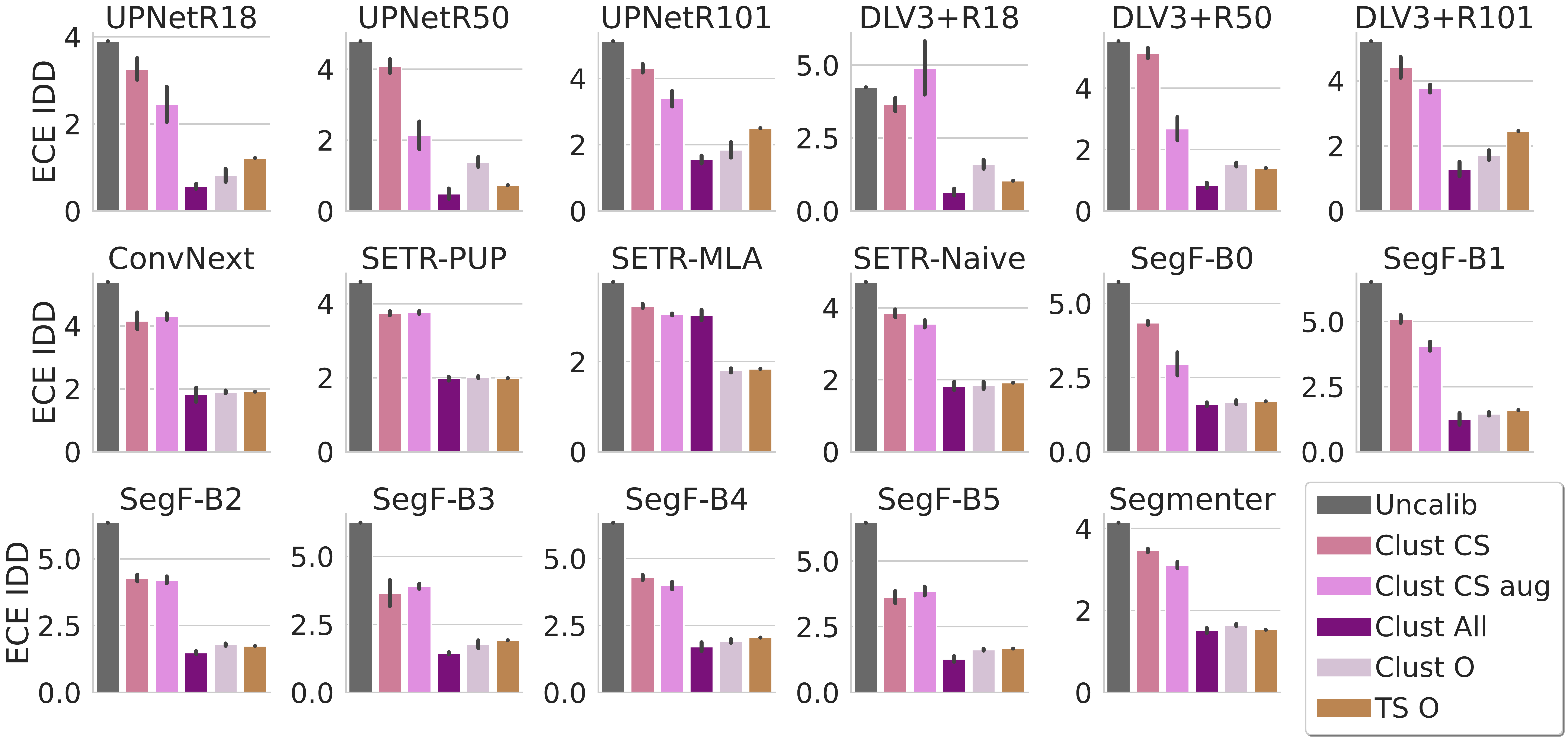}
\end{subfigure}

\vspace{0.3cm}

\begin{subfigure}[b]{\linewidth}
\centering
\includegraphics[width=0.75\linewidth]{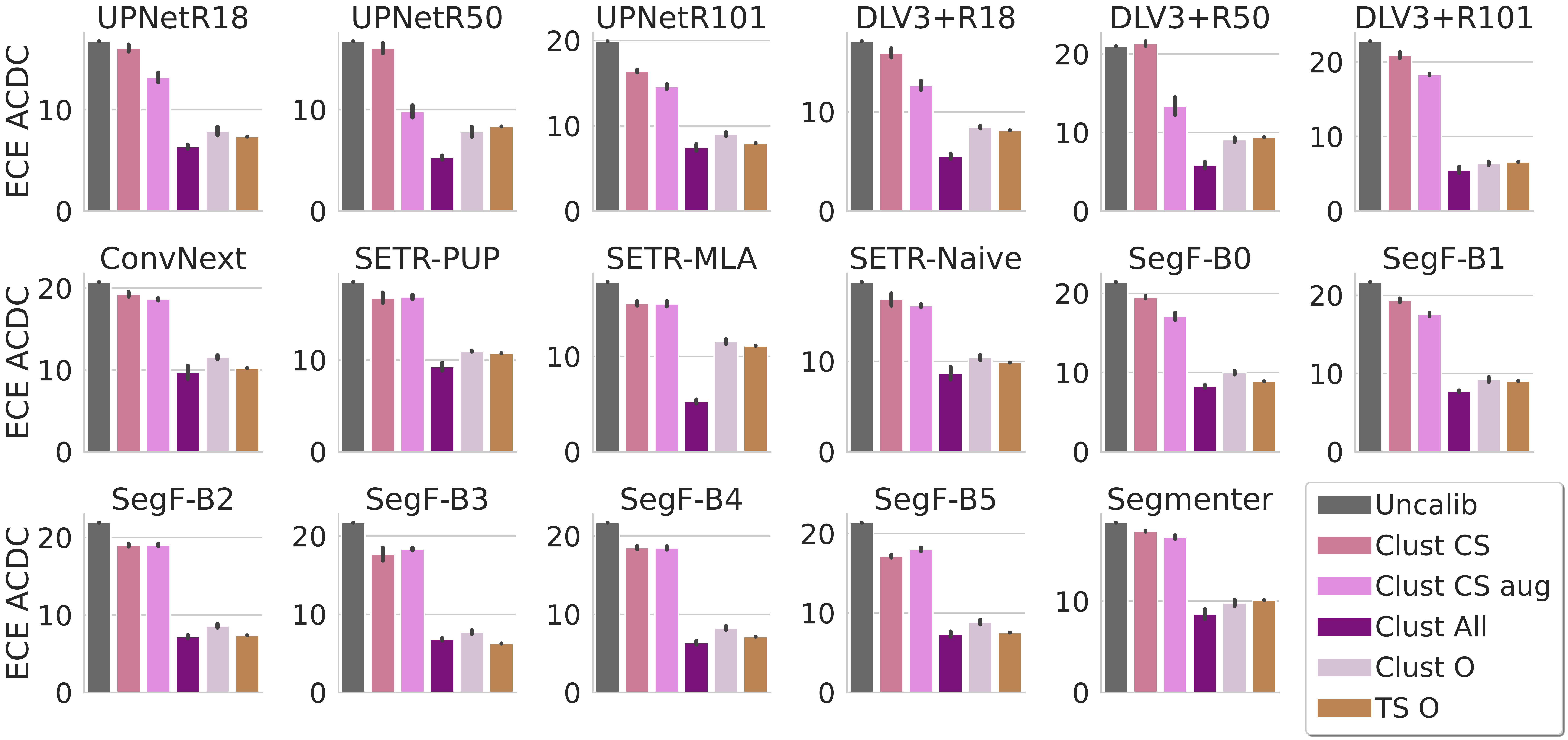}
\end{subfigure}
\caption{\textbf{ECE $(\downarrow)$ after clustering \ts} Extension of \cref{figure:cluster_ablation} in the main paper where we show results for all models.
}
\label{figure:clustering_methods_all_models}
\end{figure}

\clearpage

Complementary to \cref{figure:clustering_methods_all_models} where we show the different calibration methods for a given architecture with barplots, in \cref{figure:clustering_methods_ece_vs_miou} we present the same results but we group them by calibration method (instead of by model). For each test dataset (rows), we plot the ECE \vs mIoU after calibrating the models with the corresponding method (columns).

\begin{figure}[ht]
\centering
\includegraphics[width=0.8\linewidth]{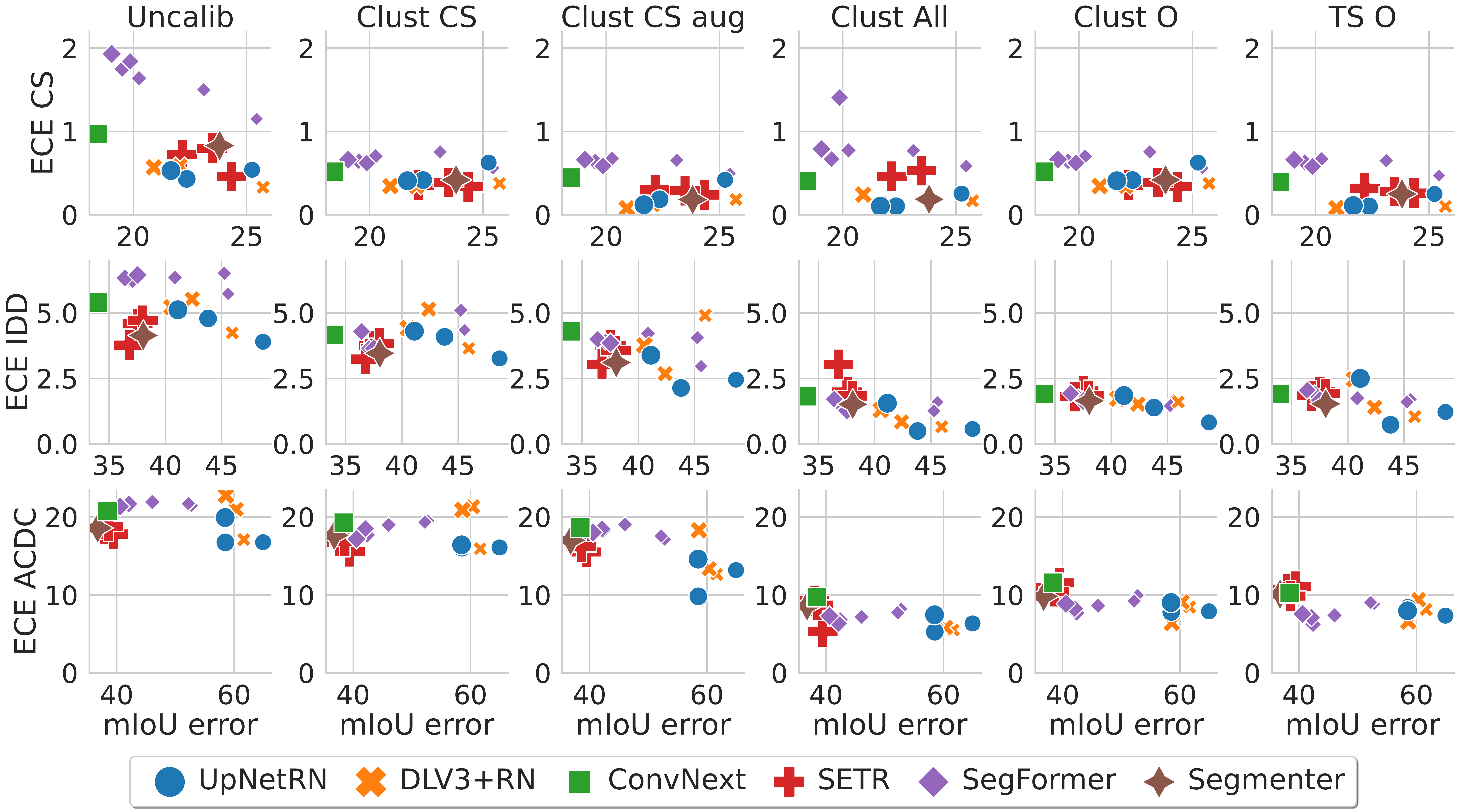}
\caption{\textbf{ECE $(\downarrow)$ \vs mIoU error $(\downarrow)$ after calibration with clustering \ts} considering different calibration sets. This plot provides a different visualization of the results in \cref{figure:clustering_methods_all_models}.}
\label{figure:clustering_methods_ece_vs_miou}
\end{figure}

\clearpage

\section{Additional plots: local temperature scaling}
\label{sec:additional_plots_LTS}
In \cref{figure:LTS_ablation} we analyzed the calibration error after applying the method by Ding~\etal~\cite{ding2021local} which learns a calibration network that predicts the temperature as a function of the image and segmentation model logits. Due to space constraints we only showed the best performing model for each family, in \cref{figure:lts_methods_all_models} we show the results for all models.

\begin{figure}[ht]
\begin{subfigure}[t]{\linewidth}
\centering
\includegraphics[width=0.75\linewidth]{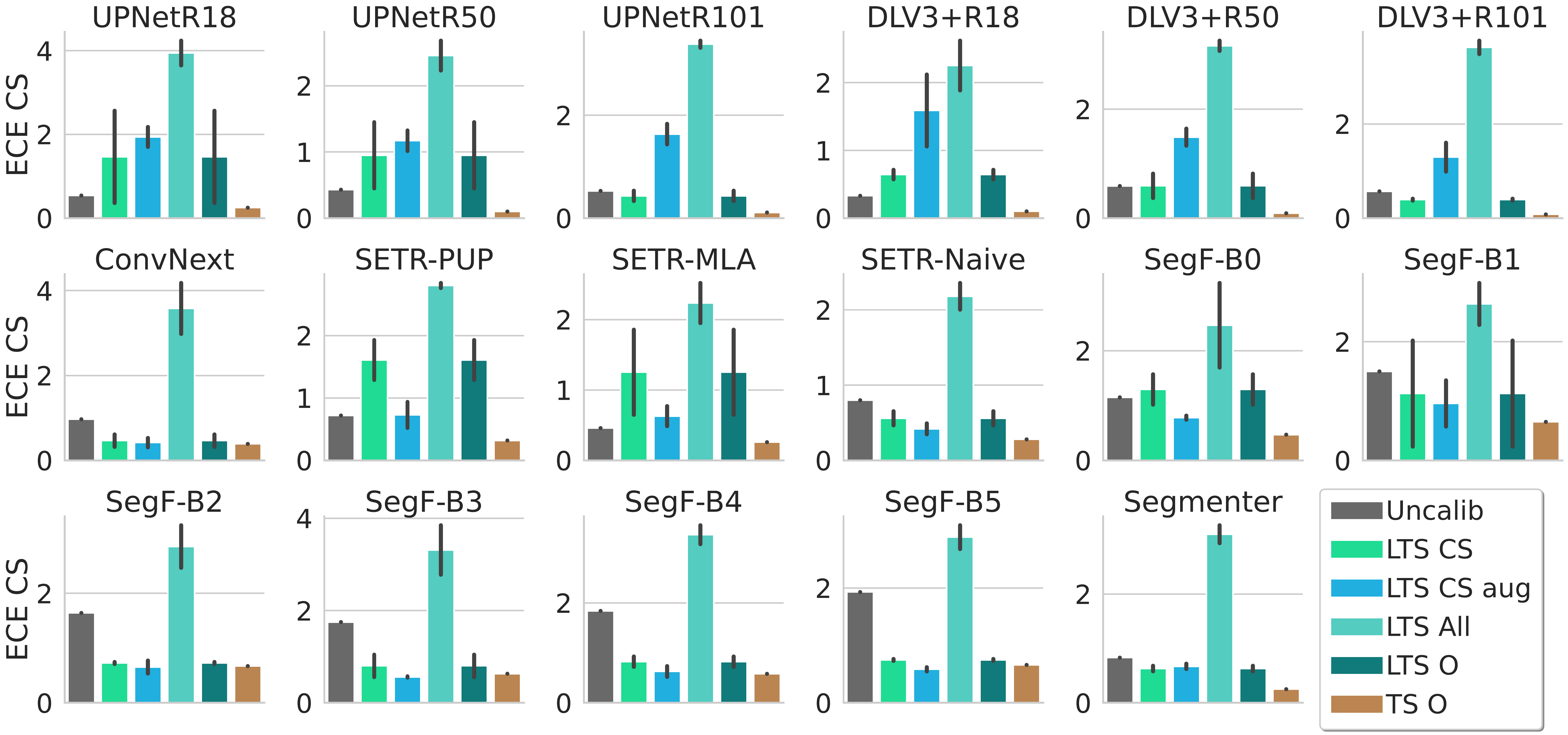}
\end{subfigure}

\vspace{0.3cm}

\begin{subfigure}[b]{\linewidth}
\centering
\includegraphics[width=0.75\linewidth]{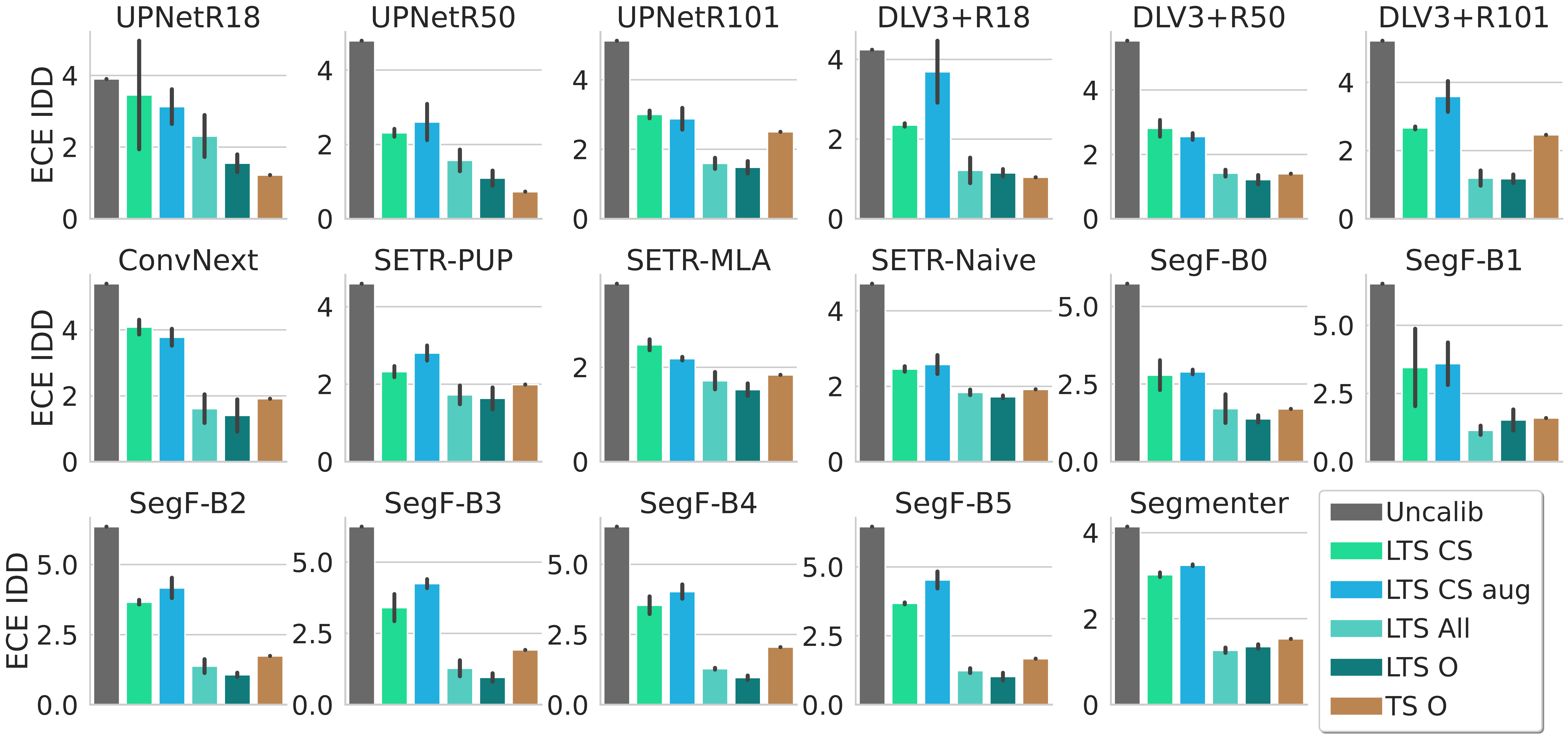}
\end{subfigure}

\vspace{0.3cm}

\begin{subfigure}[b]{\linewidth}
\centering
\includegraphics[width=0.75\linewidth]{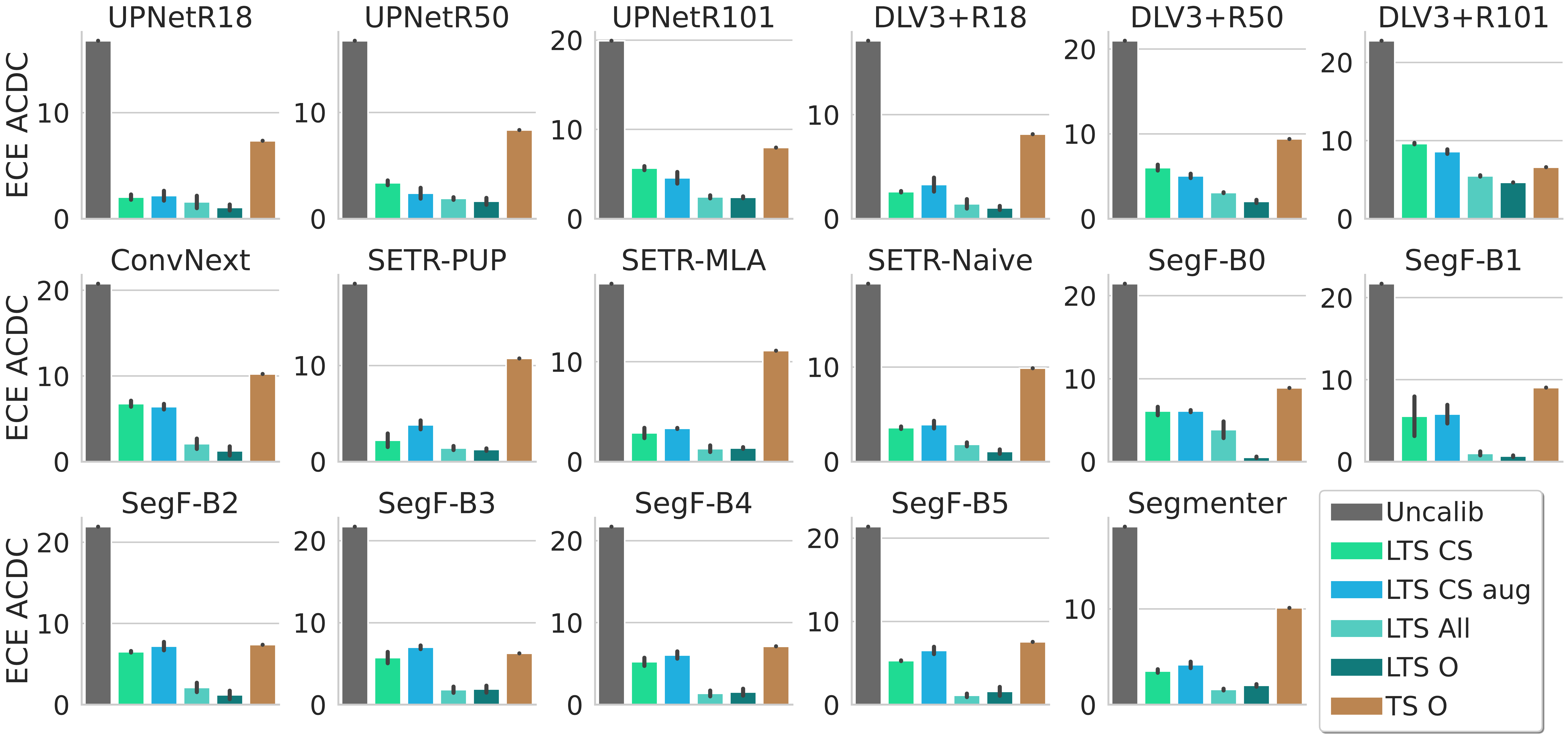}
\end{subfigure}
\caption{\textbf{ECE $(\downarrow)$ after Local \ts} Extension of \cref{figure:LTS_ablation} in the main paper where we show results for all models.
}
\label{figure:lts_methods_all_models}
\end{figure}

\clearpage

Complementary to \cref{figure:lts_methods_all_models} where we show the different calibration methods for a given architecture with barplots, in \cref{figure:lts_methods_ece_vs_miou} we present the same results but we group them by calibration method (instead of by model). For each test dataset (rows), we plot the ECE \vs mIoU after calibrating the models with the corresponding method (columns).

\begin{figure}[ht]
\centering
\includegraphics[width=0.8\linewidth]{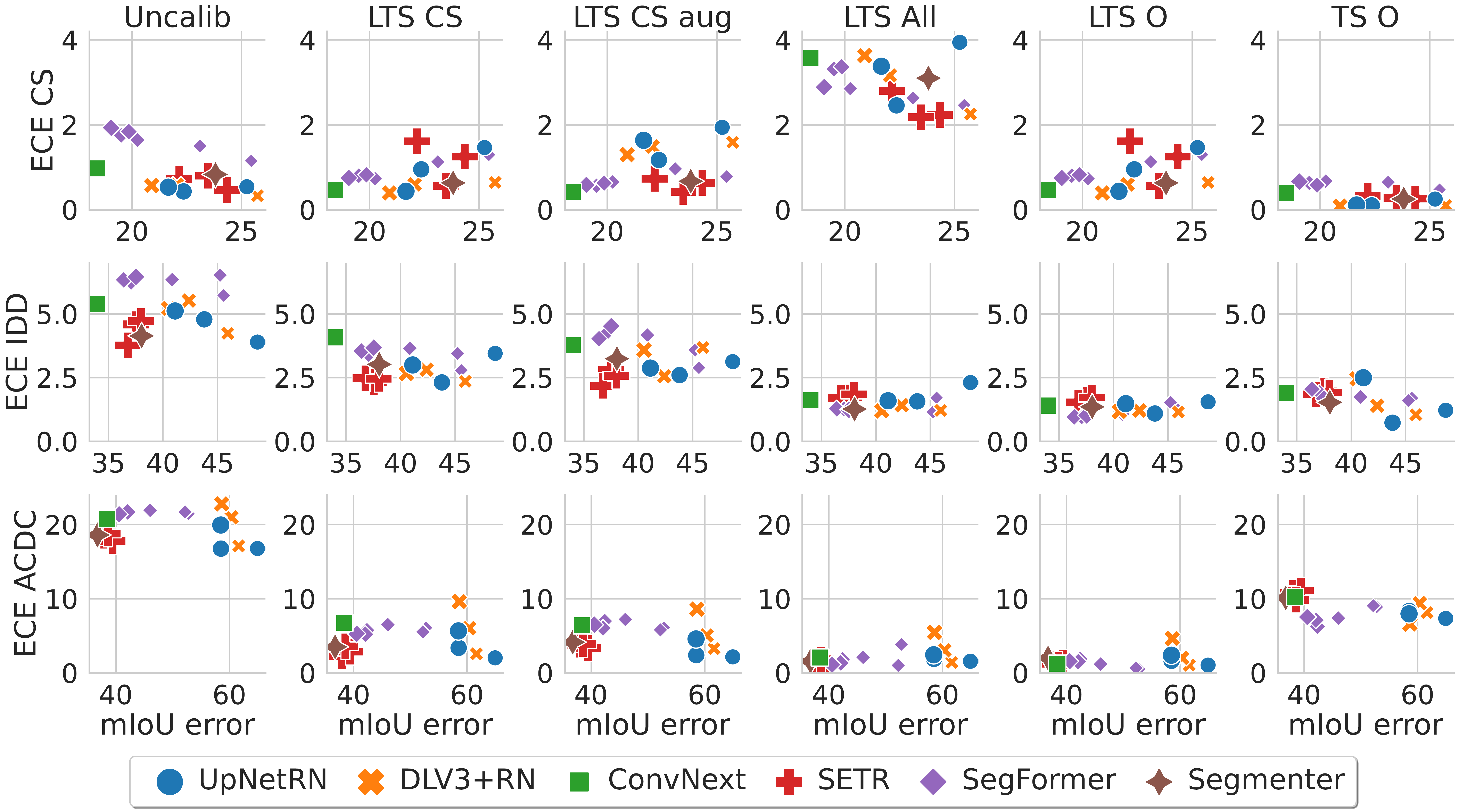}
\caption{\textbf{ECE $(\downarrow)$ \vs mIoU error $(\downarrow)$ after calibration with Local \ts} considering different calibration sets. This plot is showing the same results as \cref{figure:lts_methods_all_models} but in a different visualization.
}
\label{figure:lts_methods_ece_vs_miou}
\end{figure}

\clearpage

\section{Additional plots: Comparison calibration methods}
\label{sec:additional_plots_calib_methods_comparison}
In \cref{figure:calib_methods_comparison} we compared the calibration error after calibrating with TS CS, Clust CS and LTS CS \vs the uncalibrated baseline. Due to space constraints we only showed the best performing model for each family, in \cref{figure:calib_methods_all_models} we show the results for all models.

\begin{figure}[ht]
\centering
\begin{subfigure}[t]{\linewidth}
\centering
\includegraphics[width=0.75\linewidth]{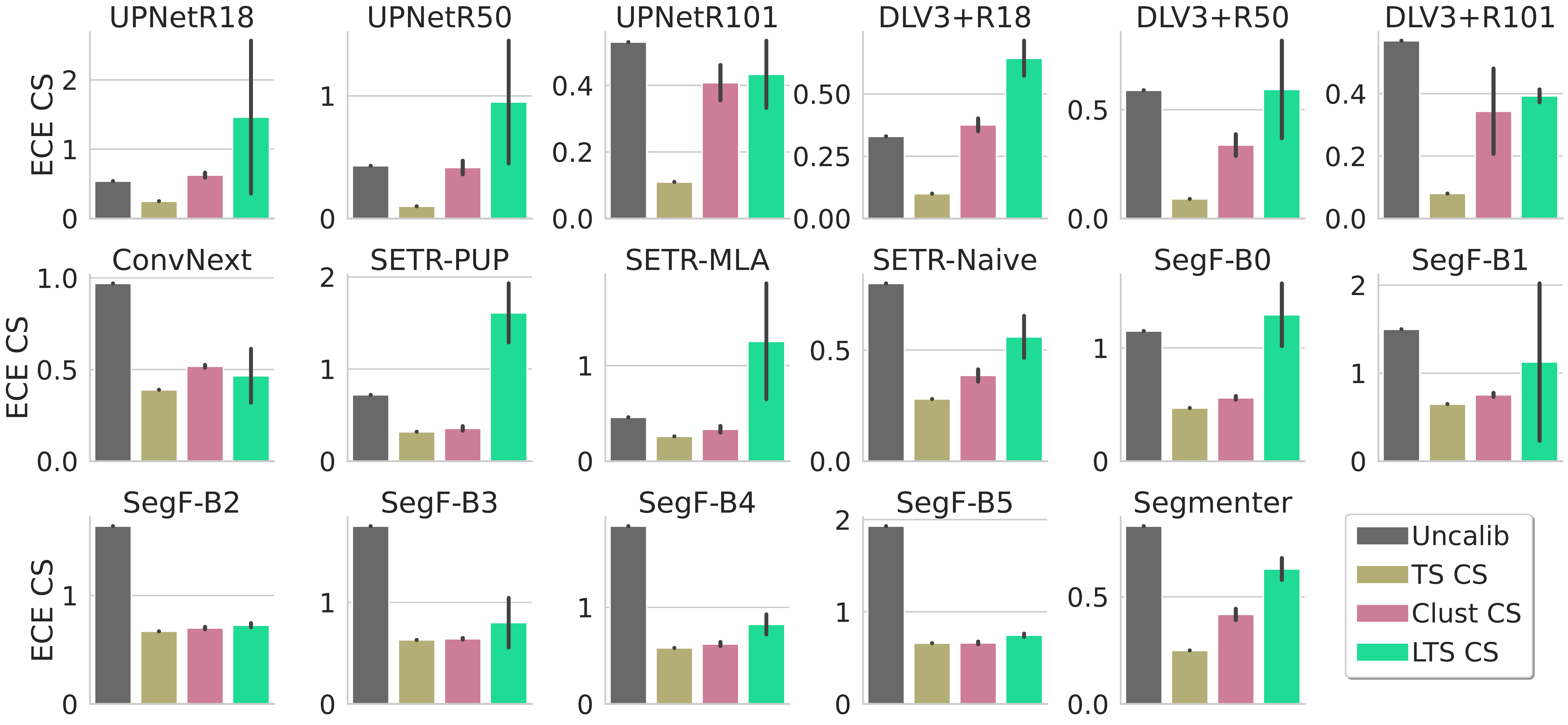}
\end{subfigure}

\vspace{0.3cm}

\begin{subfigure}[b]{\linewidth}
\centering
\includegraphics[width=0.75\linewidth]{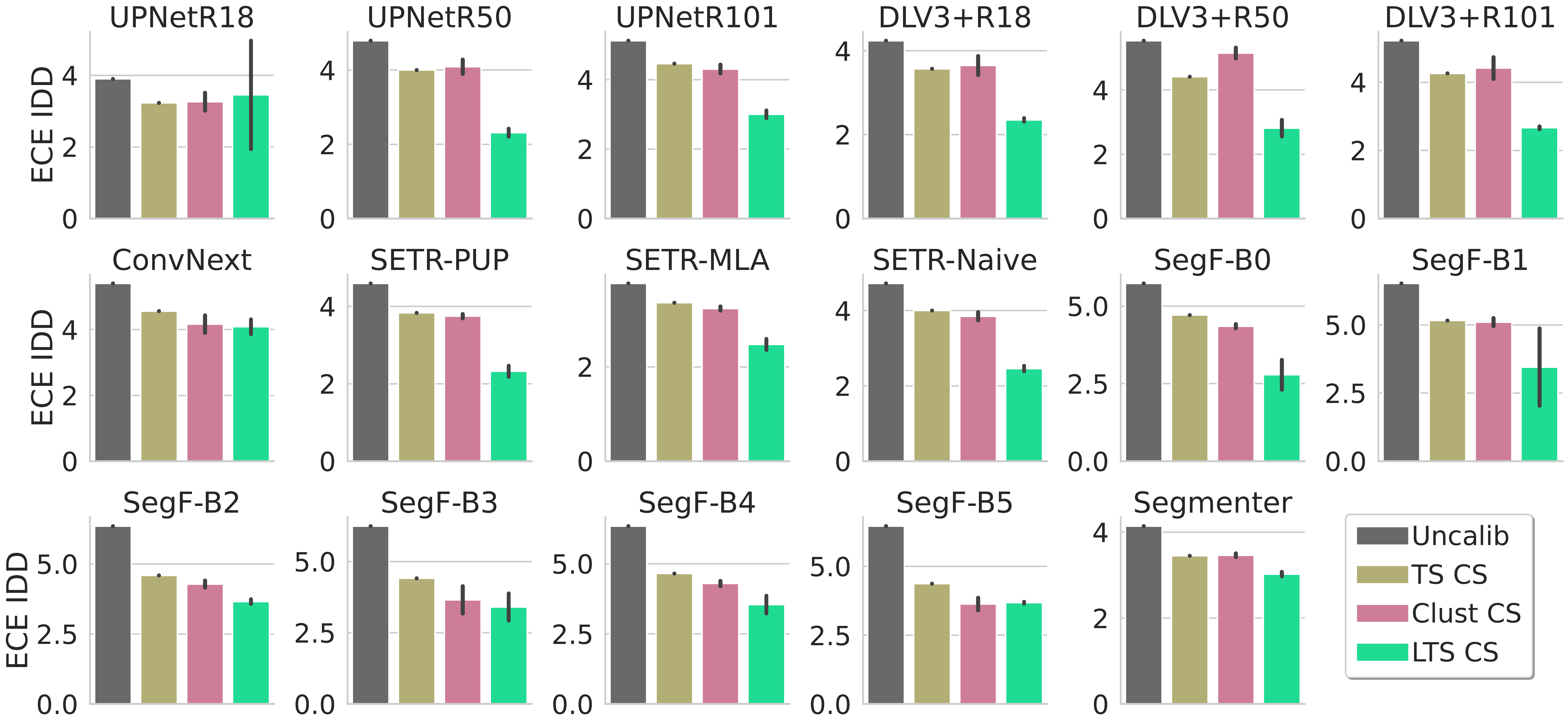}
\end{subfigure}

\vspace{0.3cm}

\begin{subfigure}[b]{\linewidth}
\centering
\includegraphics[width=0.75\linewidth]{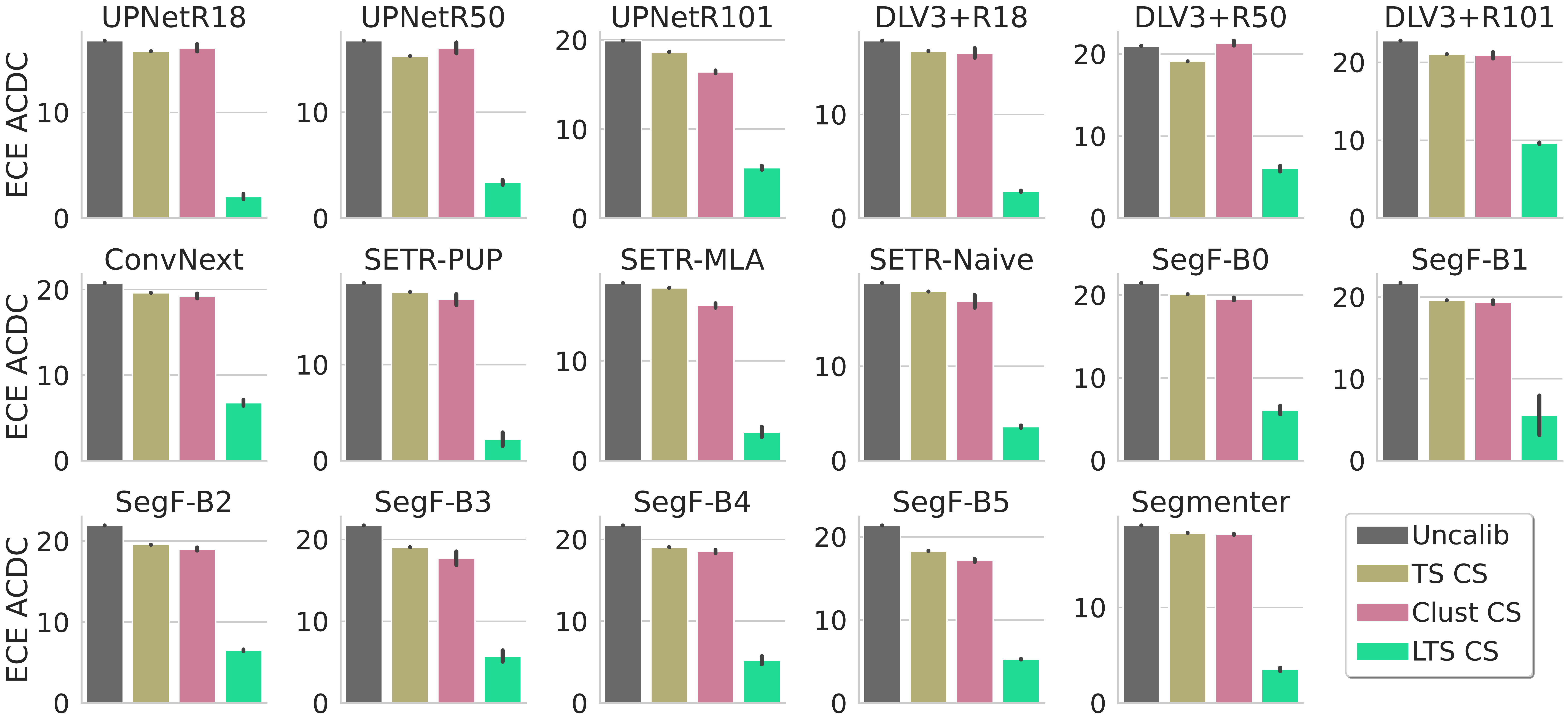}
\end{subfigure}
\caption{\textbf{ECE $(\downarrow)$ after calibration with different methods} Extension of \cref{figure:calib_methods_comparison} in the main paper where we show results for all models.
}
\label{figure:calib_methods_all_models}
\end{figure}

\clearpage

Complementary to \cref{figure:calib_methods_all_models} where we show the different calibration methods for a given architecture with barplots, in \cref{figure:calib_methods_ece_vs_miou} we present the same results but we group them by calibration method (instead of by model). For each test dataset (rows), we plot the ECE \vs mIoU after calibrating the models with the corresponding method (columns).

\begin{figure}[ht]
\centering
\includegraphics[width=0.8\linewidth]{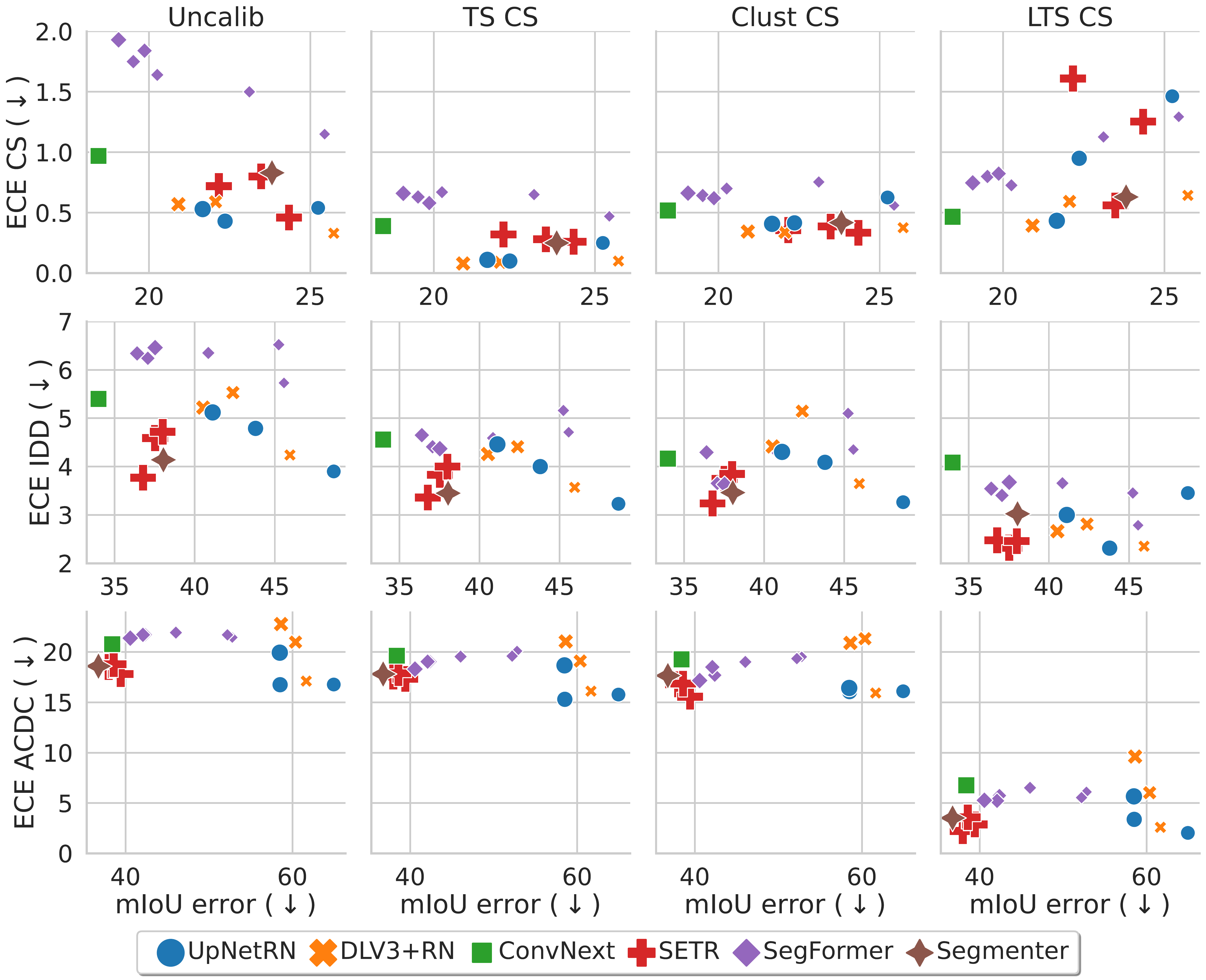}
\caption{\textbf{ECE $(\downarrow)$ \vs mIoU error $(\downarrow)$ after calibration with different methods} on the CS calibration set. This plot is showing the same results as \cref{figure:calib_methods_all_models} but in a different visualization.
}
\label{figure:calib_methods_ece_vs_miou}
\end{figure}

\clearpage

\section{Additional plots: misclassification detection}
\label{sec:additional_plots_misclassification}
In \cref{figure:misc_ood_after_calibration} we compared the misclassification detection and \ood detection performance of the networks after calibrating with TS CS, Clust CS and LTS CS \vs the uncalibrated baseline. Due to space constraints we only showed the best performing model for each family, in \cref{figure:misc_calib_all_models} we show the misclassification detection results for all models.

\begin{figure}[ht]
\centering
\begin{subfigure}[t]{\linewidth}
\centering
\includegraphics[width=0.75\linewidth]{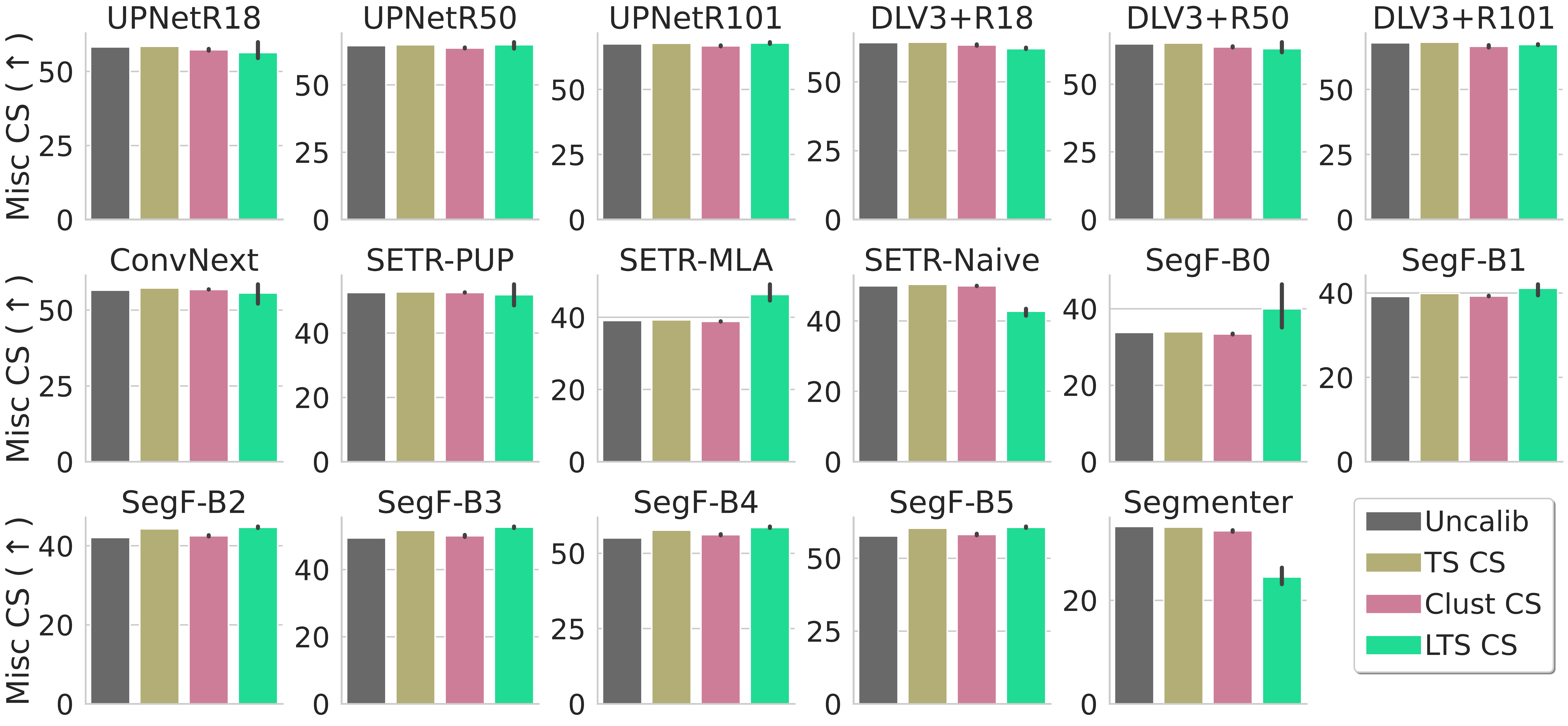}
\end{subfigure}

\vspace{0.3cm}

\begin{subfigure}[b]{\linewidth}
\centering
\includegraphics[width=0.75\linewidth]{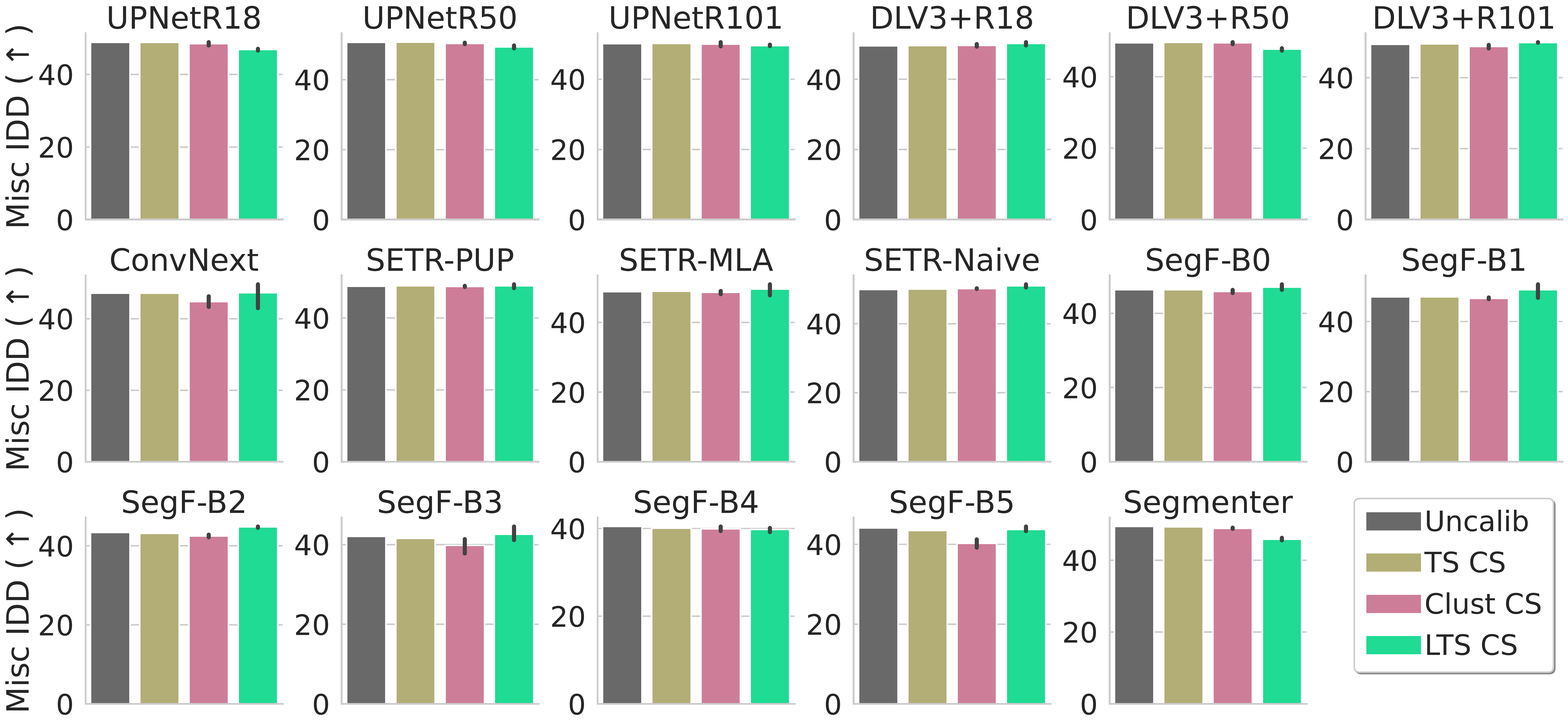}
\end{subfigure}

\vspace{0.3cm}

\begin{subfigure}[b]{\linewidth}
\centering
\includegraphics[width=0.75\linewidth]{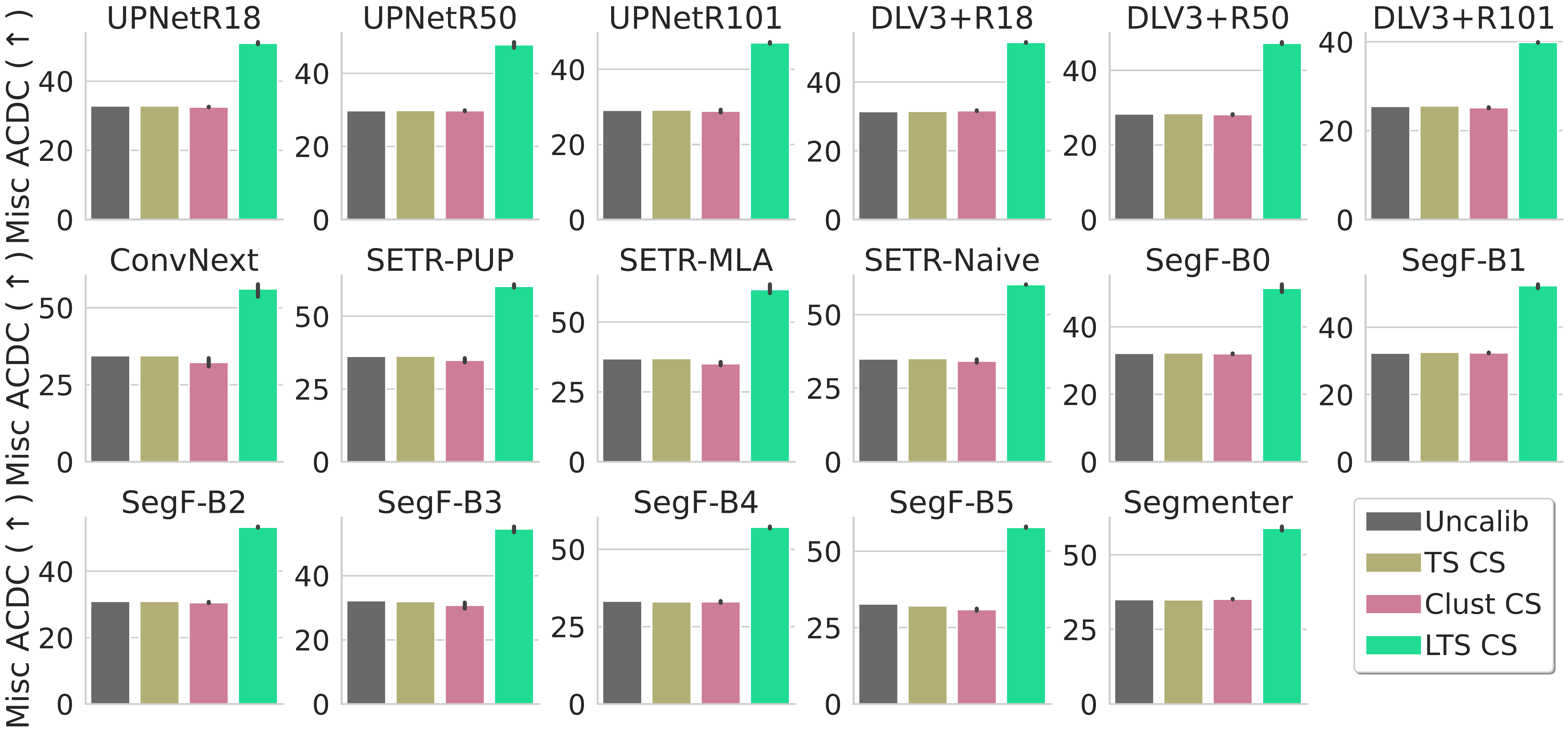}
\end{subfigure}
\caption{\textbf{Misc detection: PRR $(\uparrow)$ after calibration with different methods} Extension of \cref{figure:misc_ood_after_calibration} in the main paper where we show misclassification results for all models.
}
\label{figure:misc_calib_all_models}
\end{figure}

\clearpage

Complementary to \cref{figure:misc_calib_all_models} where we show the different calibration methods for a given architecture with barplots, in \cref{figure:misc_calib_ece_vs_miou} we present the same results but we group them by calibration method (instead of by model). For each test dataset (rows), we plot the ECE \vs mIoU after calibrating the models with the corresponding method (columns).

\begin{figure}[ht]
\centering
\includegraphics[width=0.8\linewidth]{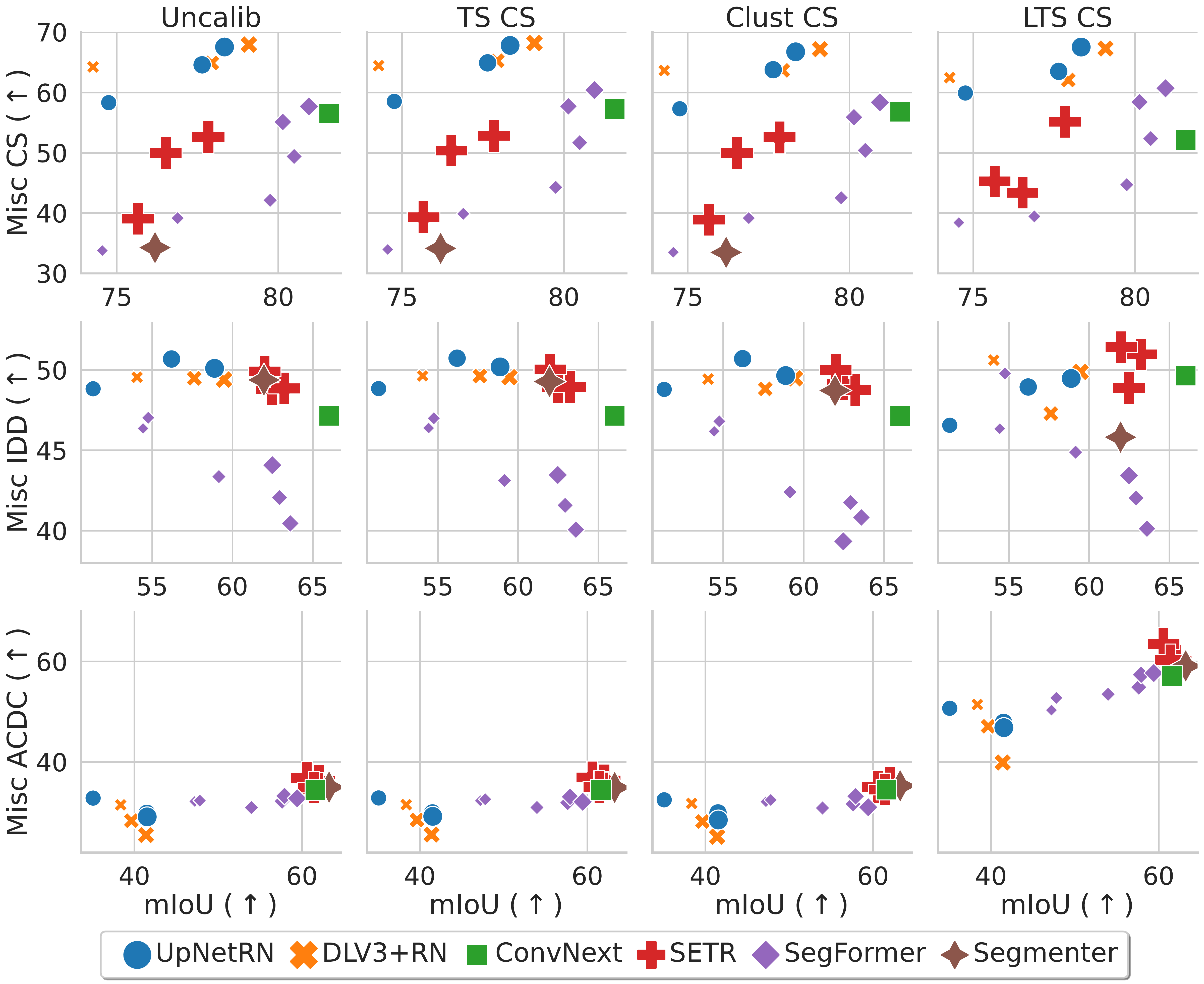}
\caption{\textbf{PRR $(\uparrow)$ \vs mIoU $(\uparrow)$ after calibration with different methods} on the CS calibration set. This plot is showing the same results as \cref{figure:misc_calib_all_models} but in a different visualization.
}
\label{figure:misc_calib_ece_vs_miou}
\end{figure}

\clearpage

\section{Additional plots: out-of-distribution detection}
\label{sec:additional_plots_ood}
In \cref{figure:misc_ood_after_calibration} we compared the misclassification detection and ood detection performance of the networks after calibrating with TS CS, Clust CS and LTS CS \vs the uncalibrated baseline. Due to space constraints we only showed the best performing model for each family, in \cref{figure:misc_calib_all_models} we show the \ood detection results for all models.

\begin{figure}[ht]
\centering
\begin{subfigure}[t]{\linewidth}
\centering
\includegraphics[width=0.7\linewidth]{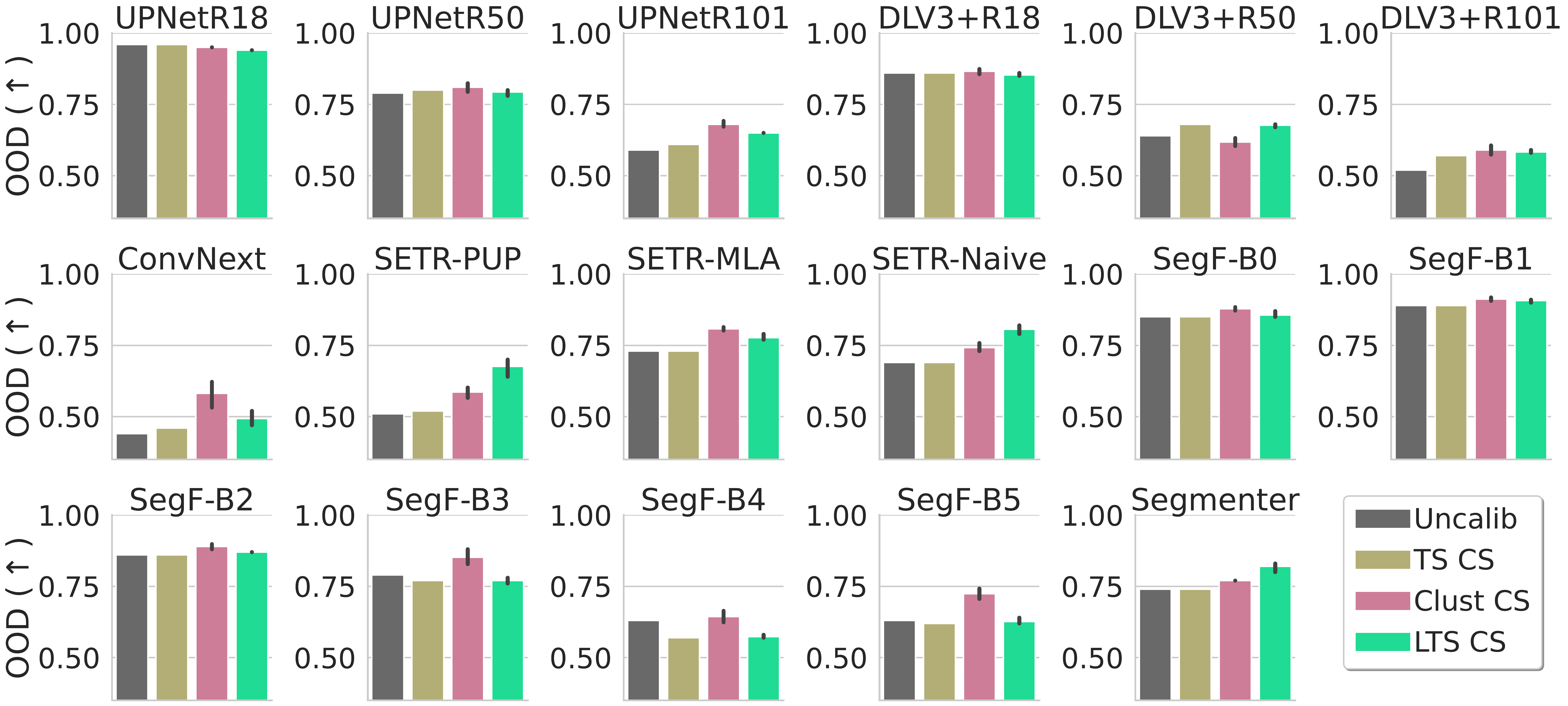}
\end{subfigure}
\caption{\textbf{\ood detection: AUROC $(\uparrow)$ after calibration with different methods} Extension of \cref{figure:misc_ood_after_calibration} in the main paper where we show \ood detection results for all models.
}
\label{figure:ood_calib_all_models}
\end{figure}

Complementary to \cref{figure:ood_calib_all_models} where we show the different calibration methods for a given architecture with barplots, in \cref{figure:ood_calib_ece_vs_miou} we present the same results but we group them by calibration method (instead of by model). For each test dataset (rows), we plot the ECE \vs mIoU after calibrating the models with the corresponding method (columns).

\begin{figure}[ht]
\centering
\includegraphics[width=0.7\linewidth]{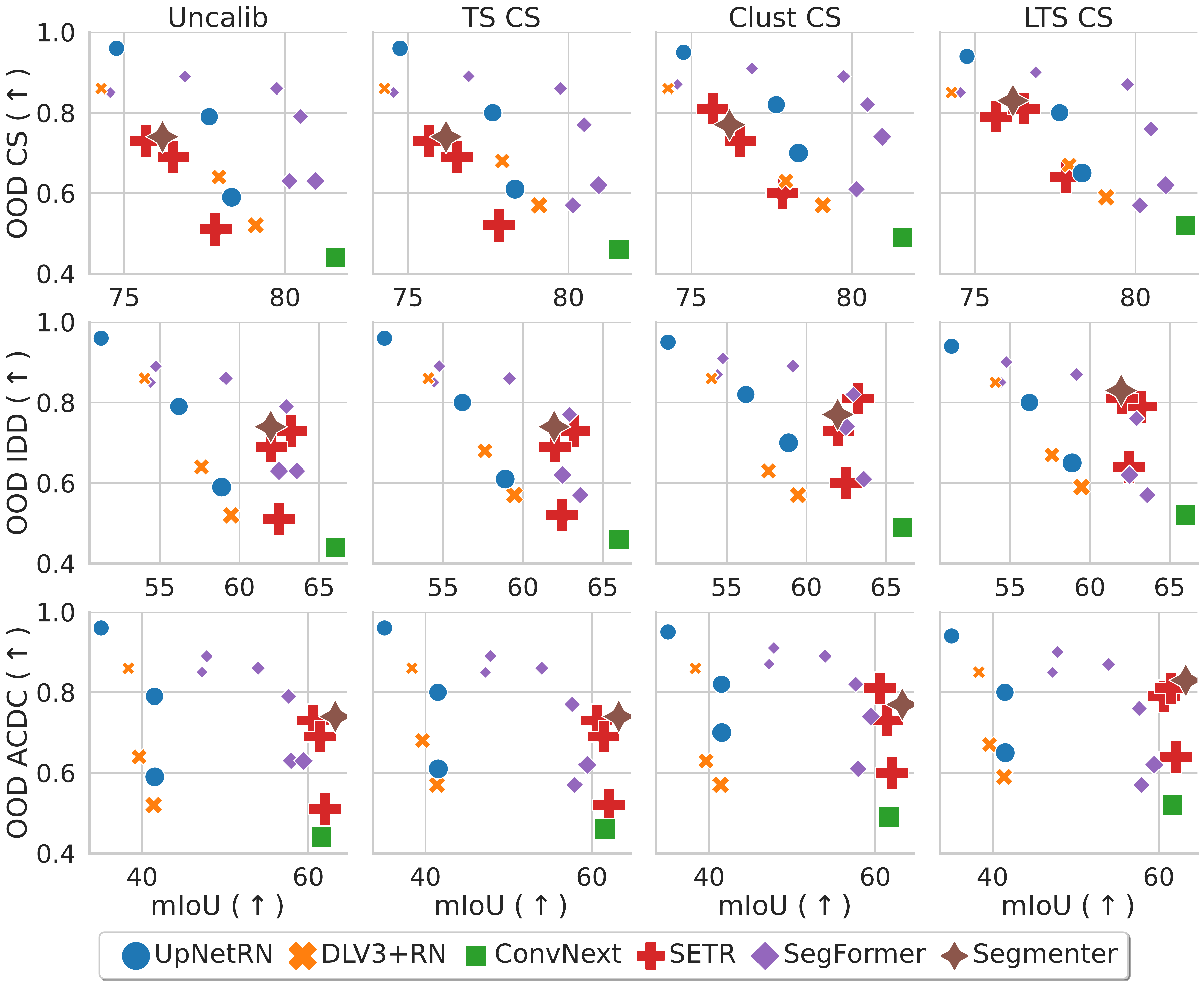}
\caption{\textbf{AUROC $(\uparrow)$ \vs mIoU $(\uparrow)$ after calibration with different methods} on the CS calibration set. This plot is showing the same results as \cref{figure:ood_calib_all_models} but in a different visualization.
}
\label{figure:ood_calib_ece_vs_miou}
\end{figure}

\end{document}